\documentclass{article} \PassOptionsToPackage{authoryear,round}{natbib}
\usepackage[preprint]{neurips_2026}   
\usepackage[utf8]{inputenc}
\usepackage[T1]{fontenc}
\usepackage{microtype}
\usepackage{dcolumn}
\usepackage{xcolor}
\usepackage{hyperref}                    \usepackage{url}

\usepackage{amsmath,amsfonts,bm}

\def\eqref#1{equation~\ref{#1}}
\def\Eqref#1{Equation~\ref{#1}}
\def\plaineqref#1{\ref{#1}}

\def\1{\bm{1}}

\def\eps{{\epsilon}}

\def\vmu{{\bm{\mu}}}

\def\va{{\bm{a}}}
\def\vb{{\bm{b}}}
\def\vc{{\bm{c}}}

\def\ve{{\bm{e}}}

\def\vh{{\bm{h}}}

\def\vq{{\bm{q}}}

\def\vs{{\bm{s}}}

\def\vu{{\bm{u}}}
\def\vv{{\bm{v}}}

\def\vx{{\bm{x}}}

\def\vz{{\bm{z}}}

\def\mI{{\bm{I}}}

\DeclareMathAlphabet{\mathsfit}{\encodingdefault}{\sfdefault}{m}{sl}
\SetMathAlphabet{\mathsfit}{bold}{\encodingdefault}{\sfdefault}{bx}{n}

\def\gA{{\mathcal{A}}}
\def\gB{{\mathcal{B}}}

\def\gD{{\mathcal{D}}}

\def\gL{{\mathcal{L}}}

\def\gN{{\mathcal{N}}}

\def\gS{{\mathcal{S}}}

\newcommand{\pdata}{p_{\rm{data}}}

\newcommand{\E}{\mathbb{E}}

\newcommand{\R}{\mathbb{R}}

\usepackage[textsize=tiny]{todonotes}
\usepackage{colortbl}
\usepackage{xspace}
\usepackage{graphicx}
\usepackage{subcaption}
\usepackage{wrapfig}
\usepackage{tabularx}
\usepackage{makecell}
\usepackage{algorithm,algorithmic}
\usepackage{placeins}
\usepackage{paralist}
\usepackage{amsmath,amssymb,bm}   \usepackage{multirow}
\usepackage{ntheorem}
\usepackage{enumitem}
\usepackage{changepage}           \usepackage{textcase}
\usepackage{booktabs}
\usepackage{fancybox}
\usepackage{mathtools}
\usepackage{pifont}
\usepackage{comment}
\usepackage{tikz}
\usetikzlibrary{positioning,calc,arrows.meta,backgrounds,fit,shapes.geometric,decorations.pathreplacing,patterns}

\newtheorem*{proof}{Proof}
\newtheorem{cor}{Corollary}

\graphicspath{{imgs/}{figs/}}

\providecommand{\yes}{\textcolor{red!75!black}{\ding{51}}}
\providecommand{\no}{\textcolor{green!55!black}{\ding{55}}}

\newlength\savewidth
\providecommand\shline{\noalign{\global\savewidth\arrayrulewidth
  \global\arrayrulewidth 1pt}\hline\noalign{\global\arrayrulewidth\savewidth}}
\providecommand\midline{\noalign{\global\savewidth\arrayrulewidth
  \global\arrayrulewidth 0.5pt}\hline\noalign{\global\arrayrulewidth\savewidth}}
\providecommand{\tablestyle}[2]{\setlength{\tabcolsep}{#1}\renewcommand{\arraystretch}{#2}\centering\footnotesize}
\newcolumntype{x}[1]{>{\centering\arraybackslash}p{#1pt}}

\def\x{\vx}                              \def\z{\vz}                              \def\h{\vh}                              \def\e{\ve}                              \def\veps{{\bm{\epsilon}}}               \def\eps{\veps}
\def\vsigma{{\bm{\sigma}}}             
\let\accentcedilla\c    \protected\def\c{\ifmmode\vc\else\expandafter\accentcedilla\fi}
\let\accentcaron\v      \protected\def\v{\ifmmode\vv\else\expandafter\accentcaron\fi}
\let\accentbreve\u      \protected\def\u{\ifmmode\vu\else\expandafter\accentbreve\fi}
\let\accenttabbing\a    \protected\def\a{\ifmmode\va\else\expandafter\accenttabbing\fi}
\let\accentbarunder\b   \protected\def\b{\ifmmode\vb\else\expandafter\accentbarunder\fi}

\def\I{\mI}                              
\def\N{\gN}                              \def\D{\gD}                                                            \def\B{\gB}                              \def\augs{\gA}                           \let\symbolsection\S    \protected\def\S{\ifmmode\gS\else\symbolsection\fi}   \let\symbolLslash\L     \protected\def\L{\ifmmode\gL\else\symbolLslash\fi}    
                      
\def\xt{\widetilde{\vx}}               \def\xh{\widehat{\vx}}                           
\def\estd{\bar{\ve}}                     \def\estdt{\widetilde{\ve}}              

\def\sg{\text{\bf sg}}

\def\cls{\texttt{[CLS]}}

\def\genb{g_\phi \circ \pi_\h \circ f_\theta}

\newcommand{\pnoise}{p_{\mathrm{noise}}}                                                                                                                   \newcommand{\logit}{\operatorname{logit}}                
                                                                                                                                                                                                                                                                                                     
                                                                                       \newcommand{\priornet}{e_{\omega}}                                                                                                                             
                                                                                                   \newcommand{\Lcbc}{\gL_{\mathrm{CBC}}}

\newcommand{\methodname}{SCION\xspace}

\newcommand{\methodlongnamebf}{\textbf{S}elf-\textbf{C}onditioned Generation on Self-Supervised Representat\textbf{ion}\xspace}

\newcommand{\dispersive}{Dispersive\xspace}

\definecolor{repC}{HTML}{1F5FA8}   \definecolor{repF}{HTML}{DCE9F8}
\definecolor{genC}{HTML}{6D3F9E}   \definecolor{genF}{HTML}{EAE1F4}
\definecolor{priC}{HTML}{A9720B}   \definecolor{priF}{HTML}{FAEBCE}
\definecolor{conC}{HTML}{B23A62}   \definecolor{conF}{HTML}{F7E3EA}
\definecolor{shC}{HTML}{37505C}    \definecolor{shF}{HTML}{EDF1F3}
\definecolor{scC}{HTML}{0FA84E}    \definecolor{scF}{HTML}{E3F7EA}
\definecolor{nulC}{HTML}{6E6E6E}   \definecolor{nulF}{HTML}{ECECEC}

\tikzset{
  mod/.style 2 args={
      rounded corners=2pt, draw=#1, fill=#2, line width=0.7pt,
      minimum height=0.64cm, inner xsep=3pt, inner ysep=2pt, align=center, text=black},
  loss/.style={
      rounded corners=3pt, draw=#1, fill=#1!10, line width=0.9pt,
      minimum height=0.80cm, minimum width=1.60cm, inner xsep=3pt, align=center},
  img/.style={inner sep=0pt, outer sep=0pt, draw=black!30, line width=0.45pt,
      rounded corners=0.085cm},
  oplus/.style={circle, draw=black!55, fill=white, line width=0.6pt,
      inner sep=0pt, minimum size=0.30cm, font=\scriptsize},
  tag/.style={font=\scriptsize, inner sep=1pt},
  ttag/.style={font=\scriptsize, inner sep=1.2pt, fill=white, fill opacity=0.92,
      text opacity=1, rounded corners=1pt},
  badge/.style={circle, draw=#1, fill=white, line width=0.7pt, inner sep=0pt,
      minimum size=0.30cm, font=\scriptsize\bfseries, text=#1},
  flow/.style={->, draw=#1, line width=0.9pt},
  nograd/.style={->, draw=#1, line width=0.8pt},
  scline/.style={->, draw=scC, line width=1.15pt, dash pattern=on 2.2pt off 1.6pt},
  panel/.style={font=\footnotesize\bfseries, inner sep=0pt, anchor=west},
}

\newcommand{\lossicon}[6][]{  \node[loss=#3, minimum width=1.10cm, minimum height=0.62cm, inner sep=0pt, #1]
       (#2) at (#4,#5) {};
  \begin{scope}[shift={(#4,#5)}, line width=0.42pt,
                >={Stealth[length=1.0mm,width=0.95mm]}]#6\end{scope}}

\newcommand{\vstrip}[5]{  \node[inner sep=0pt, outer sep=0pt, minimum width=0.72cm, minimum height=0.26cm]
       (#1) at (#2,#3) {};
  \foreach \i in {0,1,2,3}{
    \draw[draw=#4, fill=#5, line width=0.45pt]
      ($(#1.south west)+(\i*0.18,0)$) rectangle ($(#1.south west)+(\i*0.18+0.18,0.26)$);}}
\newcommand{\onestrip}[5]{  \vstrip{#1}{#2}{#3}{#4}{#5}  \foreach \i in {0,1,2,3}{
    \node[font=\tiny, inner sep=0pt, text=#4]
      at ($(#1.south west)+(\i*0.18+0.09,0.13)$) {1};}}
\newcommand{\sstrip}[5]{  \node[inner sep=0pt, outer sep=0pt, minimum width=0.54cm, minimum height=0.26cm]
       (#1) at (#2,#3) {};
  \foreach \i in {0,1,2}{
    \draw[draw=#4, fill=#5, line width=0.45pt]
      ($(#1.south west)+(\i*0.18,0)$) rectangle ($(#1.south west)+(\i*0.18+0.18,0.26)$);}}

\newcommand{\sstack}[5]{  \begin{scope}[shift={(-0.17,0.14)}]\sstrip{#1c}{#2}{#3}{#4}{#5}\end{scope}  \begin{scope}[shift={(-0.085,0.07)}]\sstrip{#1b}{#2}{#3}{#4}{#5}\end{scope}  \sstrip{#1}{#2}{#3}{#4}{#5}}
\colorlet{eb1}{priC!25}\colorlet{eb2}{priC!55}\colorlet{eb3}{priC!12}\colorlet{eb4}{priC!40}
\colorlet{ep1}{black!62}\colorlet{ep2}{black!18}\colorlet{ep3}{black!80}\colorlet{ep4}{black!34}
\colorlet{et1}{eb1!50!ep1}\colorlet{et2}{eb2!50!ep2}
\colorlet{et3}{eb3!50!ep3}\colorlet{et4}{eb4!50!ep4}
\newcommand{\cellstrip}[7]{  \node[inner sep=0pt, outer sep=0pt, minimum width=0.72cm, minimum height=0.26cm]
       (#1) at (#2,#3) {};}
\newcommand{\estrip}[3]{  \cellstrip{#1}{#2}{#3}{}{}{}{}  \foreach \i/\f in {0/eb1,1/eb2,2/eb3,3/eb4}{
    \draw[draw=priC, fill=\f, line width=0.45pt]
      ($(#1.south west)+(\i*0.18,0)$) rectangle ($(#1.south west)+(\i*0.18+0.18,0.26)$);}}
\newcommand{\nstrip}[3]{  \cellstrip{#1}{#2}{#3}{}{}{}{}  \foreach \i/\f in {0/ep1,1/ep2,2/ep3,3/ep4}{
    \draw[draw=nulC, fill=\f, line width=0.45pt]
      ($(#1.south west)+(\i*0.18,0)$) rectangle ($(#1.south west)+(\i*0.18+0.18,0.26)$);}}
\newcommand{\mixstrip}[3]{  \cellstrip{#1}{#2}{#3}{}{}{}{}  \foreach \i/\f in {0/et1,1/et2,2/et3,3/et4}{
    \draw[draw=priC, fill=\f, line width=0.45pt]
      ($(#1.south west)+(\i*0.18,0)$) rectangle ($(#1.south west)+(\i*0.18+0.18,0.26)$);}}
\newcommand{\sgcut}[3]{  \fill[white] ($(#2,#3)+(-0.04,-0.09)$) rectangle ($(#2,#3)+(0.04,0.09)$);
  \draw[#1, line width=0.9pt] ($(#2,#3)+(-0.13,-0.05)$) -- ($(#2,#3)+(0.05,0.13)$);
  \draw[#1, line width=0.9pt] ($(#2,#3)+(-0.05,-0.13)$) -- ($(#2,#3)+(0.13,0.05)$);}

\newcommand{\gstrip}[5]{  \node[inner sep=0pt, outer sep=0pt, minimum width=0.72cm, minimum height=0.54cm]
       (#1) at (#2,#3) {};
  \foreach \i in {0,1,2,3}{\foreach \j in {0,1,2}{
    \draw[draw=#4, fill=#5, line width=0.45pt]
      ($(#1.south west)+(\i*0.18,\j*0.18)$) rectangle
      ($(#1.south west)+(\i*0.18+0.18,\j*0.18+0.18)$);}}}

\newcommand{\legend}[4]{  \draw[draw=#3, fill=#3!25, line width=0.6pt, rounded corners=1pt]
       (#1,#2-0.09) rectangle (#1+0.28,#2+0.09);
  \node[tag, anchor=west] at (#1+0.36,#2) {#4};}

\title{From Pixels, Without Pre-training:\\ Joint Generative and
Self-Supervised \\Representation Learning in One Model}

\author{
Vicente Balmaseda\textsuperscript{1}\thanks{\url{https://vicentebalmaseda.com}} \hspace{2em}
Ching-Long Lin\textsuperscript{2} \hspace{2em}
Tianbao Yang\textsuperscript{1} \\[1em]
\begin{tabular}[t]{@{}c@{}}
\textsuperscript{1}Department of Computer Science and Engineering \\
Texas A\&M University \\
College Station, TX, USA \\
\texttt{\{vibalcam,tianbao-yang\}@tamu.edu}
\end{tabular}
\hspace{1em}
\begin{tabular}[t]{@{}c@{}}
\textsuperscript{2}Mechanical Engineering \\
University of Iowa \\
Iowa City, IA, USA \\
\texttt{ching-long-lin@uiowa.edu}
\end{tabular}
}

\begin{document}

\maketitle

\begin{abstract}

Strong image generation models today are conditioned on class labels, aligned to frozen pretrained encoders, or built on separately trained autoencoders or tokenizers. While effective, the cost is that generation becomes contingent on supervision and pretraining: labels must be annotated, and encoders, autoencoders, and tokenizers must be pretrained for the target domain. This paper studies {\bf joint generative and self-supervised representation learning in one model}, such that generation is self-conditioned without any labels or pretrained models. This is challenging since the two objectives are mismatched: contrastive learning consumes clean augmented views and keeps coarse, invariant semantics, while flow matching consumes noisy images and must preserve the fine detail and spatial layout that contrastive learning discards. We propose a framework Self-conditioned Generation on Self-supervised representation (SCION), whose core is a single pixel-space encoder conditioned on both the diffusion timestep and a conditioning embedding. The encoder serves two roles that differ in what that embedding is. For representation learning, the conditioning embedding is one learned global vector shared by all images, and the encoder's \cls{} token yields the semantic embedding trained by the contrastive loss. For generative training, the conditioning embedding is the \cls{} embedding of the image itself, and the encoder's patch tokens pass through a decoder to predict the image. To sample without a reference image at inference, we also learn a prior over the embedding in the same run. With gradient-norm balancing and stop-gradient tricks, all losses are jointly optimized in a single run. The resulting method is fully self-supervised, with no pretrained model and no labels.
On ImageNet $256\times256$, with the same JiT-B/16 recipe and without representation guidance, \methodname reaches $8.92$ FID,
surpassing the class-unconditional baselines iREPA, which aligns to a pretrained DINOv2 ($46.44$), and RCG, which conditions on it ($14.27$).
Under the JiT-L/16 recipe, \methodname{} achieves the best pixel-space class-unconditional FID of $5.89$ and $3.47$ without and with representation guidance, outperforming RCG with the ADM recipe ($6.24$).

\end{abstract}

\suppressfloats[t]

\section{Introduction}
\label{sec:intro}

Diffusion and flow matching models have achieved remarkable success in image generation \citep{Ho2020,Song2021,lipman_flow_2023,liu2023flow}. Yet many strong systems rely on class labels for semantic conditioning \citep{dhariwal2021diffusion,li2025jit}, frozen pretrained encoders for conditioning or representation alignment \citep{NEURIPS2024_e304d374,yuRepresentationAlignmentGeneration2025,singh_what_2025}, or separately trained autoencoders for latent-space generation \citep{rombachHighResolutionImageSynthesis2022}. These dependencies add annotation or pretraining costs and complicate extension to domains where suitable labels or pretrained models are unavailable.   A central challenge is therefore to enable \emph{semantically self-conditioned generation}, where the semantic representations that condition generation are learned within the diffusion model itself, jointly with the generative process.

We pursue semantic self-conditioned generation under four ``S'' requirements: \emph{self-contained}, without pretrained components, including autoencoders and tokenizers; \emph{self-supervised}, without label supervision anywhere in the pipeline, including that inherited through pretrained components; \emph{single-model}, with representations produced by the generator's own network; and \emph{single-stage}, with all components learned jointly from scratch. Existing methods (cf. Table~\ref{tab:deps}) address different aspects of this goal but satisfy these requirements only partially~\citep{li2025jit,yuRepresentationAlignmentGeneration2025,NEURIPS2024_e304d374,52203mage,preechakul2022diffusion,hudson2024soda,wang2025diffusedisperseimagegeneration}. These comparisons highlight the remaining challenge: achieving semantic self-conditioning under all four ``S'' requirements, with useful representations and unconditional synthesis learned together within one model.

\begin{wrapfigure}{r}{0.44\linewidth}
    \vspace{-2pt}
    \centering
    \includegraphics[width=\linewidth]{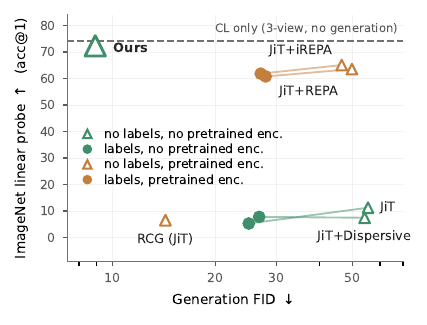}
    \caption{\textbf{Generation and Representation capabilities} of different models trained on ImageNet $256{\times}256$ with $400$
    epochs. FID under the ADM evaluation protocol. A segment joins a method's
    class-conditional and class-unconditional variants, so its length is what the class label buys.     }
    \label{fig:frontier}
    \vspace{-10pt}
\end{wrapfigure}
A natural approach to semantically self-conditioned generation is to learn the conditioning representations through contrastive self-supervised learning, which extracts semantics directly from unlabeled images \citep{chen_simple_2020,he2020momentumcontrastunsupervisedvisual,yuan_provable_2022}. However, integrating self-supervised representation learning with flow matching in one network requires reconciling their conflicting input and output requirements. Contrastive learning operates on clean augmented views and encourages compact semantic embeddings that suppress variations in appearance and spatial layout. Flow matching instead operates on noisy images and requires dense predictions that recover fine visual detail and spatial structure. The shared network must therefore accommodate both clean and noisy inputs while supporting semantic invariance without discarding information needed for synthesis. Joint training further couples the generator to an evolving conditioning representation, and sampling requires obtaining conditioning embeddings without a reference image.

To address these challenges, we propose \methodlongnamebf{} (\methodname{}), a framework for semantically self-conditioned generation through joint generative and self-supervised representation learning. At its core is a single pixel-space transformer encoder conditioned on both a timestep and a conditioning embedding.  Our key idea is to use this encoder in two conditioning modes with different readouts. For representation learning, the encoder processes clean augmented views using a learned global conditioning vector shared across all images, and a contrastive objective trains its \cls{} embedding. For generative learning, the same encoder processes noisy images conditioned on the clean image's own \cls{} embedding, and a decoder maps its patch tokens to a clean-image prediction under a flow matching objective. This design keeps contrastive inputs clean while using spatial features for generation, allowing both objectives to train the shared encoder.

To coordinate joint optimization, we employ gradient-norm balancing and stop-gradient operations. We also train a prior over the learned embeddings in the same run, enabling generation without reference images. Randomly withholding the conditioning embedding during training also supports classifier-free guidance \citep{ho2022classifier}, with guidance supplied by learned representations rather than class labels. All components are trained jointly from scratch, without labels, pretrained teachers, latent autoencoders, or tokenizers. 
We evaluate on ImageNet at $256\times256$ against flow-based baselines matched on the JiT-B/16 recipe. \methodname{} reaches $8.92$ FID, improving by $6.1\times$ over the
strongest self-contained baseline ($54.14$), and outperforming methods that import pretrained semantics, either aligning with them (iREPA, $46.44$) or conditioning on them (RCG at matched training budget, $14.27$). Self-representation guidance further reduces \methodname{} to $4.20$ FID, outperforming representation-guided RCG ($5.01$) without importing a pretrained encoder. At larger scale, \methodname{} (JiT-L/16) reaches $5.89$ unguided FID, improving on the $6.24$ reported by pixel-space RCG with an ADM training recipe. As Figure~\ref{fig:frontier} shows, \methodname{}'s learned representation also remains competitive with contrastive-only training.

\providecommand{\cyes}{\textcolor{green!55!black}{\ding{51}}}  \providecommand{\cno}{\textcolor{red!75!black}{\ding{55}}}     \providecommand{\cna}{\textcolor{black!35}{--}}                

\begin{table}[t]
\caption{\textbf{Comparison with related works under the four S requirements.} Extended discussion on related works is presented in Appendix~\ref{app:related}.
\textbf{Self-contained} means no pretrained component; \textbf{self-supervised} means no label supervision anywhere
in the pipeline; \textbf{single-model} means that
the representation is produced by the generator's own network; and \textbf{single-stage}
means that all components are learned jointly from scratch. Capabilities report a learned and
evaluated \textbf{Repr}esentation and unconditional \textbf{Synth}esis with no source image to encode. All pipelines are as published, except REPA and \dispersive{}, which we combine with the pixel-space JiT recipe.}
\label{tab:deps}
\vspace{1pt}
\centering\footnotesize
\setlength{\tabcolsep}{3.4pt}
\begin{tabular}{l@{\hskip 9pt}cccc@{\hskip 11pt}cc}
\toprule
 & \multicolumn{4}{c}{\emph{four S requirements}} & \multicolumn{2}{c}{\emph{capabilities}} \\
\cmidrule(r{9pt}){2-5}\cmidrule{6-7}
Method & Self-cont. & Self-sup. & Single-mod. & Single-stage & Repr. & Synth. \\
\midrule
\multicolumn{7}{@{}l}{\emph{Pixel-space generation}} \\[-2pt]
JiT \citep{li2025jit} & \cyes & \cno & \cna & \cyes & \cno & \cyes \\
JiT $+$ REPA \citep{yuRepresentationAlignmentGeneration2025} & \cno & \cno & \cno & \cno & \cno & \cyes \\
JiT $+$ \dispersive{} \citep{wang2025diffusedisperseimagegeneration} & \cyes & \cno & \cyes & \cyes & \cno & \cyes \\
\addlinespace[2pt]
\multicolumn{7}{@{}l}{\emph{Semantic conditions}} \\[-2pt]
RCG (ADM) \citep{NEURIPS2024_e304d374} & \cno & \cyes & \cno & \cno & \cno & \cyes \\
\addlinespace[2pt]
\multicolumn{7}{@{}l}{\emph{Joint representation and generation}} \\[-2pt]
MAGE \citep{52203mage} & \cno & \cno & \cyes & \cno & \cyes & \cyes \\
DiffAE \citep{preechakul2022diffusion} & \cyes & \cyes & \cno & \cno & \cyes & \cyes \\
SODA \citep{hudson2024soda} & \cyes & \cyes & \cno & \cyes & \cyes & \cno \\
\midrule
\rowcolor{gray!12}
\methodname{} (ours) & \cyes & \cyes & \cyes & \cyes & \cyes & \cyes \\
\bottomrule
\end{tabular}
\end{table}

\section{Preliminaries}
\label{sec:prelim}

Let \( \D = \{ \x_i \}_{i=1}^N \) be a dataset of images sampled from an unknown data distribution \( \pdata \), which we want to model.
For contrastive learning, we let \( A_a, A_b \sim \augs \) denote  two random augmentations from a predefined augmentation policy \( \augs \) (e.g., random crop, horizontal flip). For flow matching, we sample a timestep \( t \in [0,1] \) from a logit-normal distribution \( \mathrm{logit}(t) \sim \N(\mu_t, \sigma_t^2) \), and a noise vector \( \eps \sim \pnoise \) where  \( \pnoise = \N(0, \I) \) is a standard Gaussian distribution. Let $\|\cdot\|$ denote the Euclidean norm and let \( \mathrm{sim}(\u,\v) = \u^{\!\top}\v/(\|\u\|\,\|\v\|) \) denote the  cosine similarity between two vectors.

\subsection{Flow matching}
\label{sec:prelim:fm}

Flow matching \citep{lipman_flow_2023,liu2023flow,DBLP:conf/iclr/AlbergoV23} is a generative modeling framework that learns a continuous-time vector field whose integral curves transport a noise distribution \( \pnoise \) to the data distribution \( \pdata \).
Its trajectories follow the ordinary differential equation (ODE) \( d\x_t/dt = \v_t(\x_t) \) started at \( \x_0 = \eps \sim \pnoise \), so that \( \x_1 \sim \pdata \): the velocity field is the time derivative of this \emph{marginal} path, and solving the ODE is what generates samples.

The marginal field is defined from a \emph{conditional} path between a noise sample and a data sample, which we denote by \( \xt_t \).
Most works choose the conditional path according to the linear interpolant:
\begin{equation*}
    \xt_t = t\,\x + (1-t)\,\eps, \qquad \x \sim \pdata, \quad \eps \sim \pnoise.
\end{equation*}
Differentiating it gives the conditional velocity \( \v_t(\xt_t \mid \x, \eps) = d\xt_t/dt = \x - \eps \), which depends on the endpoints rather than on the current state alone. Following \citet{li2025jit}, instead of directly parameterizing the velocity field \( \v_t(\xt_t \mid \x, \eps) \), we use the fact that \( \v_t(\xt_t \mid \x, \eps) = (\x - \xt_t)/(1-t) \) and predict the clean image \( \x \) instead. Thus, we parameterize a neural network \( F_\theta(\xt_t, t) \) that predicts the clean image \( \x \) from the noised input \( \xt_t \) and time \( t \). Then, the Conditional Flow Matching (CFM) objective is given by:
\begin{equation}\label{eq:fm_loss}
    \L_{\mathrm{CFM}} = \E_{t, \eps \sim \pnoise, \x \sim \pdata} \left[ \left\| \frac{F_\theta(\xt_t, t) - \xt_t}{1 - t} - \frac{\x - \xt_t}{1 - t} \right\|^2 \right].
\end{equation}
Generation is done by approximately integrating the ODE through a numerical ODE solver and the trained model. For example, Euler's method with \( M \) steps gives
\begin{equation}
    \x_0 = \eps \sim \pnoise, \quad\quad
    \x_1 = \x_0 + \int_0^1 \v_t(\x_t) \, dt
    \approx \x_0 + \frac{1}{M} \sum_{i=0}^{M-1} \frac{F_\theta(\x_{i/M}, \frac{i}{M}) - \x_{i/M}}{1 - \frac{i}{M}}.
\end{equation}

\subsection{Contrastive Self-supervised Representation Learning}
\label{sec:prelim:cl}

For self-supervised representation learning, we consider contrastive objectives, which  learn a representation by pulling together strong augmented views of the same image
and pushing apart views of different images
\citep{chen_simple_2020,he2020momentumcontrastunsupervisedvisual}.
Thus, for a given image \( \x_i \), we take its two views \( A_a(\x_i) \) and \( A_b(\x_i) \) as a positive pair, and define its negatives as the set of augmented views of all other images in the dataset, \( \S^- = \{ A(\x) : \x \in \D \setminus \{\x_i\},\, A \sim \augs \} \).

Let \( f_\theta \) be an encoder that maps an image to a \( d \)-dimensional representation vector, and \( h_\psi \) a projection head that maps the representation to a space where the contrastive loss is applied.
We denote by \( E(\x) = h_\psi(f_\theta(\x)) \) the representation for contrastive learning.

Conventional approaches of contrastive learning either require a large batch size~\citep{chen_simple_2020} or maintain a large memory bank~\citep{he2020momentumcontrastunsupervisedvisual}. To avoid these, we follow the recent advances in optimizing a global contrastive objective~\citep{yuan_provable_2022,wei2024fastclip0}:
\begin{align}\label{eq:cl}
    \L_{\mathrm{GCL}} &\;=\; \E_{\x_i \sim \D,\; A_a, A_b \sim \augs}
    \Big[\; \ell\big(E(A_a(\x_i)),\, E(A_b(\x_i))\big) \;\Big], \\
    \ell(\u, \v) &\;=\; \tau \log Z(\u,\v),
    \qquad
    Z(\u,\v) \;=\; \E_{\x' \sim \S^-}\!\left[ \exp\!\big((\mathrm{sim}(\u, E(\x'))-  \mathrm{sim}(\u, \v))/\tau\big) \right]\nonumber.
\end{align}
To optimize this compositional objective, we use the SogCLR algorithm
\citep{yuan_provable_2022}. The key idea is to maintain a moving-average estimator of $Z(\x_i, \cdot)$ for each anchor image $\x_i$ using only mini-batch data and use it to compute a stochastic gradient estimator. We refer readers to \citet{yuan_provable_2022} for more details of the algorithm.

\section{Self-conditioned Generation on Self-Supervised Representation}
\label{sec:method}

In this section, we present the proposed framework SCION, including  the model structure and the training objective and strategies.

\subsection{The Unified Model}
\label{sec:method:arch}

Since we model images directly in pixel space, we adopt the ``Just Image Transformer'' (JiT) architecture for image generation~\citep{li2025jit}. To integrate generation with representation learning, we partition JiT into two components: an encoder \(f_\theta\) and a decoder \(g_\phi\), parameterized by \(\theta\) and \(\phi\), respectively. Both \(f_\theta\) and \(g_\phi\) consist of multiple Transformer blocks. To process an image, we split it into non-overlapping patches, linearly embed the patches into a sequence of \(T\) tokens, and prepend a learnable \cls{} token.

\noindent{\bf One encoder, two purposes.}
The encoder processes two types of inputs for two distinct purposes. For flow matching, it encodes a noise-diffused image \(\xt_t = t\x + (1-t)\eps\). For contrastive learning, it encodes a strongly augmented, non-diffused image \(A(\x)\). To accommodate both types of inputs, we condition the encoder \(f_\theta\) on the timestep \(t\in[0,1]\), where \(t=1\) corresponds to a non-diffused image. To enable semantically self-conditioned generation, we additionally condition the encoder on an embedding vector \(\c\), whose construction is described below. We therefore write the encoder as \(f_\theta(\z\mid t,\c)\), where \(\z\) denotes either a clean image $\x$, its augmented image \(A(\x)\) or a noise-diffused image \(\xt_t\).

The encoder output also serves two different purposes. For contrastive learning, it should provide a compact semantic representation that is invariant to variations in appearance and spatial layout. For generation, it must preserve fine-grained visual details and spatial structure. To accommodate these different requirements within a single encoder, we use two readouts from its output: the \cls{} token embedding \(\e\in\R^d\) used for representation learning, and the patch-token embeddings \(\h\in\R^{T\times d}\) used for generation. For convenience, we introduce two readout operators, \(\pi_\e\) and \(\pi_\h\), which extract the \cls{} token embedding and the patch-token embeddings from \(f_\theta(\z\mid t,\c)\), respectively.

\noindent{\bf Two conditioning modes for two purposes.}
For generative modeling, we would like the encoder to process the noise-diffused image \(\xt_t\) conditioned on a semantic representation \(\e(\x)\) of the corresponding clean image, namely, \(f_\theta(\xt_t\mid t,\e(\x))\). However, the representation \(\e(\x)\) must itself be produced by the same encoder, creating an apparent circular dependency: the representation used to condition the encoder cannot simultaneously depend on that conditioned forward pass.

To break this dependency, we introduce a learnable null embedding \(\e_\varnothing\). The semantic representation of a clean image is first obtained through a null conditioned encoding step, \(\e(\x)=\pi_\e\circ f_\theta(\x\mid 1,\e_\varnothing)\). Likewise, for a strongly augmented view \(A(\x)\), its representation for contrastive learning is computed as \(\e(A(\x))=\pi_\e\circ f_\theta(A(\x)\mid 1,\e_\varnothing)\). The resulting clean-image representation \(\e(\x)\) is then used as the semantic conditioning signal when processing the corresponding noise-diffused image, yielding \(\h_t=\pi_\h\circ f_\theta(\xt_t\mid t,\e(\x))\). Thus, the same encoder operates in two conditioning modes: null-conditioned encoding for representation learning and semantic-conditioned encoding for generation.

\paragraph{The decoder.}
Given a noise-diffused image \(\xt_t\), the decoder takes the encoder's patch-token embeddings \(\h_t=\pi_\h\circ f_\theta(\xt_t\mid t,\e(\x))\) as input. As with the encoder, we condition the decoder on both the timestep \(t\) and the clean-image representation \(\e(\x)\). The prediction is therefore given by \(g_\phi(\h_t\mid t,\e(\x))\).

In summary, the unified model uses the same encoder in two complementary ways. For representation learning, an augmented image is encoded with null conditioning as \(\e(A(\x))=\pi_\e\circ f_\theta(A(\x)\mid 1,\e_\varnothing)\), and its \cls{} token embedding is used for contrastive learning. For generation, the clean-image representation \(\e(\x)=\pi_\e\circ f_\theta(\x\mid 1,\e_\varnothing)\) is used to condition both the encoder and decoder when processing \(\xt_t\), with \(\h_t=\pi_\h\circ f_\theta(\xt_t\mid t,\e(\x))\) and prediction \(g_\phi(\h_t\mid t,\e(\x))\). In this way, a single encoder simultaneously supports semantic representation learning and semantically conditioned image generation. To implement conditioning in both the encoder and decoder, we adopt adaLN-Zero~\citep{peebles_scalable_2023}. The proposed model structure of SCION is presented in Figure~\ref{fig:method}.

\begin{figure}[t]
    \centering
    \resizebox{0.88\linewidth}{!}{\input{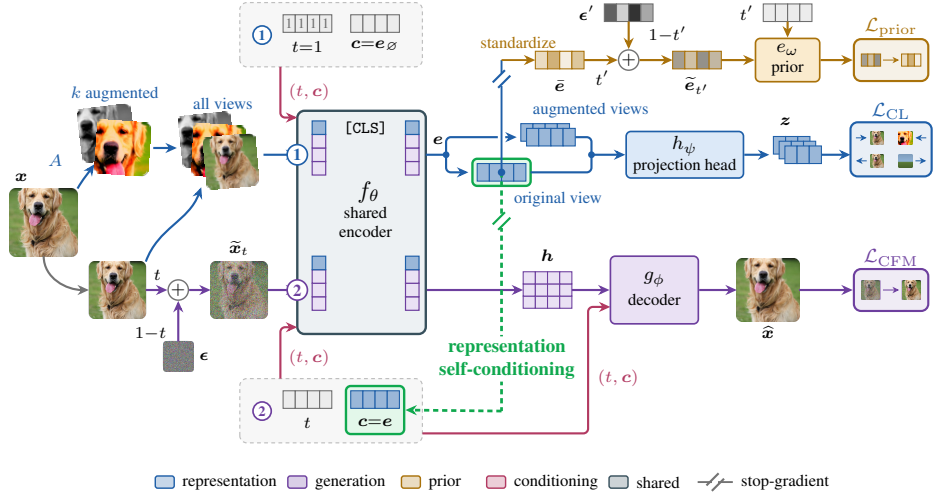}}
    \caption{The network structure and information flow of \textbf{\methodname{}.} One shared encoder \( f_\theta \) is run twice per step.
    \textcircled{\footnotesize 1} The \emph{representation pass} (blue) encodes every view clean at
    \( t{=}1 \) under the learned null \( \e_\varnothing \), and \( h_\psi \) maps the \cls{} readout
    to \( \z \) for \( \L_{\mathrm{CL}} \). \textcircled{\footnotesize 2} The \emph{generative pass} (violet)
    noises the original view and conditions on the detached \( \sg(\e) \) the first pass produced, the
    self-conditioning loop, drawn as the one dashed teal edge. Its patch tokens go to \( g_\phi \) for
    \( \L_{\mathrm{CFM}} \). The prior \( e_\omega \) (amber) regresses the standardized embedding
    \( \estd \). A \( /\!\!/ \) across an edge marks a stop-gradient. The auxiliary
    \( \L_{\mathrm{CBC}} \) term of Equation~\plaineqref{eq:joint_loss} is not drawn.}
    \label{fig:method}
\end{figure}

\subsection{The Joint Training Objective}
\label{sec:method:cond}
Below, we present each objective component and then combine them for joint training.

\noindent{\bf Contrastive loss.}
For a clean image $\x$, we generate two strong augmented copies $A_1(\x)$ and $A_2(\x)$.  All 3 views are encoded by conditioning on \( t = 1, \c=\e_\varnothing \). Following the practice of contrastive learning, the encoder's outputs are projected with a parameterized projection head $h_\psi$ to produce the embeddings for contrastive learning, i.e., $\z^0=h_\psi(f_\theta(\x \mid 1, \e_\varnothing))$, $\z^1=h_\psi(f_\theta(A_1(\x) \mid 1, \e_\varnothing))$,  and $\z^2=h_\psi(f_\theta(A_2(\x) \mid 1, \e_\varnothing))$. Each of the three views will  serve as a contrastive anchor for defining the contrastive loss:
$$\L_{\mathrm{GCL}}(\theta,\psi; \x) = \E_{A_1,A_2}\left[\frac{1}{6} \sum_{\substack{a,b=0, a \neq b}}^{2} \ell\big(\z^{a}, \z^{b}\big)\right].$$

\noindent{\bf FM loss.}
For the clean image $\x$, we sample a noise-diffused image $\xt_t$ and define the CFM loss as:
 $$ \L_{\mathrm{CFM}}(\theta,\phi; \x) =\E_{t, \eps} \Big\| \frac{g_\phi(\pi_\h\circ f_\theta(\xt_t\mid t,\e(\x))\mid t,\e(\x)) - \xt_t}{1-t} - \frac{\x - \xt_t}{1-t} \Big\|^2.$$

\noindent{\bf Learning the Representation Prior.}
During training, the generative model is conditioned on the representation $\e(\x)$ of the input image. When generating new images at inference time, however, this representation is unavailable. We therefore learn a prior over representations using a flow matching model $e_\omega$. For each image $\x$, we construct a standardized embedding $\bar\e(\x)=(\e(\x)-\vmu)/\vsigma$ and its noisy counterpart $\tilde\e_{t'}(\x)=t'\bar\e(\x)+(1-t')\eps'$, where $\vmu$ and $\vsigma$ denote the embedding mean and standard deviation vectors, division is elementwise, and $t'\in[0,1]$ and $\eps'$ are a sampled timestep and noise vector, respectively. Given $\tilde\e_{t'}(\x)$ and $t'$, the model $e_\omega$ predicts the clean standardized embedding $\bar\e(\x)$. Following  CFM, we define the representation prior FM loss as
\[
\L_{\mathrm{prior}}(\omega;\x)
=
\E_{t',\eps'}
\left\|
\frac{e_\omega(\tilde\e_{t'}(\x),t')-\tilde\e_{t'}(\x)}{1-t'}
-
\frac{\bar\e(\x)-\tilde\e_{t'}(\x)}{1-t'}
\right\|^2.
\]

\noindent{\bf Conditioned-Branch Consistency Loss.}
The shared encoder operates in two conditioning modes: null conditioning for representation learning and representation conditioning for generation. Even for a clean image, these modes may produce different outputs, namely $f_\theta(\x\mid 1,\e_\varnothing)$ and $f_\theta(\x\mid 1,\e(\x))$. Moreover, when processing a noisy image $\xt_t$ for generation, the encoder should produce representations that remain semantically consistent with the clean-image representation $\e(\x)=\pi_\e \circ f_\theta(\x\mid 1,\e_\varnothing)$. Such consistency keeps the \cls{} token informative and, as the patch tokens attend to it in the encoder, shapes the patch features fed to the decoder. Since the preceding losses do not explicitly enforce this consistency, we introduce the conditioned-branch consistency (CBC) loss:
\[
\L_{\mathrm{CBC}}(\theta;\x)
=
-\E_{t,\eps}\,
\mathrm{sim}\big(
\pi_\e \circ f_\theta(\xt_t\mid t,\e(\x)),\e(\x)
\big).
\]

Combining these losses yields the final joint objective:
\begin{equation}\label{eq:joint_loss}
\min_{\theta,\phi,\psi,\omega}
\E_{\x\sim\D}\Big[
\lambda\L_{\mathrm{GCL}}(\theta,\psi;\x)
+\L_{\mathrm{CFM}}(\theta,\phi;\x)
+\L_{\mathrm{prior}}(\omega;\x)
+\eta\L_{\mathrm{CBC}}(\theta;\x)
\Big],
\end{equation}
where $\eta$ is a fixed hyper-parameter and $\lambda$ is to balance between the global contrastive objective and the CFM objective.

 \begin{figure}[t]
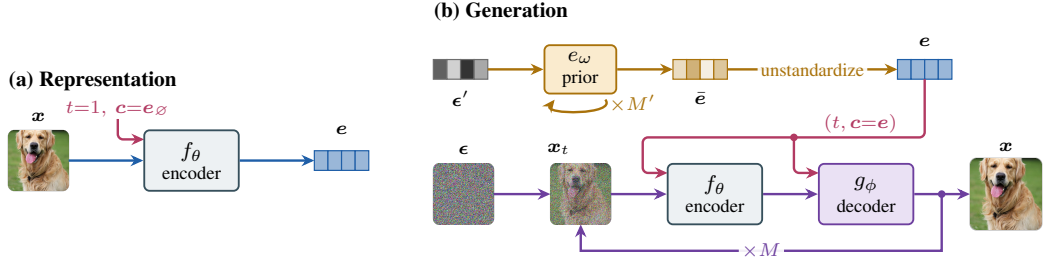

    \centering
    \begin{minipage}[c]{0.35\linewidth}\centering
        \input{figs/inference_a.tex}
    \end{minipage}\hfill
    \begin{minipage}[c]{0.61\linewidth}\centering
        \input{figs/inference_b.tex}
    \end{minipage}
    \caption{\textbf{How \methodname{} is used.}
    \textbf{(a)} A clean image is encoded at $t{=}1$ with the learned null conditioning
    $\e_\varnothing$, reading out the representation $\e$.
    \textbf{(b)} To generate, the prior $e_\omega$ is integrated from noise to draw a
    standardized embedding $\estd$, which is unstandardized to an $\e$ that conditions the
    image ODE at every step. No label enters either path.}
    \label{fig:teaser}
\end{figure}

\subsection{Training Strategies and Sampling}
\label{sec:method:weighting}
To facilitate the training of the unified model by optimizing the joint objective, we employ the following two strategies.

\noindent{\bf Stop Gradients.}
For CFM, the gradient equivalence between conditional and marginal velocity regression established by \citet{lipman_flow_2023} relies on target probability paths and velocity fields that are independent of the learnable parameters. In our setting, the conditioning representation $\e(\x)$ depends on $\theta$, so differentiating through it introduces an additional gradient path that is not covered by the standard equivalence argument (Appendix~\ref{app:grad_equiv}). We therefore apply a stop-gradient operation to $\e(\x)$ in the CFM loss, treating it as fixed conditioning during each update. We likewise detach $\e(\x)$ in $\L_{\mathrm{prior}}(\omega;\x)$ and $\L_{\mathrm{CBC}}(\theta;\x)$. Consequently, only the contrastive loss backpropagates through the computation of $\e(\x)$, while the CFM and CBC losses update the shared encoder through the $f_\theta$.

\noindent{\bf Gradient Balancing.}
The losses $\L_{\mathrm{CFM}}$ and $\L_{\mathrm{GCL}}$ operate on different scales: the former is a weighted squared error in pixel space, whereas the latter is a contrastive objective governed by the temperature $\tau$. Consequently, the trade-off coefficient $\lambda$ has no natural a priori scale. In our measurements, their gradient norms at the shared encoder differ by more than an order of magnitude,
with $\|\nabla \L_{\mathrm{CL}}\| \approx 40\, \|\nabla \L_{\mathrm{CFM}}\|$.
This ratio also varies with augmentation strength, learning rate, and batch size, making a fixed $\lambda$ difficult to transfer across settings. We therefore adapt $\lambda$ to balance the weighted gradient norms at a subset of shared encoder parameters $\varsigma\subseteq\theta$:
\begin{equation}\label{eq:lambda}
    \lambda = \mathrm{clamp}\!\left(
    \frac{\mathrm{EMA}\big(\|\nabla_\varsigma \L_{\mathrm{CFM}}\|\big) + \xi}
         {\mathrm{EMA}\big(\|\nabla_\varsigma \L_{\mathrm{GCL}}\|\big) + \xi},
    \lambda_{\min},\,\lambda_{\max}
    \right).
\end{equation}
Here, $\mathrm{EMA}$ denotes an exponential moving average, $\xi=10^{-8}$ provides numerical stability, and $[\lambda_{\min},\lambda_{\max}]=[10^{-3},1]$. We recompute the gradient norms and update their EMAs every $10$ training steps, with the clamped ratio tracking changes in the gradient imbalance. We choose $\varsigma$ as the parameters of the encoder's last MLP layer, the deepest layer whose parameters receive gradients from both losses. At this layer, CFM backpropagates through the generative pass's non-\cls{} tokens, while the contrastive loss backpropagates through the clean-view pass's \cls{} tokens. This separation allows their gradient norms to be recovered during the ordinary backward pass with additional computation confined to local gradient products and norms at this layer (Appendix~\ref{app:balancer_probe}).
In Appendix~\ref{app:algorithms}, we present Algorithm~\ref{alg:balancer} for the complete balancing procedure, and Algorithm~\ref{alg:train_step} for the joint training.

\paragraph{Sampling at Inference.}
At inference time, we first generate a standardized embedding by integrating the representation prior's ODE, defined by the model $e_\omega$, starting from embedding-space noise $\eps'$. We then reverse the standardization to obtain the representation $\e$ and use it to condition the image-generation ODE, defined by the model $F_{\theta, \phi}=g_\phi \circ \pi_\h \circ f_\theta$ (see Figure~\ref{fig:teaser}(b)).

\paragraph{Representation Guidance.} To enable classifier-free guidance \citep{ho2022classifier} with representation conditioning at inference, we replace the conditioning representation with a second learned null embedding $\c_\varnothing$ with probability $0.1$ during training. This allows the same network to learn both representation-conditioned and unconditional generation. At inference time, we integrate the image-generation ODE using the guided velocity field. Let the predicted velocity be \( \v_{\theta,\phi}(\x_t,t\mid\c) = (F_{\theta,\phi}(\x_t, t \mid \c) - \x_t)/(1-t) \):
\begin{equation}\label{eq:cfg}
    \v_{\mathrm{cfg}}
    =
    \v_{\theta,\phi}(\x_t, t \mid \c_\varnothing)
    +
    s\left(
    \v_{\theta,\phi}(\x_t, t\mid  \e)
    -
    \v_{\theta,\phi}(\x_t, t \mid  \c_\varnothing)
    \right),
    \qquad s\geq 1,
\end{equation}
where $s$ is the guidance scale. It is notable that the two null embeddings serve distinct purposes: $\c_\varnothing$ enables unconditional generation, whereas $\e_\varnothing$ allows the encoder to extract $\e(\x)$ without conditioning on the representation itself.

\section{Experiments}
\label{sec:experiments}

This section evaluates \methodname{} in three settings. We first conduct a controlled
comparison with baselines that share the same backbone, training budget, and flow-matching
recipe, isolating the contribution of the proposed framework. We then compare with
published unconditional pixel-space systems using their reported scores, providing a
system-level reference. Finally, we ablate \methodname{} to identify the contribution of each component and compare its representation with that of a contrastive-only model.

\textbf{Baselines}.
We compare \methodname{} with vanilla flow matching (FM) and four methods built on it:
\dispersive{} Loss \citep{wang2025diffusedisperseimagegeneration}, REPA
\citep{yuRepresentationAlignmentGeneration2025}, iREPA \citep{singh_what_2025}, and RCG
\citep{NEURIPS2024_e304d374}.
Thus, the baselines span FM alone, self-contained feature
regularization (\dispersive{}), external-feature alignment (REPA/iREPA), and external semantic conditioning (RCG).

\textbf{Setup}.
All controlled methods use the same JiT-B/16 backbone \citep{li2025jit} with attention residuals (Appendix~\ref{app:arch}), flow matching recipe, ImageNet-1K ($256 \times 256$) training data \citep{deng2009imagenet,imagenet15russakovsky}, and a total training budget of $400$ epochs.
The baselines that import semantics use a frozen DINOv2 ViT-B encoder
\citep{oquab2023dinov2}, pretrained on external data.
REPA and iREPA retain their original split of $4$ encoder and $8$ decoder blocks, whereas
\dispersive{} Loss and \methodname{} use $8$ encoder and $4$ decoder blocks.
We split RCG's training budget between its generator and prior, training for $200$ epochs each.
We also evaluate RCG with both its original prior architecture and ours: the original architecture performs better,
so we report that result.

\textbf{Evaluation}.
We evaluate generation with FID and Inception Score (IS) under the ADM protocol
\citep{dhariwal2021diffusion} on $50$K samples. We evaluate representations using ImageNet
top-1 accuracy (\emph{Acc.}) and the average accuracy across eight datasets for transfer learning (\emph{transfer-8}; datasets listed in Appendix~\ref{app:evaluation}). For both probes, we train a logistic-regression classifier on the
frozen representations produced by the model.
For the baselines, the representation is the mean pooling of patch token embeddings at the encoder--decoder split, or after the last block for vanilla JiT and RCG;
for \methodname{}, we use its \cls{} token at its own split.
Training and evaluation details are in Appendix~\ref{app:training}, and architecture details in Appendix~\ref{app:arch}.

\subsection{Controlled matched-recipe comparison}
\label{sec:exp_main}

Table~\ref{tab:main_unguided} compares all methods under the matched recipe.
\methodname{} achieves the best performance on both generation and representation capabilities.
On unconditional generation, it achieves $8.92$ FID, greatly improving on the best self-contained
baseline ($54.14$). It also improves on the strongest alignment baseline, iREPA ($46.44$), and
on RCG ($14.27$), both of which import semantics from a pretrained DINOv2 encoder. Moreover,
\methodname{} surpasses the class-conditional baselines, whose best unguided FID is $24.93$,
without using labels or an external encoder.
Notably, after only $200$ epochs, \methodname{} already leads every matched baseline and surpasses
RCG's $400$-epoch-equivalent checkpoint on all four metrics, despite RCG's additional frozen
prior-training stage (Table~\ref{tab:controlled_e200}, Appendix~\ref{app:controlled_e200}).

External DINOv2 semantics improve the other baselines in different ways:
iREPA transfers part of them into the generator's own features, reaching $65.11\%$ Acc.\ at
$46.44$ FID, whereas RCG improves generation to $14.27$ FID but reaches only $6.54\%$ Acc.\ and
$30.28\%$ transfer-8.
\methodname{} achieves $72.35\%$ Acc.\ and $79.89\%$ transfer-8, exceeding every baseline.
Together with Table~\ref{tab:deps}, these results show that \methodname{} is the only controlled
method satisfying all four ``\emph{S}'' requirements while also achieving the best generation and
representation results.

\begin{table*}[t]
\begin{minipage}[t]{0.54\textwidth}
\centering
\caption{\textbf{Controlled comparison on ImageNet $256\times256$.} All rows use the JiT-B/16
recipe for $400$ epochs.}
\label{tab:main_unguided}
\scriptsize\setlength{\tabcolsep}{2pt}
\resizebox{\linewidth}{!}{\begin{tabular}{lcrrrr}
\toprule
& & \multicolumn{2}{c}{\emph{generation}} & \multicolumn{2}{c}{\emph{representation}} \\
\cmidrule(r){3-4}\cmidrule(l){5-6}
Method & pretrained enc. & FID$\downarrow$ & IS$\uparrow$ & Acc.$\uparrow$ & transfer-8$\uparrow$ \\
\midrule
\multicolumn{6}{l}{\emph{class-conditional}} \\
JiT                        & \no  & 24.93 & 63.6 & 5.30 & 30.10 \\
JiT $+$ \dispersive{}     & \no  & 26.73 & 58.7 & 7.79 & 20.72 \\
JiT $+$ REPA               & \yes & 27.92 & 58.0 & 60.80 & 69.01 \\
JiT $+$ iREPA              & \yes & 27.01 & 62.1 & 61.86 & 74.14 \\
\midrule
\multicolumn{6}{l}{\emph{class-unconditional}} \\
JiT                        & \no  & 55.40 & 19.2 & 11.27 & 26.84 \\
JiT $+$ \dispersive{}     & \no  & 54.14 & 19.7 & 7.48 & 22.33 \\
JiT $+$ REPA               & \yes & 49.73 & 24.6 & 63.57 & 72.62 \\
JiT $+$ iREPA              & \yes & 46.44 & 26.2 & 65.11 & 75.43 \\
RCG (JiT)                  & \yes & 14.27 & 93.4 & 6.54 & 30.28 \\
\midrule
\rowcolor{gray!12}
\methodname{}       & \no & \textbf{8.92} & \textbf{118.4} & \textbf{72.35} & \textbf{79.89} \\
\bottomrule
\end{tabular}}
\end{minipage}\hfill
\begin{minipage}[t]{0.435\textwidth}
\centering
\caption{\textbf{Comparison with published unconditional pixel-space systems on ImageNet $256\times256$.} $\dagger$ retrieves real
images; RG denotes representation guidance; SCION (B/16) uses 800 epochs on JiT-B/16 and SCION (L/16) uses 400 epochs on JiT-L/16.}
\vspace*{-0.05in}
\label{tab:system_pixel}
\scriptsize\setlength{\tabcolsep}{2pt}
\resizebox{\linewidth}{!}{\begin{tabular}{lcrrr}
\toprule
Method & pretrained enc. & params & FID$\downarrow$ & IS$\uparrow$ \\
\midrule
BigGAN       & \no & $\sim$70M & 38.61 & 24.7 \\
ADM          & \no & 554M & 26.21 & 39.7 \\
RCDM$^\dagger$ & \yes\ DINO & --- & 19.0 & 51.9 \\
IC-GAN$^\dagger$ & \yes\ SwAV & $\sim$75M & 15.6 & 59.0 \\
RCG (ADM)    & \yes\ MoCo v3 & $\sim$617M & 6.24 & 136.9 \\
\midrule
\rowcolor{gray!12}
\methodname{} (B/16)                        & \no & 231M & 7.67 & 127.6 \\
\rowcolor{gray!12}
\methodname{} (B/16, RG)                  & \no & 231M & 3.51 & 178.7 \\
\rowcolor{gray!12}
\methodname{} (L/16)                       & \no & 559M & 5.89 & 151.6 \\
\rowcolor{gray!12}
\methodname{} (L/16, RG)                  & \no & 559M & \textbf{3.47} & \textbf{201.7} \\
\bottomrule
\end{tabular}}
\end{minipage}
\end{table*}

\textbf{Representation guidance.}
\methodname{}'s learned representation can also replace labels as a guidance signal
(Equation~\ref{eq:cfg}).
Table~\ref{tab:guided} compares \methodname{}'s generation performance under representation guidance.
\methodname{} reaches $4.20$ FID, improving on RCG's $5.01$ while conditioning on its own learned
representation and sampling it from a jointly learned prior.
The class-guided JiT and REPA rows are provided for reference: they use ground-truth labels at sampling time, which \methodname{} never observes.

\begin{table*}[t]
\begin{minipage}[t]{0.42\textwidth}
\centering
\caption{\textbf{Guided generation without label supervision.} Each row is evaluated at its
independently swept guidance scale. Rows in gray text use ground-truth class labels and are included for
reference.}
\label{tab:guided}
\scriptsize\setlength{\tabcolsep}{2pt}
\resizebox{\linewidth}{!}{\begin{tabular}{lcrr}
\toprule
Method & pretrained enc. & FID$\downarrow$ & IS$\uparrow$ \\
\midrule
\multicolumn{4}{l}{\emph{guided on a representation}} \\
\methodname{} (ours) & \no & \textbf{4.20} & 162.1 \\
RCG (JiT)            & \yes & 5.01 & \textbf{164.1} \\
\midrule
\multicolumn{4}{l}{\emph{\color{black!50}guided on a class label, for reference}} \\
\color{black!50}JiT          & \color{black!50}\no  & \color{black!50}3.94 & \color{black!50}266.7 \\
\color{black!50}JiT $+$ REPA & \color{black!50}\yes & \color{black!50}3.57 & \color{black!50}286.9 \\
\bottomrule
\end{tabular}}
\end{minipage}\hfill
\begin{minipage}[t]{0.55\textwidth}
\centering
\caption{\textbf{Building \methodname{} from unconditional flow matching.} Cumulative JiT-B/16
construction at a matched $200$-epoch budget. Each row adds the stated component. CL-only is a separately trained contrastive-only control and
has no generative objective.}
\label{tab:abl_construction}
\scriptsize\setlength{\tabcolsep}{2pt}
\resizebox{\linewidth}{!}{\begin{tabular}{lrrr}
\toprule
Components & FID$\downarrow$ & Acc.$\uparrow$ & transfer-8$\uparrow$ \\
\midrule
\color{black!55}CL only & \color{black!55}--- & \color{black!55}71.76 & \color{black!55}77.68 \\
\midrule
FM & 61.17 & 10.65 & 26.19 \\
$\quad+$ CL & 57.40 & 65.43 & 76.69 \\
$\quad+$ gradient-norm balancing & 54.79 & 65.49 & 75.95 \\
$\quad+$ self-conditioning $+$ prior & 11.66 & 70.97 & 78.86 \\
\rowcolor{gray!12}
$\quad+$ CBC \quad (\textbf{\methodname{}}) & \textbf{10.10} & \textbf{71.32} & \textbf{79.73} \\
\bottomrule
\end{tabular}}
\end{minipage}
\end{table*}

\subsection{System-level comparison}
\label{sec:exp_system}

Table~\ref{tab:system_pixel} compares \methodname{} at two different model scales with published unconditional pixel-space
systems: BigGAN~\citep{donahue2019large}, ADM~\citep{dhariwal2021diffusion},
RCDM~\citep{bordes2021high}, IC-GAN~\citep{casanova2021instance}, and
RCG~\citep{NEURIPS2024_e304d374}.  SCION (B/16) uses JiT-B/16 as the backbone, while SCION (L/16) uses JiT-L/16.
Unlike the controlled comparison, these reported results span different architectures,
training budgets, and protocols, and should therefore be read as a system-level reference.
At $231$M parameters, with self-representation guidance \methodname{} (B/16) reaches $3.51$ FID and $178.7$ IS after $800$ epochs, which are better than RCG (ADM), despite using less than half of RCG's $\sim$617M inference-time parameters and requiring no pretrained encoder. 
For \methodname{} (L/16) with $559$M parameters, which remains slightly smaller than RCG (ADM), we only train for $400$ epochs. Under self-representation guidance it reaches $3.47$ FID and $201.7$ IS, significantly outperforming every
previously published pixel-space system in the table. Moreover, representation guidance improves SCION over unguided variants.
Table~\ref{tab:system_full} in
Appendix~\ref{app:add_exps} provides the broader comparison including latent-space systems.

\subsection{Ablations}
\label{sec:exp_ablation}

To isolate the contribution of each component, we use the same JiT-B/16 recipe as the controlled
comparison and train every model for $200$ epochs.

\textbf{From flow matching to \methodname{}.}
We decompose \methodname{} by sequentially adding its components to unconditional flow matching.
Table~\ref{tab:abl_construction} shows that contrastive learning first makes the representation
useful, increasing Acc.\ from $10.65\%$ to $65.43\%$ and transfer-8 from $26.19\%$ to
$76.69\%$, while slightly improving FID from $61.17$ to $57.40$. Gradient-norm balancing further reduces FID to $54.79$ while maintaining representation
quality and avoiding
manual selection of the contrastive weight. Adding representation self-conditioning with the
jointly learned prior converts these representation gains into substantially better generation,
reducing FID to $11.66$. Finally, CBC completes \methodname{} and reaches $10.10$ FID, with
$71.32\%$ Acc.\ and $79.73\%$ transfer-8.
Appendix~\ref{app:detailed_ablations} provides a more detailed ladder.

\textbf{Comparison with contrastive-only training.}
Table~\ref{tab:abl_construction} also compares with a contrastive-only control using the same $8$-layer encoder, evaluated at $200$ epochs. Its learning rate, temperature, and weight decay are tuned independently for contrastive learning.
\methodname{} reaches $71.32\%$ Acc., close to the control's $71.76\%$, while
attaining higher transfer-8 accuracy ($79.73\%$ versus $77.68\%$). Thus, joint training retains
representation quality comparable to dedicated contrastive learning while also learning to
generate. Appendix~\ref{app:detailed_ablations} repeats this comparison at $200$, $400$, and $600$
epochs.

\textbf{Additional ablations: Qualitative results, and generator vs. prior.}
Appendix~\ref{app:qualitative} shows generations from real-image representations and interpolations
between them.
Appendix~\ref{app:prior} shows that conditioning on the prior instead of the encoder's representations raises precision and lowers recall.

\section{Conclusion}
We presented SCION, a framework that jointly learns image generation and semantic representations in a single pixel-space model, trained from scratch without labels or pretrained components. By integrating contrastive learning with flow matching, representation self-conditioning, and a jointly learned embedding prior, SCION supports both representation extraction and unconditional image generation. Experiments on ImageNet demonstrate substantial improvements over matched generative baselines while retaining representation quality competitive with dedicated contrastive learning. These results show that strong generation and useful representations can be learned together in a self-contained, self-supervised training process. Future work will explore downstream applications and extensions to domains where labels and pretrained models are scarce.

\subsection*{Ethics statement}

Like other image-generation methods, \methodname{} could be used to
produce misleading or harmful imagery. We consider the potential for
misuse of the models studied here to be limited by their current
image fidelity, $256\times256$ resolution, and class-unconditional
sampling without text-based or explicit identity control. These
limitations do not eliminate risk: generated images and learned
representations may reflect biases or sensitive content in the
training data, and we have not evaluated training-image memorization
or demographic fairness. By removing the need for labels and
pretrained components, the method also lowers the barrier to training
on new domains and could enable more capable models. Such extensions
would require assessment of risks appropriate to their capabilities
and intended use.

Training uses multiple NVIDIA H200 GPUs
(Appendix~\ref{app:training}), with an associated energy cost.
Our single-stage approach avoids separate pretraining of encoders,
autoencoders, or tokenizers, although we do not establish a reduction
in total compute or energy consumption relative to other approaches.

\bibliography{refs}

@Techreport{krizhevskyLearningMultipleLayers2009,
  author = {Krizhevsky, Alex and Hinton, Geoffrey},
  address = {Toronto, Ontario},
  institution = {University of Toronto},
  publisher = {Technical report, University of Toronto},
  title = {Learning multiple layers of features from tiny images},
  year = {2009},
  title_with_no_special_chars = {Learning multiple layers of features from tiny images},
  url = {https://www.cs.toronto.edu/~kriz/learning-features-2009-TR.pdf}
}

@inproceedings{schuhmann2022laion5b,
  author = {Schuhmann, Christoph and Beaumont, Romain and Vencu, Richard and Gordon, Cade and Wightman, Ross and Cherti, Mehdi and Coombes, Theo and Katta, Aarush and Mullis, Clayton and Wortsman, Mitchell and Schramowski, Patrick and Kundurthy, Srivatsa and Crowson, Katherine and Schmidt, Ludwig and Kaczmarczyk, Robert and Jitsev, Jenia},
  title = {{LAION-5B}: an open large-scale dataset for training next generation image-text models},
  year = {2022},
  isbn = {9781713871088},
  publisher = {Curran Associates Inc.},
  address = {Red Hook, NY, USA},
  booktitle = {Proceedings of the 36th International Conference on Neural Information Processing Systems},
  articleno = {1833},
  numpages = {17},
  location = {New Orleans, LA, USA},
  series = {NIPS '22}
}

@article{scikit-learn,
  title={Scikit-learn: Machine Learning in {P}ython},
  author={Pedregosa, F. and Varoquaux, G. and Gramfort, A. and Michel, V.
          and Thirion, B. and Grisel, O. and Blondel, M. and Prettenhofer, P.
          and Weiss, R. and Dubourg, V. and Vanderplas, J. and Passos, A. and
          Cournapeau, D. and Brucher, M. and Perrot, M. and Duchesnay, E.},
  journal={Journal of Machine Learning Research},
  volume={12},
  pages={2825--2830},
  year={2011}
}

@inproceedings{52203mage,
  title = {{{MAGE}}: {{Masked}} Generative Encoder to Unify Representation Learning and Image Synthesis},
  booktitle = {{{CVPR}}},
  author = {Li, Tianhong and Chang, Huiwen and Mishra, Shlok Kumar and Zhang, Han and Katabi, Dina and Krishnan, Dilip},
  year = 2023,
  url = {https://arxiv.org/pdf/2211.09117.pdf}
}

@inproceedings{bao2021beit,
  title={{BEiT}: {BERT} Pre-Training of Image Transformers},
  author={Hangbo Bao and Li Dong and Songhao Piao and Furu Wei},
  booktitle={International Conference on Learning Representations},
  year={2022},
  url={https://openreview.net/forum?id=p-BhZSz59o4}
}

@inproceedings{bao2022conditional,
  title={Why Are Conditional Generative Models Better Than Unconditional Ones?},
  author={Fan Bao and Chongxuan Li and Jiacheng Sun and Jun Zhu},
  booktitle={NeurIPS 2022 Workshop on Score-Based Methods},
  year={2022},
  url={https://openreview.net/forum?id=sbDyvrvvKn7}
}

@article{blattmann2022retrieval,
  title = {Retrieval-Augmented Diffusion Models},
  author = {Blattmann, Andreas and Rombach, Robin and Oktay, Kaan and M{\"u}ller, Jonas and Ommer, Bj{\"o}rn},
  year = 2022,
  journal = {NeurIPS},
  volume = {35},
  pages = {15309--15324}
}

@article{bordes2021high,
  title={High Fidelity Visualization of What Your Self-Supervised Representation Knows About},
  author={Florian Bordes and Randall Balestriero and Pascal Vincent},
  journal={Transactions on Machine Learning Research},
  issn={2835-8856},
  year={2022},
  url={https://openreview.net/forum?id=urfWb7VjmL},
  note={}
}

@inproceedings{bossard14,
  title = {Food-101 -- Mining Discriminative Components with Random Forests},
  booktitle = {European Conference on Computer Vision},
  author = {Bossard, Lukas and Guillaumin, Matthieu and Van Gool, Luc},
  year = 2014
}

@inproceedings{brock2019large,
  title = {Large Scale {{GAN}} Training for High Fidelity Natural Image Synthesis},
  booktitle = {{{ICLR}}},
  author = {Brock, Andrew and Donahue, Jeff and Simonyan, Karen},
  year = 2019
}

@inproceedings{caron2021emerging,
  title = {Emerging Properties in Self-Supervised Vision {{Transformers}}},
  booktitle = {{{ICCV}}},
  author = {Caron, Mathilde and Touvron, Hugo and Misra, Ishan and J{\'e}gou, Herv{\'e} and Mairal, Julien and Bojanowski, Piotr and Joulin, Armand},
  year = 2021,
  pages = {9650--9660}
}

@article{casanova2021instance,
  title = {Instance-Conditioned {GAN}},
  author = {Casanova, Arantxa and Careil, Marlene and Verbeek, Jakob and Drozdzal, Michal and Romero Soriano, Adriana},
  year = 2021,
  journal = {NeurIPS},
  volume = {34},
  pages = {27517--27529}
}

@inproceedings{Chang2022,
  title = {{{MaskGIT}}: {{Masked}} Generative Image {{Transformer}}},
  booktitle = {{{CVPR}}},
  author = {Chang, Huiwen and Zhang, Han and Jiang, Lu and Liu, Ce and Freeman, William T},
  year = 2022
}

@inproceedings{chen_simple_2020,
  title = {A Simple Framework for Contrastive Learning of Visual Representations},
  booktitle = {Proceedings of the 37th International Conference on Machine Learning, {{ICML}} 2020, 13-18 July 2020, Virtual Event},
  author = {Chen, Ting and Kornblith, Simon and Norouzi, Mohammad and Hinton, Geoffrey E.},
  year = 2020,
  series = {Proceedings of Machine Learning Research},
  pages = {1597--1607},
  publisher = {PMLR},
  url = {http://proceedings.mlr.press/v119/chen20j.html}
}

@article{chen2025pixelflow,
  title = {{{PixelFlow}}: {{Pixel-space}} Generative Models with Flow},
  author = {Chen, Shoufa and Ge, Chongjian and Zhang, Shilong and Sun, Peize and Luo, Ping},
  year = 2025,
  journal = {arXiv:2504.07963},
  eprint = {2504.07963},
  archiveprefix = {arXiv}
}

@misc{chenBigSelfSupervisedModels2020,
  title = {Big {{Self-Supervised Models}} Are {{Strong Semi-Supervised Learners}}},
  author = {Chen, Ting and Kornblith, Simon and Swersky, Kevin and Norouzi, Mohammad and Hinton, Geoffrey},
  year = 2020,
  month = oct,
  number = {arXiv:2006.10029},
  eprint = {2006.10029},
  primaryclass = {cs, stat},
  publisher = {arXiv},
  url = {http://arxiv.org/abs/2006.10029},
  urldate = {2023-09-26},
  archiveprefix = {arXiv}
}

@inproceedings{chenDeconstructingDenoisingDiffusion2024,
  title = {Deconstructing Denoising Diffusion Models for Self-Supervised Learning},
  booktitle = {International Conference on Learning Representations},
  author = {Chen, Xinlei and Liu, Zhuang and Xie, Saining and He, Kaiming},
  year = 2025,
  volume = {2025},
  pages = {55458--55472},
  url = {https://proceedings.iclr.cc/paper_files/paper/2025/file/8b52c29d1a77e3fcc6c8f90fa0ac0e8e-Paper-Conference.pdf}
}

@inproceedings{cimpoi14describing,
  author={Cimpoi, Mircea and Maji, Subhransu and Kokkinos, Iasonas and Mohamed, Sammy and Vedaldi, Andrea},
  booktitle={2014 IEEE Conference on Computer Vision and Pattern Recognition},
  title={Describing Textures in the Wild},
  year={2014},
  volume={},
  number={},
  pages={3606-3613},
  doi={10.1109/CVPR.2014.461}}

@misc{dao_flow_2023,
  title = {Flow {{Matching}} in {{Latent Space}}},
  author = {Dao, Quan and Phung, Hao and Nguyen, Binh and Tran, Anh},
  year = 2023,
  month = jul,
  doi = {10.48550/arXiv.2307.08698},
  url = {http://arxiv.org/abs/2307.08698},
  urldate = {2025-10-18}
}

@inproceedings{DBLP:conf/iclr/AlbergoV23,
  title = {Building Normalizing Flows with Stochastic Interpolants},
  booktitle = {The Eleventh International Conference on Learning Representations, {{ICLR}} 2023, Kigali, Rwanda, May 1-5, 2023},
  author = {Albergo, Michael S. and {Vanden-Eijnden}, Eric},
  year = 2023,
  publisher = {OpenReview.net},
  url = {https://openreview.net/forum?id=li7qeBbCR1t},
  bibsource = {dblp computer science bibliography, https://dblp.org}
}

@inproceedings{deng2009imagenet,
  title = {{{ImageNet}}: {{A}} Large-Scale Hierarchical Image Database},
  booktitle = {{{CVPR}}},
  author = {Deng, Jia and Dong, Wei and Socher, Richard and Li, Li-Jia and Li, Kai and {Fei-Fei}, Li},
  year = 2009
}

@inproceedings{dhariwal2021diffusion,
  title = {Diffusion Models Beat {{GANs}} on Image Synthesis},
  booktitle = {{{NeurIPS}}},
  author = {Dhariwal, Prafulla and Nichol, Alexander},
  year = 2021
}

@article{donahue2019large,
  title = {Large Scale Adversarial Representation Learning},
  author = {Donahue, Jeff and Simonyan, Karen},
  year = 2019,
  journal = {NeurIPS},
  volume = {32}
}

@inproceedings{dosovitskiy2021an,
  title = {An Image Is Worth 16x16 Words: {{Transformers}} for Image Recognition at Scale},
  booktitle = {International Conference on Learning Representations},
  author = {Dosovitskiy, Alexey and Beyer, Lucas and Kolesnikov, Alexander and Weissenborn, Dirk and Zhai, Xiaohua and Unterthiner, Thomas and Dehghani, Mostafa and Minderer, Matthias and Heigold, Georg and Gelly, Sylvain and Uszkoreit, Jakob and Houlsby, Neil},
  year = 2021,
  url = {https://openreview.net/forum?id=YicbFdNTTy}
}

@article{elfwing2018sigmoid,
  title = {Sigmoid-Weighted Linear Units for Neural Network Function Approximation in Reinforcement Learning},
  author = {Elfwing, Stefan and Uchibe, Eiji and Doya, Kenji},
  year = 2018,
  journal = {Neural networks}
}

@inproceedings{Esser2021,
  title = {Taming {{Transformers}} for High-Resolution Image Synthesis},
  booktitle = {{{CVPR}}},
  author = {Esser, Patrick and Rombach, Robin and Ommer, Bjorn},
  year = 2021
}

@inproceedings{esser2024scaling,
  title = {Scaling Rectified Flow Transformers for High-Resolution Image Synthesis},
  booktitle = {Forty-First International Conference on Machine Learning},
  author = {Esser, Patrick and Kulal, Sumith and Blattmann, Andreas and Entezari, Rahim and M{\"u}ller, Jonas and Saini, Harry and Levi, Yam and Lorenz, Dominik and Sauer, Axel and Boesel, Frederic and Podell, Dustin and Dockhorn, Tim and English, Zion and Rombach, Robin},
  year = 2024,
  url = {https://openreview.net/forum?id=FPnUhsQJ5B}
}

@article{gidaris2018unsupervised,
  title = {Unsupervised Representation Learning by Predicting Image Rotations},
  author = {Gidaris, Spyros and Singh, Praveer and Komodakis, Nikos},
  year = 2018,
  journal = {arXiv:1803.07728},
  eprint = {1803.07728},
  archiveprefix = {arXiv}
}

@article{grill2020bootstrap,
  title = {Bootstrap Your Own Latent: {{A}} New Approach to Self-Supervised Learning},
  author = {Grill, Jean-Bastien and Strub, Florian and Altch{\'e}, Florent and Tallec, Corentin and Richemond, Pierre H and Buchatskaya, Elena and Doersch, Carl and Pires, Bernardo Avila and Guo, Zhaohan Daniel and Azar, Mohammad Gheshlaghi and others},
  year = 2020,
  journal = {arXiv:2006.07733},
  eprint = {2006.07733},
  archiveprefix = {arXiv}
}

@inproceedings{he2020momentumcontrastunsupervisedvisual,
  title = {Momentum Contrast for Unsupervised Visual Representation Learning},
  booktitle = {Proceedings of the {{IEEE}}/{{CVF}} Conference on Computer Vision and Pattern Recognition ({{CVPR}})},
  author = {He, Kaiming and Fan, Haoqi and Wu, Yuxin and Xie, Saining and Girshick, Ross},
  year = 2020,
  pages = {9726--9735},
  doi = {10.1109/CVPR42600.2020.00975},
  url = {https://openaccess.thecvf.com/content_CVPR_2020/html/He_Momentum_Contrast_for_Unsupervised_Visual_Representation_Learning_CVPR_2020_paper.html}
}

@inproceedings{he2022masked,
  title = {Masked Autoencoders Are Scalable Vision Learners},
  booktitle = {{{CVPR}}},
  author = {He, Kaiming and Chen, Xinlei and Xie, Saining and Li, Yanghao and Doll{\'a}r, Piotr and Girshick, Ross},
  year = 2022
}

@inproceedings{henry2020query,
  title = {Query-Key Normalization for {{Transformers}}},
  booktitle = {Findings of {{EMNLP}}},
  author = {Henry, Alex and Dachapally, Prudhvi Raj and Pawar, Shubham Shantaram and Chen, Yuxuan},
  year = 2020
}

@article{heun1900neue,
  title = {Neue {{Methoden}} Zur Approximativen {{Integration}} Der {{Differentialgleichungen}} Einer Unabh\"angigen {{Ver\"anderlichen}}},
  author = {Heun, Karl},
  year = 1900,
  journal = {Z. Math. Phys},
  volume = {45},
  pages = {23--38}
}

@article{heusel2017gans,
  title = {{{GANs}} Trained by a Two Time-Scale Update Rule Converge to a Local {{Nash}} Equilibrium},
  author = {Heusel, Martin and Ramsauer, Hubert and Unterthiner, Thomas and Nessler, Bernhard and Hochreiter, Sepp},
  year = 2017,
  journal = {NeurIPS}
}

@inproceedings{Ho2020,
  title = {Denoising Diffusion Probabilistic Models},
  booktitle = {{{NeurIPS}}},
  author = {Ho, Jonathan and Jain, Ajay and Abbeel, Pieter},
  year = 2020
}

@inproceedings{ho2022classifier,
  title = {Classifier-Free Diffusion Guidance},
  booktitle = {{{NeurIPS}} Workshops},
  author = {Ho, Jonathan and Salimans, Tim},
  year = 2021
}

@article{hoogeboom2023simple,
  title = {Simple Diffusion: {{End-to-end}} Diffusion for High Resolution Images},
  author = {Hoogeboom, Emiel and Heek, Jonathan and Salimans, Tim},
  year = 2023,
  journal = {ICML}
}

@inproceedings{hoogeboom2024simpler,
  title = {Simpler {{Diffusion}} ({{SiD2}}): 1.5 {{FID}} on {{ImageNet512}} with Pixel-Space Diffusion},
  booktitle = {{{CVPR}}},
  author = {Hoogeboom, Emiel and Mensink, Thomas and Heek, Jonathan and Lamerigts, Kay and Gao, Ruiqi and Salimans, Tim},
  year = 2025
}

@inproceedings{Hu_2023_CVPR,
  title = {Self-Guided Diffusion Models},
  booktitle = {Proceedings of the {{IEEE}}/{{CVF}} Conference on Computer Vision and Pattern Recognition ({{CVPR}})},
  author = {Hu, Vincent Tao and Zhang, David W. and Asano, Yuki M. and Burghouts, Gertjan J. and Snoek, Cees G. M.},
  year = 2023,
  month = jun,
  pages = {18413--18422}
}

@article{imagenet15russakovsky,
  title = {{{ImageNet Large Scale Visual Recognition Challenge}}},
  author = {Russakovsky, Olga and Deng, Jia and Su, Hao and Krause, Jonathan and Satheesh, Sanjeev and Ma, Sean and Huang, Zhiheng and Karpathy, Andrej and Khosla, Aditya and Bernstein, Michael and Berg, Alexander C. and {Fei-Fei}, Li},
  year = 2015,
  journal = {International Journal of Computer Vision (IJCV)},
  volume = {115},
  number = {3},
  pages = {211--252},
  doi = {10.1007/s11263-015-0816-y}
}

@inproceedings{kouzelis2026boostinggenerativeimagemodeling,
  title={Boosting Generative Image Modeling via Joint Image-Feature Synthesis},
  author={Theodoros Kouzelis and Efstathios Karypidis and Ioannis Kakogeorgiou and Spyros Gidaris and Nikos Komodakis},
  booktitle={The Thirty-ninth Annual Conference on Neural Information Processing Systems},
  year={2025},
  url={https://openreview.net/forum?id=i4qAfV04rZ}
}

@misc{Krasin2017,
  author = {Krasin, Ivan and
            Duerig, Tom and
            Alldrin, Neil and
            Ferrari, Vittorio
            and Abu-El-Haija, Sami and
            Kuznetsova, Alina and
            Rom, Hassan and
            Uijlings, Jasper and
            Popov, Stefan and
            Kamali, Shahab and
            Malloci, Matteo and
            Pont-Tuset, Jordi and
            Veit, Andreas and
            Belongie, Serge and
            Gomes, Victor and
            Gupta, Abhinav and
            Sun, Chen and
            Chechik, Gal and
            Cai, David and
            Feng, Zheyun and
            Narayanan, Dhyanesh and
            Murphy, Kevin},
  title = {OpenImages: A public dataset for large-scale multi-label and
           multi-class image classification.},
  journal = {Dataset available from
             https://storage.googleapis.com/openimages/web/index.html},
  year={2017}
}

@inproceedings{Krause2013CollectingAL,
  title = {Collecting a Large-Scale Dataset of Fine-Grained Cars},
  booktitle = {Second {{Workshop}} on {{Fine-Grained Visual Categorization}}},
  author = {Krause, Jonathan and Deng, Jia and Stark, Michael and {Fei-Fei}, Li},
  year = 2013
}

@article{Kynkaanniemi2019,
  title = {Improved Precision and Recall Metric for Assessing Generative Models},
  author = {Kynk{\"a}{\"a}nniemi, Tuomas and Karras, Tero and Laine, Samuli and Lehtinen, Jaakko and Aila, Timo},
  year = 2019,
  journal = {NeurIPS}
}

@inproceedings{kynkaanniemi2024applying,
  title = {Applying Guidance in a Limited Interval Improves Sample and Distribution Quality in Diffusion Models},
  booktitle = {{{NeurIPS}}},
  author = {Kynk{\"a}{\"a}nniemi, Tuomas and Aittala, Miika and Karras, Tero and Laine, Samuli and Aila, Timo and Lehtinen, Jaakko},
  year = 2024
}

@inproceedings{lavoie2025compositionaldiscretelatentcode,
  title={Compositional Discrete Latent Code for High Fidelity, Productive Diffusion Models},
  author={Samuel Lavoie and Michael Noukhovitch and Aaron Courville},
  booktitle={The Thirty-ninth Annual Conference on Neural Information Processing Systems},
  year={2025},
  url={https://openreview.net/forum?id=1iSnpztjbD}
}

@misc{li_andreeto_ranzato_perona_2022,
  title = {Caltech 101},
  author = {Li, Fei-Fei and Andreeto, Marco and Ranzato, Marc'Aurelio and Perona, Pietro},
  year = 2022,
  month = apr,
  publisher = {CaltechDATA},
  doi = {10.22002/D1.20086}
}

@inproceedings{li2024autoregressive,
  title = {Autoregressive Image Generation without Vector Quantization},
  booktitle = {{{NeurIPS}}},
  author = {Li, Tianhong and Tian, Yonglong and Li, He and Deng, Mingyang and He, Kaiming},
  year = 2024
}

@inproceedings{li2025jit,
  title = {Back to Basics: {{Let}} Denoising Generative Models Denoise},
  author = {Li, Tianhong and He, Kaiming},
  year = 2026,
  booktitle = {CVPR},
}

@misc{lin2025latentzoningnetworkunified,
  title = {Latent Zoning Network: A Unified Principle for Generative Modeling, Representation Learning, and Classification},
  author = {Lin, Zinan and Liu, Enshu and Ning, Xuefei and Zhu, Junyi and Wang, Wenyu and Yekhanin, Sergey},
  year = 2025,
  eprint = {2509.15591},
  primaryclass = {cs.LG},
  url = {https://arxiv.org/abs/2509.15591},
  archiveprefix = {arXiv}
}

@inproceedings{lipman_flow_2023,
  title = {Flow Matching for Generative Modeling},
  booktitle = {The Eleventh International Conference on Learning Representations, {{ICLR}} 2023, Kigali, Rwanda, May 1-5, 2023},
  author = {Lipman, Yaron and Chen, Ricky T. Q. and {Ben-Hamu}, Heli and Nickel, Maximilian and Le, Matthew},
  year = 2023,
  publisher = {OpenReview.net},
  url = {https://openreview.net/forum?id=PqvMRDCJT9t}
}

@inproceedings{liu2020diverse,
  title = {Diverse Image Generation via Self-Conditioned {{GANs}}},
  booktitle = {Proceedings of the {{IEEE}}/{{CVF}} Conference on Computer Vision and Pattern Recognition},
  author = {Liu, Steven and Wang, Tongzhou and Bau, David and Zhu, Jun-Yan and Torralba, Antonio},
  year = 2020,
  pages = {14286--14295}
}

@inproceedings{liu2023flow,
  title = {Flow Straight and Fast: {{Learning}} to Generate and Transfer Data with Rectified Flow},
  booktitle = {The Eleventh International Conference on Learning Representations},
  author = {Liu, Xingchao and Gong, Chengyue and Liu, Qiang},
  year = 2023,
  url = {https://openreview.net/forum?id=XVjTT1nw5z}
}

@misc{liu2025efficientgenerativemodeltraining,
  title = {Efficient Generative Model Training via Embedded Representation Warmup},
  author = {Liu, Deyuan and Sun, Peng and Li, Xufeng and Lin, Tao},
  year = 2025,
  eprint = {2504.10188},
  primaryclass = {cs.LG},
  url = {https://arxiv.org/abs/2504.10188},
  archiveprefix = {arXiv}
}

@inproceedings{Loshchilov2019,
  title={Decoupled Weight Decay Regularization},
  author={Ilya Loshchilov and Frank Hutter},
  booktitle={International Conference on Learning Representations},
  year={2019},
  url={https://openreview.net/forum?id=Bkg6RiCqY7},
}

@inproceedings{luvcic2019high,
  title = {High-Fidelity Image Generation with Fewer Labels},
  booktitle = {International Conference on Machine Learning},
  author = {Lucic, Mario and Tschannen, Michael and Ritter, Marvin and Zhai, Xiaohua and Bachem, Olivier and Gelly, Sylvain},
  year = 2019,
  pages = {4183--4192},
  publisher = {PMLR}
}

@inproceedings{ma2024sit,
  title = {{{SiT}}: {{Exploring}} Flow and Diffusion-Based Generative Models with Scalable Interpolant {{Transformers}}},
  booktitle = {{{ECCV}}},
  author = {Ma, Nanye and Goldstein, Mark and Albergo, Michael S and Boffi, Nicholas M and {Vanden-Eijnden}, Eric and Xie, Saining},
  year = 2024
}

@article{mirza2014conditional,
  title = {Conditional Generative Adversarial Nets},
  author = {Mirza, Mehdi and Osindero, Simon},
  year = 2014,
  journal = {arXiv:1411.1784},
  eprint = {1411.1784},
  archiveprefix = {arXiv}
}

@article{nash2021generating,
  title = {Generating Images with Sparse Representations},
  author = {Nash, Charlie and Menick, Jacob and Dieleman, Sander and Battaglia, Peter W},
  year = 2021,
  journal = {arXiv:2103.03841},
  eprint = {2103.03841},
  archiveprefix = {arXiv}
}

@inproceedings{NEURIPS2024_e304d374,
  title = {Return of Unconditional Generation: A Self-Supervised Representation Generation Method},
  booktitle = {Advances in Neural Information Processing Systems},
  author = {Li, Tianhong and Katabi, Dina and He, Kaiming},
  editor = {Globerson, A. and Mackey, L. and Belgrave, D. and Fan, A. and Paquet, U. and Tomczak, J. and Zhang, C.},
  year = 2024,
  volume = {37},
  pages = {125441--125468},
  publisher = {Curran Associates, Inc.},
  doi = {10.52202/079017-3985},
  url = {https://proceedings.neurips.cc/paper_files/paper/2024/file/e304d374c85e385eb217ed4a025b6b63-Paper-Conference.pdf}
}

@inproceedings{Nichol2021,
  title = {Improved Denoising Diffusion Probabilistic Models},
  booktitle = {{{ICML}}},
  author = {Nichol, Alexander Quinn and Dhariwal, Prafulla},
  year = 2021
}

@inproceedings{Nilsback08,
  title = {Automated Flower Classification over a Large Number of Classes},
  booktitle = {Indian Conference on Computer Vision, Graphics and Image Processing},
  author = {Nilsback, Maria-Elena and Zisserman, Andrew},
  year = 2008,
  month = dec
}

@inproceedings{NIPS2017_3f5ee243,
  title = {Attention Is All You Need},
  booktitle = {Advances in Neural Information Processing Systems},
  author = {Vaswani, Ashish and Shazeer, Noam and Parmar, Niki and Uszkoreit, Jakob and Jones, Llion and Gomez, Aidan N and Kaiser, Lukasz and Polosukhin, Illia},
  editor = {Guyon, I. and Luxburg, U. Von and Bengio, S. and Wallach, H. and Fergus, R. and Vishwanathan, S. and Garnett, R.},
  year = 2017,
  volume = {30},
  publisher = {Curran Associates, Inc.}
}

@inproceedings{noroozi2016unsupervised,
  title = {Unsupervised Learning of Visual Representations by Solving Jigsaw Puzzles},
  booktitle = {European Conference on Computer Vision},
  author = {Noroozi, Mehdi and Favaro, Paolo},
  year = 2016,
  pages = {69--84},
  publisher = {Springer}
}

@article{oord2018representation,
  title = {Representation Learning with Contrastive Predictive Coding},
  author = {van den Oord, Aaron and Li, Yazhe and Vinyals, Oriol},
  year = 2018,
  journal = {arXiv:1807.03748},
  eprint = {1807.03748},
  archiveprefix = {arXiv}
}

@article{oquab2023dinov2,
  title={{DINO}v2: Learning Robust Visual Features without Supervision},
  author={Maxime Oquab and Timoth{\'e}e Darcet and Th{\'e}o Moutakanni and Huy V. Vo and Marc Szafraniec and Vasil Khalidov and Pierre Fernandez and Daniel HAZIZA and Francisco Massa and Alaaeldin El-Nouby and Mido Assran and Nicolas Ballas and Wojciech Galuba and Russell Howes and Po-Yao Huang and Shang-Wen Li and Ishan Misra and Michael Rabbat and Vasu Sharma and Gabriel Synnaeve and Hu Xu and Herve Jegou and Julien Mairal and Patrick Labatut and Armand Joulin and Piotr Bojanowski},
  journal={Transactions on Machine Learning Research},
  issn={2835-8856},
  year={2024},
  url={https://openreview.net/forum?id=a68SUt6zFt},
  note={Featured Certification}
}

@inproceedings{parkhi12a,
  title = {Cats and Dogs},
  booktitle = {{{IEEE}} Conference on Computer Vision and Pattern Recognition},
  author = {Parkhi, Omkar M. and Vedaldi, Andrea and Zisserman, Andrew and Jawahar, C. V.},
  year = 2012
}

@inproceedings{peebles_scalable_2023,
  title = {Scalable Diffusion Models with Transformers},
  booktitle = {Proceedings of the {{IEEE}}/{{CVF}} International Conference on Computer Vision ({{ICCV}})},
  author = {Peebles, William and Xie, Saining},
  year = 2023,
  pages = {4195--4205},
  doi = {10.1109/ICCV51070.2023.00387},
  url = {https://openaccess.thecvf.com/content/ICCV2023/html/Peebles_Scalable_Diffusion_Models_with_Transformers_ICCV_2023_paper.html}
}

@inproceedings{preechakul2022diffusion,
  title = {Diffusion Autoencoders: {{Toward}} a Meaningful and Decodable Representation},
  booktitle = {Proceedings of the {{IEEE}}/{{CVF}} Conference on Computer Vision and Pattern Recognition},
  author = {Preechakul, Konpat and Chatthee, Nattanat and Wizadwongsa, Suttisak and Suwajanakorn, Supasorn},
  year = 2022,
  pages = {10619--10629}
}

@InProceedings{radfordLearningTransferableVisual2021,
  title = 	 {Learning Transferable Visual Models From Natural Language Supervision},
  author =       {Radford, Alec and Kim, Jong Wook and Hallacy, Chris and Ramesh, Aditya and Goh, Gabriel and Agarwal, Sandhini and Sastry, Girish and Askell, Amanda and Mishkin, Pamela and Clark, Jack and Krueger, Gretchen and Sutskever, Ilya},
  booktitle = 	 {Proceedings of the 38th International Conference on Machine Learning},
  pages = 	 {8748--8763},
  year = 	 {2021},
  editor = 	 {Meila, Marina and Zhang, Tong},
  volume = 	 {139},
  series = 	 {Proceedings of Machine Learning Research},
  month = 	 {18--24 Jul},
  publisher =    {PMLR},
  url = 	 {https://proceedings.mlr.press/v139/radford21a.html},
}

@inproceedings{rombachHighResolutionImageSynthesis2022,
  title = {High-Resolution Image Synthesis with Latent Diffusion Models},
  booktitle = {Proceedings of the {{IEEE}}/{{CVF}} Conference on Computer Vision and Pattern Recognition ({{CVPR}})},
  author = {Rombach, Robin and Blattmann, Andreas and Lorenz, Dominik and Esser, Patrick and Ommer, Bj{\"o}rn},
  year = 2022,
  pages = {10684--10695},
  doi = {10.1109/CVPR52688.2022.01042},
  url = {https://openaccess.thecvf.com/content/CVPR2022/html/Rombach_High-Resolution_Image_Synthesis_With_Latent_Diffusion_Models_CVPR_2022_paper.html}
}

@article{salimans2016improved,
  title = {Improved Techniques for Training {{GANs}}},
  author = {Salimans, Tim and Goodfellow, Ian and Zaremba, Wojciech and Cheung, Vicki and Radford, Alec and Chen, Xi},
  year = 2016,
  journal = {NeurIPS},
  volume = {29}
}

@article{shazeer2020glu,
  title = {{{GLU}} Variants Improve {{Transformer}}},
  author = {Shazeer, Noam},
  year = 2020,
  journal = {arXiv:2002.05202},
  eprint = {2002.05202},
  archiveprefix = {arXiv}
}

@article{simonyan2014very,
  title = {Very Deep Convolutional Networks for Large-Scale Image Recognition},
  author = {Simonyan, Karen and Zisserman, Andrew},
  year = 2014,
  journal = {arXiv:1409.1556},
  eprint = {1409.1556},
  archiveprefix = {arXiv}
}

@inproceedings{singh_what_2025,
  title = {What Matters for Representation Alignment: {{Global}} Information or Spatial Structure?},
  booktitle = {The Fourteenth International Conference on Learning Representations},
  author = {Singh, Jaskirat and Leng, Xingjian and Wu, Zongze and Zheng, Liang and Zhang, Richard and Shechtman, Eli and Xie, Saining},
  year = 2026,
  url = {https://openreview.net/forum?id=y0UxFtXqXf}
}

@inproceedings{Song2021,
  title = {Score-Based Generative Modeling through Stochastic Differential Equations},
  booktitle = {{{ICLR}}},
  author = {Song, Yang and {Sohl-Dickstein}, Jascha and Kingma, Diederik P and Kumar, Abhishek and Ermon, Stefano and Poole, Ben},
  year = 2021
}

@article{su2024roformer,
  title = {{{RoFormer}}: {{Enhanced Transformer}} with Rotary Position Embedding},
  author = {Su, Jianlin and Ahmed, Murtadha and Lu, Yu and Pan, Shengfeng and Bo, Wen and Liu, Yunfeng},
  year = 2024,
  journal = {Neurocomputing},
  volume = {568},
  pages = {127063},
  publisher = {Elsevier}
}

@inproceedings{tian2023addp,
  title={{ADDP}: Learning General Representations for Image Recognition and Generation with Alternating Denoising Diffusion Process},
  author={Changyao Tian and Chenxin Tao and Jifeng Dai and Hao Li and Ziheng Li and Lewei Lu and Xiaogang Wang and Hongsheng Li and Gao Huang and Xizhou Zhu},
  booktitle={ICLR},
  year={2024},
  url={https://openreview.net/forum?id=cMPm8YFXZe}
}

@inproceedings{wang_repa_2025,
  title={{REPA} Works Until It Doesn{\textquoteright}t: Early-Stopped, Holistic Alignment Supercharges Diffusion Training},
  author={Ziqiao Wang and Wangbo Zhao and Yuhao Zhou and Zekai Li and Zhiyuan Liang and Mingjia Shi and Xuanlei Zhao and Pengfei Zhou and Kaipeng Zhang and Zhangyang Wang and Kai Wang and Yang You},
  booktitle={The Thirty-ninth Annual Conference on Neural Information Processing Systems},
  year={2025},
  url={https://openreview.net/forum?id=HK96GI5s7G}
}

@misc{wang2025diffusedisperseimagegeneration,
  title = {Diffuse and Disperse: {{Image}} Generation with Representation Regularization},
  author = {Wang, Runqian and He, Kaiming},
  year = 2025,
  eprint = {2506.09027},
  primaryclass = {cs.CV},
  url = {https://arxiv.org/abs/2506.09027},
  archiveprefix = {arXiv}
}

@article{wang2025pixnerd,
  title = {{{PixNerd}}: {{Pixel}} Neural Field Diffusion},
  author = {Wang, Shuai and Gao, Ziteng and Zhu, Chenhui and Huang, Weilin and Wang, Limin},
  year = 2025,
  journal = {arXiv:2507.23268},
  eprint = {2507.23268},
  archiveprefix = {arXiv}
}

@article{wei2024fastclip0,
  title = {{{FastCLIP}}: A Suite of Optimization Techniques to Accelerate {{CLIP}} Training with Limited Resources},
  author = {Wei, Xiyuan and Ye, Fanjiang and Yonay, Ori and Chen, Xingyu and Sun, Baixi and Tao, Dingwen and Yang, Tianbao},
  year = 2024,
  journal = {arXiv preprint arXiv: 2407.01445},
  eprint = {2407.01445},
  archiveprefix = {arXiv}
}

@inproceedings{wu2025representationentanglementgenerationtraining,
  title={Representation Entanglement for Generation: Training Diffusion Transformers Is Much Easier Than You Think},
  author={Ge Wu and Shen Zhang and Ruijing Shi and Shanghua Gao and Zhenyuan Chen and Lei Wang and Zhaowei Chen and Hongcheng Gao and Yao Tang and Jian Yang and Ming-Ming Cheng and Xiang Li},
  booktitle={The Thirty-ninth Annual Conference on Neural Information Processing Systems},
  year={2025},
  url={https://openreview.net/forum?id=koEALFNBj1}
}

@inproceedings{yao2025reconstruction,
  title = {Reconstruction vs. Generation: {{Taming}} Optimization Dilemma in Latent Diffusion Models},
  booktitle = {{{CVPR}}},
  author = {Yao, Jingfeng and Yang, Bin and Wang, Xinggang},
  year = 2025
}

@inproceedings{yuan_provable_2022,
  title = {Provable Stochastic Optimization for Global Contrastive Learning: {{Small}} Batch Does Not Harm Performance},
  booktitle = {Proceedings of the 39th International Conference on Machine Learning},
  author = {Yuan, Zhuoning and Wu, Yuexin and Qiu, Zi-Hao and Du, Xianzhi and Zhang, Lijun and Zhou, Denny and Yang, Tianbao},
  year = 2022,
  series = {Proceedings of Machine Learning Research},
  volume = {162},
  pages = {25760--25782},
  publisher = {PMLR},
  url = {https://proceedings.mlr.press/v162/yuan22b.html}
}

@inproceedings{yuRepresentationAlignmentGeneration2025,
  title = {Representation Alignment for Generation: {{Training}} Diffusion Transformers Is Easier than You Think},
  booktitle = {International Conference on Learning Representations},
  author = {Yu, Sihyun and Kwak, Sangkyung and Jang, Huiwon and Jeong, Jongheon and Huang, Jonathan and Shin, Jinwoo and Xie, Saining},
  year = 2025,
  volume = {2025},
  pages = {87400--87442},
  url = {https://proceedings.iclr.cc/paper_files/paper/2025/file/d9e42b4d7163931f3689d6d6fbaa11d0-Paper-Conference.pdf}
}

@inproceedings{zhang2018unreasonable,
  title = {The Unreasonable Effectiveness of Deep Features as a Perceptual Metric},
  booktitle = {{{CVPR}}},
  author = {Zhang, Richard and Isola, Phillip and Efros, Alexei A and Shechtman, Eli and Wang, Oliver},
  year = 2018
}

@inproceedings{zhang2019root,
  title = {Root Mean Square Layer Normalization},
  booktitle = {{{NeurIPS}}},
  author = {Zhang, Biao and Sennrich, Rico},
  year = 2019
}

@misc{zhang2025repulsoracceleratinggenerativemodeling,
  title = {Repulsor: {{Accelerating}} Generative Modeling with a Contrastive Memory Bank},
  author = {Zhang, Shaofeng and Chen, Xuanqi and Liao, Ning and Zhao, Haoxiang and Wang, Xiaoxing and Tan, Haoru and Wu, Sitong and Jia, Xiaosong and Fan, Qi and Yan, Junchi},
  year = 2025,
  eprint = {2512.08648},
  primaryclass = {cs.CV},
  url = {https://arxiv.org/abs/2512.08648},
  archiveprefix = {arXiv}
}

@misc{zheng2025diffusiontransformersrepresentationautoencoders,
  title = {Diffusion Transformers with Representation Autoencoders},
  author = {Zheng, Boyang and Ma, Nanye and Tong, Shengbang and Xie, Saining},
  year = 2025,
  eprint = {2510.11690},
  primaryclass = {cs.CV},
  url = {https://arxiv.org/abs/2510.11690},
  archiveprefix = {arXiv}
}

@article{zhou2021ibot,
  title = {{iBOT}: {{Image}} {BERT} Pre-Training with Online Tokenizer},
  author = {Zhou, Jinghao and Wei, Chen and Wang, Huiyu and Shen, Wei and Xie, Cihang and Yuille, Alan and Kong, Tao},
  year = 2021,
  journal = {arXiv:2111.07832},
  eprint = {2111.07832},
  archiveprefix = {arXiv}
}

@inproceedings{wangImageUnderstandingTokenizer2024,
  author = {Wang, Luting and Zhao, Yang and Zhang, Zijian and Feng, Jiashi and Liu, Si and Kang, Bingyi},
  booktitle = {Advances in Neural Information Processing Systems},
  doi = {10.52202/079017-0976},
  editor = {A. Globerson and L. Mackey and D. Belgrave and A. Fan and U. Paquet and J. Tomczak and C. Zhang},
  pages = {31015--31035},
  publisher = {Curran Associates, Inc.},
  title = {Image Understanding Makes for A Good Tokenizer for Image Generation},
  url = {https://proceedings.neurips.cc/paper_files/paper/2024/file/37611b0fc2b65cdbb60865af5f6cf453-Paper-Conference.pdf},
  volume = {37},
  year = {2024}
}

@article{qiuRobustLatentMatters2025,
  title   = {Robust Latent Matters: Boosting Image Generation with Sampling Error Synthesis},
  author  = {Kai Qiu and Xiang Li and Jason Kuen and Hao Chen and Xiaohao Xu and Jiuxiang Gu and Yinyi Luo and Bhiksha Raj and Zhe Lin and Marios Savvides},
  year    = {2025},
  journal = {arXiv preprint arXiv: 2503.08354}
}

@inproceedings{hudson2024soda,
  title = {{{SODA}}: {{Bottleneck}} Diffusion Models for Representation Learning},
  booktitle = {{{CVPR}}},
  author={Hudson, Drew A and Zoran, Daniel and Malinowski, Mateusz and Lampinen, Andrew K and Jaegle, Andrew and McClelland, James L and Matthey, Loic and Hill, Felix and Lerchner, Alexander},
  year = 2024,
  eprint = {2311.17901},
  archiveprefix = {arXiv}
}

@inproceedings{zhang2022pdae,
  author = {Zhang, Zijian and Zhao, Zhou and Lin, Zhijie},
  title = {Unsupervised representation learning from pre-trained diffusion probabilistic models},
  year = {2022},
  isbn = {9781713871088},
  publisher = {Curran Associates Inc.},
  address = {Red Hook, NY, USA},
  booktitle = {Proceedings of the 36th International Conference on Neural Information Processing Systems},
  articleno = {1607},
  numpages = {14},
  location = {New Orleans, LA, USA},
  series = {NIPS '22}
}

@inproceedings{wang2023infodiffusion,
  author = {Wang, Yingheng and Schiff, Yair and Gokaslan, Aaron and Pan, Weishen and Wang, Fei and De Sa, Christopher and Kuleshov, Volodymyr},
  title = {{InfoDiffusion}: representation learning using information maximizing diffusion models},
  year = {2023},
  publisher = {JMLR.org},
  booktitle = {Proceedings of the 40th International Conference on Machine Learning},
  articleno = {1511},
  numpages = {19},
  location = {Honolulu, Hawaii, USA},
  series = {ICML'23}
}

@misc{kimiteam2026attentionresiduals,
  title={Attention Residuals},
  author={{Kimi Team} and Guangyu Chen and Yu Zhang and Jianlin Su and Weixin Xu and Siyuan Pan and Yaoyu Wang and Yucheng Wang and Guanduo Chen and Bohong Yin and Yutian Chen and Junjie Yan and Ming Wei and Y. Zhang and Fanqing Meng and Chao Hong and Xiaotong Xie and Shaowei Liu and Enzhe Lu and Yunpeng Tai and Yanru Chen and Xin Men and Haiqing Guo and Y. Charles and Haoyu Lu and Lin Sui and Jinguo Zhu and Zaida Zhou and Weiran He and Weixiao Huang and Xinran Xu and Yuzhi Wang and Guokun Lai and Yulun Du and Yuxin Wu and Zhilin Yang and Xinyu Zhou},
  year={2026},
  eprint={2603.15031},
  archivePrefix={arXiv},
  primaryClass={cs.CL},
  url={https://arxiv.org/abs/2603.15031},
}

@misc{jordan2024muon,
  author       = {Keller Jordan and Yuchen Jin and Vlado Boza and Jiacheng You and
                  Franz Cesista and Laker Newhouse and Jeremy Bernstein},
  title        = {Muon: An optimizer for hidden layers in neural networks},
  year         = {2024},
  url          = {https://kellerjordan.github.io/posts/muon/}
}

@inproceedings{ioffe2015batch,
  title = 	 {Batch Normalization: Accelerating Deep Network Training by Reducing Internal Covariate Shift},
  author = 	 {Ioffe, Sergey and Szegedy, Christian},
  booktitle = 	 {Proceedings of the 32nd International Conference on Machine Learning},
  pages = 	 {448--456},
  year = 	 {2015},
  editor = 	 {Bach, Francis and Blei, David},
  volume = 	 {37},
  series = 	 {Proceedings of Machine Learning Research},
  address = 	 {Lille, France},
  month = 	 {07--09 Jul},
  publisher =    {PMLR},
  url = 	 {https://proceedings.mlr.press/v37/ioffe15.html},
}
\bibliographystyle{iclr2027_conference}

\appendix
\clearpage

\suppressfloats[t]
\section{Extended related work}
\label{app:related}

\providecommand{\cyes}{\textcolor{green!55!black}{\ding{51}}}  \providecommand{\cno}{\textcolor{red!75!black}{\ding{55}}}     \providecommand{\cna}{\textcolor{black!35}{--}}                

\begin{table}[t]
\caption{\textbf{Comparison under the four S requirements}, extending
Table~\ref{tab:deps} to the comparable systems tabulated in this appendix.
\textbf{Self-contained} means no pretrained component, including a latent autoencoder or
tokenizer; \textbf{self-supervised} means no label supervision anywhere in the pipeline,
including through pretrained components; \textbf{single-model} means that the representation
is produced by the generator's own network; and \textbf{single-stage} means that all
components are learned jointly from scratch. Capabilities report a learned and evaluated
\textbf{Repr}esentation and unconditional \textbf{Synth}esis with no source image to encode. All pipelines are as published, except REPA and \dispersive{}, which we combine with the JiT recipe removing the latent model dependency.}
\label{tab:deps-full}
\vspace{2pt}
\centering\footnotesize
\setlength{\tabcolsep}{3.5pt}
\resizebox{\textwidth}{!}{\begin{tabular}{l@{\hskip 10pt}cccc@{\hskip 12pt}cc}
\toprule
 & \multicolumn{4}{c}{\emph{four S requirements}} & \multicolumn{2}{c}{\emph{capabilities}} \\
\cmidrule(r{10pt}){2-5}\cmidrule{6-7}
Method & Self-cont. & Self-sup. & Single-mod. & Single-stage & Repr. & Synth. \\
\midrule
\multicolumn{7}{@{}l}{\emph{Generative recipes}} \\[-2pt]
DiT / SiT \citep{peebles_scalable_2023,ma2024sit} & \cno & \cno & \cna & \cno & \cno & \cyes \\
JiT \citep{li2025jit} & \cyes & \cno & \cna & \cyes & \cno & \cyes \\
JiT $+$ REPA \citep{yuRepresentationAlignmentGeneration2025} & \cno & \cno & \cno & \cno & \cno & \cyes \\
JiT $+$ \dispersive{} \citep{wang2025diffusedisperseimagegeneration} & \cyes & \cno & \cyes & \cyes & \cno & \cyes \\
ReDi \citep{kouzelis2026boostinggenerativeimagemodeling} & \cno & \cno & \cno & \cno & \cno & \cyes \\
\addlinespace[2pt]
\multicolumn{7}{@{}l}{\emph{Semantic conditions}} \\[-2pt]
RCG (ADM) \citep{NEURIPS2024_e304d374} & \cno & \cyes & \cno & \cno & \cno & \cyes \\
DLC \citep{lavoie2025compositionaldiscretelatentcode} & \cno & \cno & \cno & \cno & \cno & \cyes \\
\addlinespace[2pt]
\multicolumn{7}{@{}l}{\emph{Joint representation and generation}} \\[-2pt]
MAGE \citep{52203mage} & \cno & \cno & \cyes & \cno & \cyes & \cyes \\
ADDP \citep{tian2023addp} & \cno & \cno & \cyes & \cno & \cyes & \cyes \\
PDAE \citep{zhang2022pdae} & \cno & \cyes & \cno & \cno & \cyes & \cyes \\
DiffAE \citep{preechakul2022diffusion} & \cyes & \cyes & \cno & \cno & \cyes & \cyes \\
SODA \citep{hudson2024soda} & \cyes & \cyes & \cno & \cyes & \cyes & \cno \\
\midrule
\rowcolor{gray!12}
\methodname{} (ours) & \cyes & \cyes & \cyes & \cyes & \cyes & \cyes \\
\bottomrule
\end{tabular}
}
\end{table}

\noindent{\bf Semantic conditioning for generation.}
Conditional GANs established that class information can substantially improve image synthesis
\citep{mirza2014conditional,brock2019large}. Thus, although DDPM initially demonstrated high-quality
unconditional synthesis \citep{Ho2020}, it was soon followed by Improved DDPM incorporating
class labels into diffusion \citep{Nichol2021}. Classifier guidance further strengthened the effect
of this semantic signal \citep{dhariwal2021diffusion}, while classifier-free guidance subsequently
removed the need for a separate classifier, retaining label conditioning \citep{ho2022classifier}.
Together, these developments made semantics a central lever for generation, but supplied them
directly through annotations.

Alongside this use of labels, efforts to reduce computation brought another source of semantics
into the pipeline. Latent diffusion and flow models generate through pretrained autoencoders
\citep{rombachHighResolutionImageSynthesis2022,dao_flow_2023}, a setting later continued by transformer
generators such as DiT and SiT \citep{peebles_scalable_2023,ma2024sit}.
Although introduced for compression, these autoencoders can inherit semantic information from their
training objectives. VQGAN \citep{Esser2021}, for example, uses a perceptual loss
\citep{zhang2018unreasonable} based on supervised features \citep{simonyan2014very}.
The semantic information carried by these tokenizers can therefore shape the representations learned
on top of them \citep{chenDeconstructingDenoisingDiffusion2024} and also benefit generation beyond
reconstruction fidelity
\citep{wangImageUnderstandingTokenizer2024}.

\noindent{\bf Importing representations for generation.}
Beyond the semantics inherited from autoencoders, generators can also draw on representations
supplied explicitly by pretrained encoders.
Some align autoencoder latents with frozen features
\citep{yao2025reconstruction,qiuRobustLatentMatters2025}, while representation autoencoders generate
directly in a pretrained encoder's latent space
\citep{zheng2025diffusiontransformersrepresentationautoencoders}.
Others bring the same semantic signal into the generator itself through feature alignment, as in
REPA and iREPA \citep{yuRepresentationAlignmentGeneration2025,singh_what_2025}.
Since fixed alignment can conflict with generative detail, subsequent work refines how pretrained
features interact with generative training
\citep{wang_repa_2025,wu2025representationentanglementgenerationtraining,liu2025efficientgenerativemodeltraining}.
ReDi instead models images and pretrained features jointly
\citep{kouzelis2026boostinggenerativeimagemodeling}.

Beyond defining the latent space or guiding internal features, representations can serve directly
as conditions. RCG and DLC take this route, conditioning on pretrained embeddings or discrete codes derived from them,
and training separate models to sample them
\citep{NEURIPS2024_e304d374,lavoie2025compositionaldiscretelatentcode}.
Related retrieval and instance-conditioned methods obtain conditions from real images
\citep{blattmann2022retrieval,bordes2021high,casanova2021instance}, while earlier approaches replace
labels with clusters or proxy labels \citep{luvcic2019high,liu2020diverse,bao2022conditional,Hu_2023_CVPR}.
These approaches show that semantic conditioning need not rely on class labels, but importing
representations introduces dependencies on pretrained components or separate training stages.
Learning these representations within the generator itself would therefore reduce the pipeline's
external dependencies.

\noindent{\bf Self-contained generation.}
The pretrained components used by many of these approaches motivate a complementary direction
that removes such dependencies. Pixel-space diffusion and flow models avoid the upstream autoencoder
\citep{Ho2020,Song2021,lipman_flow_2023}, and recent recipes have made this setting increasingly
competitive at high resolution
\citep{hoogeboom2023simple,hoogeboom2024simpler,chen2025pixelflow,wang2025pixnerd,li2025jit}.
Operating on pixels therefore removes one dependency, although it does not itself replace the
semantic signal supplied by class labels.
Dispersive Loss and Repulsor extend this effort to representation regularization, pushing apart
intermediate features without a pretrained representation encoder
\citep{wang2025diffusedisperseimagegeneration,zhang2025repulsoracceleratinggenerativemodeling}.
In particular, Dispersive Loss finds positive-pair regularization sensitive to noisy-view construction
and therefore retains only contrastive repulsion, arguing that denoising supplies attraction.
The features nevertheless remain auxiliary regularizers, without downstream representation evaluation
or explicit semantic conditioning, and the high-quality ImageNet settings retain class labels.

\noindent{\bf Joint generation and representation learning.}
Self-supervised learning offers a source of such semantics through pretext tasks
\citep{noroozi2016unsupervised,gidaris2018unsupervised}, masked modeling
\citep{bao2021beit,he2022masked,zhou2021ibot}, self-distillation
\citep{grill2020bootstrap,caron2021emerging}, and contrastive objectives
\citep{oord2018representation,chen_simple_2020,he2020momentumcontrastunsupervisedvisual}.
The latter learn invariances by matching augmented views of the same image while separating different
images. Although SimCLR benefits from large batches \citep{chenBigSelfSupervisedModels2020},
SogCLR reduces this requirement through a moving-average approximation to a global contrastive
objective \citep{yuan_provable_2022}. These objectives provide useful representations, but on their
own do not define an image generator. Joint methods therefore seek to combine the two capabilities.
MAGE learns and evaluates both in one
transformer \citep{52203mage}, but operates on pretrained VQGAN tokens \citep{Esser2021}, thus retaining the inherited perceptual supervision discussed above. ADDP and LZN follow a similar joint
learning goal through tokenized or modular latent pipelines
\citep{tian2023addp,lin2025latentzoningnetworkunified}.

A different route makes the learned representation a condition for a diffusion decoder.
DiffAE learns this code through reconstruction, then fits a latent prior
for unconditional sampling \citep{preechakul2022diffusion}. InfoDiffusion likewise conditions diffusion
on learned codes, sampled from a prior or a separately trained prior model
\citep{wang2023infodiffusion}. SODA moves from reconstruction to cross-view generation, using a
source view's code to generate a related view \citep{hudson2024soda}.
All three methods produce their representations with an encoder separate from the diffusion decoder,
and SODA additionally requires a source image for synthesis.
\methodname{} follows the joint learning goal while bringing the
representation into the pixel-space generator itself, where it is learned, evaluated, and used to
condition synthesis. A prior trained in the same run supplies conditions without a source image,
combining the four requirements summarized in Table~\ref{tab:deps-full} and defined in
Section~\ref{sec:intro}.

\FloatBarrier

\section{Algorithms}
\label{app:algorithms}

We detail the training and gradient-balancing procedures of Section~\ref{sec:method}.
Appendix~\ref{app:joint_training} describes the joint training step,
Appendix~\ref{app:gradient_balancing} explains the gradient-norm balancer, and
Appendix~\ref{app:balancer_probe} shows how to implement it without additional backward passes. Algorithms~\ref{alg:train_step} and~\ref{alg:balancer} provide the
pseudocode after these explanations.

\subsection{Joint training step}
\label{app:joint_training}

Algorithm~\ref{alg:train_step} describes a single training step. From every training image it
forms \( k_O=1 \) original and \( k_A=2 \) augmented images, runs one encoder-only pass over all of them at \( t=1 \) under the
null conditioning \( \e_\varnothing \) to obtain the embeddings \( \e^{(k)} \). It then reuses
\( \sg(\e^{(1)}) \) both as the conditioning of the generative path and as the prior's target,
adds the conditioned-branch consistency term on the noised input, and combines the four losses
with the balanced weight \( \lambda \) before one backward pass and one optimizer step.

\subsection{Gradient-norm balancing}
\label{app:gradient_balancing}

Algorithm~\ref{alg:balancer} sets the contrastive weight \( \lambda \). Every \( U \) steps it reads
the gradient norms of \( \L_{\mathrm{CFM}} \) and \( \L_{\mathrm{CL}} \) at the shared subset
\( \varsigma \), the feed-forward submodule of the last encoder block, aggregates them across
GPUs, and updates exponential moving averages whose clamped ratio becomes \( \lambda \) for this step
and the next \( U-1 \) steps.
The pseudocode describes direct per-loss gradient evaluation before the joint backward pass. These evaluations traverse the paths from the losses to \( \varsigma \). Appendix~\ref{app:balancer_probe} describes how to reuse the joint backward pass, thus without the need for additional backward passes and applying the updated weights to the following step.

\subsection{Gradient measurement during the joint backward pass}
\label{app:balancer_probe}

We can recover the per-loss gradient norms from the joint backward pass because the losses
reach different tokens at the last encoder MLP. In the generative pass, patch and in-context
tokens feed the decoder, while the \cls{} token is removed before the decoder and feeds
only CBC. In the clean-view pass, only \cls{} tokens feed the contrastive head, and the
stop-gradient operations block contributions from the other losses. The MLP acts on each
token independently, with no further token mixing before the encoder readouts. Its gradients
can therefore be separated by pass and token type.

For a linear map \( y_t = W x_t + b \) within this MLP, let
\( \delta_t = \partial\L_{\mathrm{total}}/\partial y_t \) be its incoming gradient.
A token group \( S \), spanning all examples on one GPU, contributes
\begin{equation}
    G_{W,S} = \sum_{t\in S}\delta_t x_t^\top,
    \qquad G_{b,S} = \sum_{t\in S}\delta_t.
\end{equation}
The incoming gradient already accounts for downstream operations, including the SwiGLU
activation for the first linear map. Since this activation acts within each token, it
preserves the separation above. Concatenating the contributions from both
linear maps into \( G_S \), gives
\begin{equation}
    \|\nabla_\varsigma\L_{\mathrm{CFM}}\|^2
      = \|G_{\mathrm{gen,non\text{-}CLS}}\|^2,
    \qquad
    \|\nabla_\varsigma\L_{\mathrm{CL}}\|^2
      = \frac{\|G_{\mathrm{repr,CLS}}\|^2}{\lambda_{\mathrm{used}}^2}.
\end{equation}
Here, \( \lambda_{\mathrm{used}}>0 \) is the contrastive weight used in the backward pass,
so dividing by its square recovers the unweighted squared norm. The CBC contribution is
separate and does not enter the ratio. Summing the local squared norms across GPUs and
taking square roots gives the same gradient statistics as Algorithm~\ref{alg:balancer}.

The layer inputs and incoming gradients are already available during backward, so this
calculation adds only local products and norms at one MLP and a reduction of two scalars
across GPUs. It requires no additional backward traversal and leaves the training gradients
unchanged for the given step.
We then use these norms to update the EMA and \( \lambda \), with the new weight taking effect on the following step.

\clearpage
\begin{algorithm}[p]
    \caption{One training step of \methodname{} with gradient-norm balancing.
    }
    \label{alg:train_step}
    \begin{algorithmic}[1]
        \REQUIRE Shared encoder \( f_\theta \), generative decoder \( g_\phi \), contrastive head
        \( h_\psi \), prior \( e_\omega \); original-image and augmented-image families
        \( \augs_{\mathrm{o}}, \augs_{\mathrm{a}} \) with \( k_O \) and \( k_A \) images,
        \( K = k_O + k_A \)
        \REQUIRE Representation null \( \e_\varnothing \), generation null \( \c_\varnothing \)
        with drop probability \( p \); CBC weight \( \eta \); optimizer \( \mathcal{O} \);
        subroutine \textsc{GradBalance} (Algorithm~\ref{alg:balancer}); stop gradient operator \( \sg(\cdot) \)
        \medskip
        \STATE \textbf{Minibatch and views.} Draw \( \B = \{\x_i\}_{i=1}^{B} \sim \D \). For each
        \( i \), draw \( A_1 \sim \augs_{\mathrm{o}} \) and \( A_2,\dots,A_K \sim
        \augs_{\mathrm{a}} \), and set \( \x_i^{(k)} = A_k(\x_i) \).
        \medskip
        \STATE \textbf{Representation pass} (clean, \( t=1 \), null conditioning; encoder only,
        all \( K \) views in one forward).
        \STATE \quad \( \e_i^{(k)} \;=\; \pi_\e\big(f_\theta(\x_i^{(k)} \mid 1,\, \e_\varnothing)\big) \)
        \STATE \quad \( \z_i^{(k)} \;=\; h_\psi(\e_i^{(k)})/\|h_\psi(\e_i^{(k)})\| \)
        \STATE \quad Write \( \e_i = \e_i^{(1)} \) for the original image's embedding.
        \medskip
        \STATE \textbf{Generative pass} (original image only, full network).
        \STATE \quad \( \eps_i \sim \pnoise \), \quad \( \logit(t_i) \sim \N(\mu,\sigma^2) \), \quad
        \( \xt_{i,t} = t_i\,\x_i^{(1)} + (1-t_i)\,\eps_i \)
        \STATE \quad \( \c_i \gets \c_\varnothing \) with probability \( p \), else \( \c_i \gets \sg(\e_i) \)
        \STATE \quad \( \e_{\mathrm{gen}}(\xt_{i,t},\, t_i) = \pi_\e\big(f_\theta(\xt_{i,t} \mid t_i,\, \c_i)\big) \) and
        \( \h_i = \pi_\h\big(f_\theta(\xt_{i,t} \mid t_i,\, \c_i)\big) \), one forward
        \STATE \quad \( F_{\theta,\phi}(\xt_{i,t},\, t_i \mid \c_i) \;=\; g_\phi\big(\h_i \mid t_i,\, \c_i\big) \) \hfill (\( \x \)-prediction)
        \medskip
        \STATE \textbf{Prior pass} (original image only, no image forward).
        \STATE \quad Update the EMA frame \( (\vmu, \vsigma) \) from the original image; set \( \estd_i = (\e_i - \vmu)/\vsigma \).
        \STATE \quad \( \eps'_i \sim \pnoise \), \quad \( \logit(t'_i) \sim \N(\mu',\sigma'^2) \), \quad
        \( \estdt_{i,t'} = t'_i\,\estd_i + (1-t'_i)\,\eps'_i \)
        \STATE \quad Evaluate the prior at \( (\estdt_{i,t'},\, t'_i) \).
        \medskip
        \STATE \textbf{Losses}, as in Equation~\plaineqref{eq:joint_loss}.
        \STATE \quad \( \L_{\mathrm{CL}} \) from \( \{\z_i^{(k)}\} \) with negatives gathered across
        GPUs, the multi-view form of Equation~\plaineqref{eq:cl}
        \STATE \quad \( \L_{\mathrm{CFM}} \) on \( F_{\theta,\phi}(\xt_{i,t},\, t_i \mid \c_i) \) against \( \x_i^{(1)} \), Equation~\plaineqref{eq:fm_loss}
        \STATE \quad \( \L_{\mathrm{prior}} \) on \( e_\omega(\estdt_{i,t'},\, t'_i) \) against \( \sg(\estd_i) \)
        \STATE \quad \( \L_{\mathrm{CBC}} = -\mathrm{sim}\big(\e_{\mathrm{gen}}(\xt_{i,t},\, t_i),\, \sg(\e_i)\big) \)
        \medskip
        \STATE \textbf{Balance, backward and update.}
        \STATE \quad \( \lambda \gets \textsc{GradBalance}(\L_{\mathrm{CFM}},\, \L_{\mathrm{CL}}) \)
        \STATE \quad One backward pass on \( \L_{\mathrm{CFM}} + \lambda\,\L_{\mathrm{CL}} + \L_{\mathrm{prior}} + \eta\,\L_{\mathrm{CBC}} \), then \( \mathcal{O}.\mathrm{step}() \)
    \end{algorithmic}
\end{algorithm}

\begin{algorithm}[p]
    \caption{Gradient-norm balancer of Section~\ref{sec:method:weighting}.
    It keeps
    \( \L_{\mathrm{CFM}} \) and \( \L_{\mathrm{CL}} \) at comparable influence on the shared
    encoder by setting \( \lambda \) to the clamped ratio of their moving-average gradient norms
    at \( \varsigma \), which alleviates the need to tune \( \lambda \) per setting. It is called
    with the current losses and step number and returns the weight to use.
    }
    \label{alg:balancer}
    \begin{algorithmic}[1]
        \REQUIRE Shared parameters subset \( \varsigma \subseteq \theta \), the feed-forward of the
        last encoder block
        \REQUIRE Scale \( \lambda_0 \), clamp \( [\lambda_{\min}, \lambda_{\max}] \),
        update period \( U \), decay range \( [\beta_{\min}, \beta_{\max}] \), warmup \( W \),
        small number \( \xi \) (e.g. \( 10^{-8} \)).
        \STATE \textbf{State:} \( \mathrm{EMA}_{\mathrm{CFM}} \gets 0 \),
        \( \mathrm{EMA}_{\mathrm{CL}} \gets 0 \), \( \lambda \gets \lambda_0 \), update count \( k \gets 0 \).
        \STATE \textbf{Inputs at this call:} the two losses, already computed, and the step index \( s \).
        \IF{\( s \bmod U \neq 0 \)}
            \STATE \textbf{return} \( \lambda \) \hfill (cached; the step does no extra work)
        \ENDIF
        \STATE \( k \gets k+1 \)
        \STATE Backpropagate \( \L_{\mathrm{CFM}} \) through the decoder and \( \L_{\mathrm{CL}} \)
        through the head, each down to the output of \( \varsigma \), retaining the graph for
        the subsequent joint backward pass.
        \STATE Continue each adjoint through \( \varsigma \) alone to obtain
        \( \nabla_\varsigma \L_{\mathrm{CFM}} \) and \( \nabla_\varsigma \L_{\mathrm{CL}} \).
        \STATE Local squared norms \( q_\bullet \gets \|\nabla_\varsigma \L_\bullet\|^2 \) for
        \( \bullet \in \{\mathrm{CFM}, \mathrm{CL}\} \).
        \STATE Reduce across GPUs, \( Q_\bullet \gets \textsc{all\_reduce}(q_\bullet) \), and take
        global norms \( n_\bullet \gets \sqrt{Q_\bullet} \).
        \STATE Decay warmup: \( \beta \gets \beta_{\min} + \min(1,\, s/W)\,(\beta_{\max} - \beta_{\min}) \).
        \STATE \( \mathrm{EMA}_\bullet \gets \beta\,\mathrm{EMA}_\bullet + (1-\beta)\,n_\bullet \).
        \STATE \( \lambda \gets \lambda_0 \cdot \mathrm{clamp}\!\left(
        \dfrac{{\mathrm{EMA}}_{\mathrm{CFM}} + \xi}{{\mathrm{EMA}}_{\mathrm{CL}} + \xi},\;
        \lambda_{\min},\, \lambda_{\max}\right) \)
        \STATE \textbf{return} \( \lambda \)
    \end{algorithmic}
\end{algorithm}

\FloatBarrier
\section{Additional Experimental Results}
\label{app:add_exps}
\subsection{Detailed ablations}
\label{app:detailed_ablations}

This section expands the construction ablation in Section~\ref{sec:exp_ablation}.
Table~\ref{tab:abl_ladder} decomposes our joint-training recipe. It first isolates
adding the contrastive loss with different view splits to the flow matching model.
It then cumulatively adds gradient-norm balancing, representation self-conditioning with the learned prior, and, lastly, the CBC loss.
All results are reported using the JiT-B/16 recipe, trained for $200$ epochs in the unconditional setting.
Table~\ref{tab:cl_comparison} then tracks \methodname{} and capacity-matched contrastive-only
controls over training. All rows train the same $8$-layer encoder from scratch, while the controls
use learning rate, temperature, and weight decay tuned independently for contrastive learning.
Each control is a single $600$-epoch run with a cosine learning-rate schedule, evaluated at its
$200$-, $400$-, and $600$-epoch checkpoints.

\begin{table}[t]
\caption{\textbf{Detailed ablation ladder.} All rows use the JiT-B/16 architecture under the same recipe for
$200$ epochs and unconditional setting. Rungs~0 to~3 isolate the joint-training recipe
before representation self-conditioning and the learned prior. Rung~4 adds both components, and
rung~5 adds the conditioned-branch consistency loss $\Lcbc$ to complete \methodname{}. Rung~0 is
our unconditional flow-matching reproduction. \emph{views} gives the number of original and
augmented inputs, respectively; \emph{bal.} denotes gradient-norm balancing, which rescales the
contrastive weight online.}
\label{tab:abl_ladder}
\vspace{1pt}
\centering\small
\setlength{\tabcolsep}{4.5pt}
\begin{tabular}{ll@{\hskip 8pt}c@{\hskip 11pt}rrrrr@{\hskip 11pt}rr}
\toprule
 & & & \multicolumn{5}{c}{\emph{generation}} & \multicolumn{2}{c}{\emph{representation}} \\
\cmidrule(r{11pt}){4-8}\cmidrule{9-10}
Rung & views & $\lambda_{CL}$ & FID$\downarrow$ & sFID$\downarrow$ & IS$\uparrow$ & Prec.$\uparrow$ & Rec.$\uparrow$ & Acc.$\uparrow$ & transfer-8$\uparrow$ \\
\midrule
\multicolumn{10}{l}{\emph{0. flow matching alone (JiT-B/16 unconditional)}} \\
 & $1+0$ & --- & 61.17 & 9.05 & 17.53 & .384 & .559 & 10.65 & 26.19 \\
\midrule
\multicolumn{10}{l}{\emph{1. add contrastive learning: which views?}} \\
1i & $3+0$ & 0.025 & 59.08 & 9.25 & 17.96 & .385 & .571 & 33.74 & 50.93 \\
1ii & $0+3$ & 0.025 & 70.46 & 16.89 & 15.18 & .237 & \textbf{.625} & 65.30 & 77.36 \\
1iii & $1+2$ & 0.025 & 57.40 & 9.65 & 18.72 & .406 & .599 & 65.43 & 76.69 \\
\midrule
\multicolumn{10}{l}{\emph{2. sweep the contrastive weight}} \\
2a & $1+2$ & 0.01 & 56.00 & 8.83 & 19.11 & .413 & .574 & 62.40 & 74.53 \\
2b & $1+2$ & 0.025 & 57.18 & 8.83 & 19.27 & .411 & .579 & 65.47 & 76.00 \\
2c & $1+2$ & 0.06 & 60.25 & 9.09 & 18.51 & .389 & .578 & 65.37 & 75.07 \\
\midrule
\multicolumn{10}{l}{\emph{3. online gradient balancer instead of manual tuning}} \\
 & $1+2$ & bal. & 54.79 & 8.62 & 19.76 & .425 & .586 & 65.49 & 75.95 \\
\midrule
\multicolumn{10}{l}{\emph{4. add representation self-conditioning and the prior}} \\
 & $1+2$ & bal. & 11.66 & 7.49 & 95.49 & .680 & .520 & 70.97 & 78.86 \\
\midrule
\multicolumn{10}{l}{\emph{5. add CBC (final \methodname{})}} \\
 & $1+2$ & bal. & \textbf{10.10} & \textbf{7.24} & \textbf{101.91} & \textbf{.715} & .497 & \textbf{71.32} & \textbf{79.73} \\
\bottomrule
\end{tabular}
\end{table}

\paragraph{Contrastive learning benefits both tasks before self-conditioning.}
Adding contrastive learning with one original and two augmented inputs substantially improves
representation quality while modestly improving the still weak pre-prior generation baseline.
FID decreases from $61.17$ to $57.40$, while Acc.\ rises from $10.65\%$ to $65.43\%$ and
transfer-8 from $26.19\%$ to $76.69\%$. Using only augmented inputs obtains similar
representation quality but degrades FID to $70.46$. From rung~2 on, the representation is read
from a \cls{} token. With it, gradient-norm balancing reaches $54.79$ FID, better than every fixed
contrastive weight in rung~2 ($56.00$ to $60.25$), without materially changing representation
quality and without manual selection of the contrastive weight.

\paragraph{Self-conditioning is the generative lever.}
Adding representation self-conditioning together with the learned prior converts the improved
representation into substantially better generation, reducing FID to $11.66$ while improving
Acc.\ and transfer-8 to $70.97\%$ and $78.86\%$, respectively. Adding CBC completes
\methodname{} and further improves generation, reducing FID to $10.10$.

\begin{table}[t]
\caption{\textbf{Representation quality throughout training.} We compare \methodname{} with
capacity-matched contrastive-only controls trained from scratch. All rows use the same $8$-layer
encoder; the controls have no generative objective and are independently tuned for contrastive
learning. During inference, the models use approximately the same number of parameters, $57.2$M for
\methodname{} and $57.1$M for the controls.}
\label{tab:cl_comparison}
\centering\small
\setlength{\tabcolsep}{4.5pt}
\begin{tabular}{llrrrrrr}
\toprule
& & \multicolumn{3}{c}{Acc.$\uparrow$} & \multicolumn{3}{c}{transfer-8$\uparrow$} \\
\cmidrule(r){3-5}\cmidrule(l){6-8}
Method & views & $200$ & $400$ & $600$ & $200$ & $400$ & $600$ \\
\midrule
CL only & $2$ augmented & 70.48 & 73.61 & 74.84 & 77.04 & 78.88 & 77.09 \\
CL only & $1$ original $+$ $2$ augmented & \textbf{71.76} & \textbf{74.15} & \textbf{74.98} & 77.68 & 78.66 & 76.61 \\
\rowcolor{gray!12}
\methodname{} (ours) & $1$ original $+$ $2$ augmented & 71.32 & 72.35 & 72.61 & \textbf{79.73} & \textbf{79.89} & \textbf{79.87} \\
\bottomrule
\end{tabular}
\end{table}

\paragraph{Comparison with contrastive-only training.}
The control with one original and two augmented views achieves the highest Acc.\ at
every budget, and the control with two augmented views also exceeds \methodname{} at $400$ and
$600$ epochs, whereas \methodname{} achieves the highest transfer-8 score at every budget. Therefore, across the $200$, $400$, and $600$-epoch budgets,
\methodname{} learns representations broadly comparable to those from dedicated
contrastive training.

\subsection{Controlled comparison at a shorter training budget}
\label{app:controlled_e200}

Table~\ref{tab:controlled_e200} repeats the matched comparison after $200$ epochs. \methodname{}
already leads every single-run baseline on both generation and representation metrics, and exceeds
RCG's $400$-epoch-equivalent checkpoint on every reported metric.
\begin{table}[t]
\centering
\small
\caption{\textbf{Controlled class-unconditional comparison at 200 epochs.} All methods use the
same JiT-B/16 flow-matching recipe on ImageNet ($256\times256$). The RCG row is a
reference: its 200-epoch generator is preceded by a frozen 200-epoch prior stage, for $400$
total epoch-equivalents.}
\label{tab:controlled_e200}
\setlength{\tabcolsep}{4pt}
\begin{tabular}{lrrrr}
\toprule
Method & FID$\downarrow$ & IS$\uparrow$ & Acc.$\uparrow$ & transfer-8$\uparrow$ \\
\midrule
JiT                    & 61.17 & 17.5  & 10.65 & 26.19 \\
JiT $+$ \dispersive{} & 58.68 & 18.0  & 8.88  & 24.06 \\
JiT $+$ REPA          & 53.07 & 23.0  & 63.94 & 73.59 \\
JiT $+$ iREPA         & 48.43 & 24.8  & 65.81 & 76.37 \\
\midrule
RCG (JiT; $400$ total) & 14.27 & 93.4 & 6.54 & 30.28 \\
\midrule
\rowcolor{gray!12}
\methodname{} (ours)  & \textbf{10.10} & \textbf{101.9} & \textbf{71.32} & \textbf{79.73} \\
\bottomrule
\end{tabular}
\end{table}

\subsection{Unconditional generation system-level comparison}
\label{app:unc_sys}

Table~\ref{tab:system_full} presents latent and pixel-space published unconditional ImageNet systems in a common
dependency view. The lowest latent-space FIDs come from multi-stage pipelines built on pretrained autoencoders or tokenizers and, for ReDi and DLC, external DINOv2 features trained beyond ImageNet. Within the
pixel-space block, \methodname{} (L/16) reaches $5.89$ FID without an autoencoder, tokenizer, external model, labels,
or external data, and representation guidance reduces this to $3.47$ using a prior learned in the
same run. Since architectures, protocols, and training budgets differ across systems, the table
should be read as a system-level comparison. Tables~\ref{tab:main_unguided} and~\ref{tab:guided} give the controlled matched-recipe comparisons.
\providecommand{\headspace}{\hspace{.2em}}
\providecommand{\shrink}[1]{{\fontsize{6pt}{7.2pt}\selectfont{#1}}}

\begin{table}[t]
\caption{\textbf{Unconditional ImageNet-1K~\citep{deng2009imagenet} $256\times256$ generation systems and their dependencies.}
FID and IS as reported by the corresponding work. \emph{Lat.} names the pretrained
VQGAN tokenizer~\citep{Esser2021} or latent autoencoder, SD-VAE\textsuperscript{a} or an LDM autoencoder (LDM-AE)~\citep{rombachHighResolutionImageSynthesis2022}, and \emph{Ext.} any other pretrained
model. \emph{Lab.} marks label supervision anywhere in the
pipeline, including the supervised perceptual loss~\citep{zhang2018unreasonable} inherited by
VQGAN and SD-VAE. \emph{External data} marks a dependency trained on data beyond ImageNet-1K.
In particular, SD-VAE uses OpenImages~\citep{Krasin2017} and, per its model card, is fine-tuned on
LAION-Aesthetics~\citep{schuhmann2022laion5b} and an unreleased LAION-Humans subset, DINOv2 uses LVD-142M~\citep{oquab2023dinov2},
and CLIP uses WIT-400M~\citep{radfordLearningTransferableVisual2021}. \emph{Multi-stage} marks more than one separately trained
component. \emph{Params} count all inference-time models, including a latent autoencoder or tokenizer
where applicable. $\dagger$ retrieves real images at sampling. $^*$ denotes DiT values reported
by ReDi and $^\ddagger$ values reported by RCG. \shrink{(CFG)} and \shrink{(RG)} denote
classifier-free and representation guidance, respectively. Training budgets are not matched. \methodname{} (B/16) is reported at $800$ epochs and
\methodname{} (L/16) at $400$ epochs.}
\label{tab:system_full}
\vspace{1pt}
\centering\scriptsize
\setlength{\tabcolsep}{3.25pt}
\resizebox{\textwidth}{!}{\begin{tabular}{l@{\hskip 6pt}llccccrD{.}{.}{2.2}D{.}{.}{3.1}}
\toprule
& \multicolumn{5}{c}{\emph{dependencies}} & \multicolumn{3}{c}{\emph{generation}} \\
\cmidrule(r{5pt}){2-6}\cmidrule{7-9}
Method & Lat. & Ext. & Lab. & External data & Multi-stage & Params & \multicolumn{1}{c}{FID$\downarrow$} & \multicolumn{1}{c}{IS$\uparrow$} \\
\midrule
\multicolumn{9}{l}{\emph{latent-space}} \\
\headspace DiT-B/2$^*$ \citep{peebles_scalable_2023} & SD-VAE & --- & \yes & \yes & \yes & 179M & 69.3\phantom{0} & \multicolumn{1}{c}{---} \\
\headspace DiT-XL/2$^*$ \citep{peebles_scalable_2023} & SD-VAE & --- & \yes & \yes & \yes & 724M & 44.6\phantom{0} & \multicolumn{1}{c}{---} \\
\headspace ReDi (DiT-B/2) \citep{kouzelis2026boostinggenerativeimagemodeling} & SD-VAE & DINOv2 & \yes & \yes & \yes & 179M & 51.7\phantom{0} & \multicolumn{1}{c}{---} \\
\headspace ReDi (DiT-B/2) \shrink{(RG)} \citep{kouzelis2026boostinggenerativeimagemodeling} & SD-VAE & DINOv2 & \yes & \yes & \yes & 179M & 47.3\phantom{0} & \multicolumn{1}{c}{---} \\
\headspace ReDi (DiT-XL/2) \citep{kouzelis2026boostinggenerativeimagemodeling} & SD-VAE & DINOv2 & \yes & \yes & \yes & 724M & 25.1\phantom{0} & \multicolumn{1}{c}{---} \\
\headspace ReDi (DiT-XL/2) \shrink{(RG)} \citep{kouzelis2026boostinggenerativeimagemodeling} & SD-VAE & DINOv2 & \yes & \yes & \yes & 724M & 22.6\phantom{0} & \multicolumn{1}{c}{---} \\
\headspace MaskGIT$^\ddagger$ \citep{Chang2022} & VQGAN & --- & \yes & \no & \yes & 227M & 20.72 & 42.1 \\
\headspace ADDP-B \citep{tian2023addp} & VQGAN & --- & \yes & \no & \yes & 176M & 8.9\phantom{0} & 95.3 \\
\headspace ADDP-L \citep{tian2023addp} & VQGAN & --- & \yes & \no & \yes & 439M & 7.6\phantom{0} & 105.1 \\
\headspace MAGE-B \citep{52203mage} & VQGAN & --- & \yes & \no & \yes & 176M & 8.67 & 94.8 \\
\headspace MAGE-L \citep{52203mage} & VQGAN & --- & \yes & \no & \yes & 439M & 7.04 & 123.5 \\
\headspace RDM-IN$^\dagger$ \shrink{(CFG)} \citep{blattmann2022retrieval} & LDM-AE & CLIP & \yes & \yes & \yes & 400M & 5.91 & 158.8 \\
\headspace RCG (MAGE-B) \citep{NEURIPS2024_e304d374} & VQGAN & MoCo v3 & \yes & \no & \yes & 239M & 3.98 & 177.8 \\
\headspace RCG (MAGE-L) \citep{NEURIPS2024_e304d374} & VQGAN & MoCo v3 & \yes & \no & \yes & 502M & 3.44 & 186.9 \\
\headspace RCG (MAGE-B) \shrink{(CFG)} \citep{NEURIPS2024_e304d374} & VQGAN & MoCo v3 & \yes & \no & \yes & 239M & 3.19 & 212.6 \\
\headspace RCG (MAGE-L) \shrink{(CFG)} \citep{NEURIPS2024_e304d374} & VQGAN & MoCo v3 & \yes & \no & \yes & 502M & 2.15 & 253.4 \\
\headspace DLC512 \citep{lavoie2025compositionaldiscretelatentcode} & SD-VAE & DINOv2 & \yes & \yes & \yes & 1198M & 2.00 & 260.7 \\
\headspace DLC512 \shrink{(CFG)} \citep{lavoie2025compositionaldiscretelatentcode} & SD-VAE & DINOv2 & \yes & \yes & \yes & 1198M & 1.59 & 255.4 \\
\midrule
\multicolumn{9}{l}{\emph{pixel-space}} \\
\headspace BigGAN$^\ddagger$ \citep{donahue2019large} & --- & --- & \no & \no & \no & $\sim$70M & 38.61 & 24.7 \\
\headspace ADM \citep{dhariwal2021diffusion} & --- & --- & \no & \no & \no & 554M & 26.21 & 39.7 \\
\headspace RCDM$^\dagger$ \citep{bordes2021high} & --- & DINO & \no & \no & \yes & --- & 19.0\phantom{0} & 51.9 \\
\headspace IC-GAN$^\dagger$ \citep{casanova2021instance} & --- & SwAV & \no & \no & \yes & $\sim$75M & 15.6\phantom{0} & 59.0 \\
\headspace RCG (ADM) \citep{NEURIPS2024_e304d374} & --- & MoCo v3 & \no & \no & \yes & $\sim$617M & 6.24 & 136.9 \\
\midrule
\rowcolor{gray!12}
\headspace \methodname{} (B/16) & --- & --- & \no & \no & \no & 231M & 7.67 & 127.6 \\
\rowcolor{gray!12}
\headspace \methodname{} (B/16) \shrink{(RG)} & --- & --- & \no & \no & \no & 231M & 3.51 & 178.7 \\
\rowcolor{gray!12}
\headspace \methodname{} (L/16) & --- & --- & \no & \no & \no & 559M & 5.89 & 151.6 \\
\rowcolor{gray!12}
\headspace \methodname{} (L/16) \shrink{(RG)} & --- & --- & \no & \no & \no & 559M & 3.47 & 201.7 \\
\bottomrule
\end{tabular}}
\par\vspace{1pt}\footnotesize\textsuperscript{a}\;SD-VAE denotes the commonly used public
\texttt{sd-vae-ft-ema} checkpoint.
\end{table}

\subsection{Analysis of the representation prior}
\label{app:prior}

At inference, the generator is conditioned on a representation sampled from the prior $\priornet$, so
sample quality depends on both the generator and the prior. To separate the two, we replace the prior
sample with the representation $\e(\x)$ of a real training image, which we call \emph{oracle}
conditioning. It requires real images at generation time, so it serves as a diagnostic rather than as a
way to sample.

\paragraph{Protocol.}
All results use \methodname{} (B/16) at $400$ epochs, without guidance. We draw a conditioning set of
$50$ images per class, $50$K in total, from the ImageNet training set with a fixed seed, and generate
one image from the representation of each. The prior arm consists of $50$K images generated by the same
model from prior samples, namely the samples behind the $8.92$ FID in Table~\ref{tab:main_unguided}. We
compute precision and recall (Appendix~\ref{app:evaluation}) of both arms against two references: the
conditioning set itself, and a disjoint set drawn in the same way from the remaining training images.
For each reference we use a random $10$K subset for precision and recall computation, matching what ADM does.

\paragraph{Real representations raise recall, prior samples raise precision.}
Table~\ref{tab:semiparam} reports the results. The prior arm serves as a control: its samples were
generated from neither reference, and it scores nearly the same against both, with precision $.718$
and $.722$ and recall $.524$ and $.510$. The two references therefore behave alike for samples not
generated from either, so a gap between them in the oracle arm reflects that its samples were generated
from the conditioning images.

The oracle arm shows such a gap. Its recall is $.828$ against the conditioning set and $.664$ against
the disjoint set, so the generated images cover the specific images they were conditioned on better
than other images from the same distribution. The representation thus carries image-specific content
into the generated image.

Against the disjoint set, where the oracle arm has no such advantage, it still reaches higher recall
than the prior arm, $.664$ against $.510$, at lower precision, $.626$ against $.722$. The two arms share
the generator and differ only in the conditioning vectors, so this difference is due to the prior:
conditioned on real representations, the samples cover more of the reference images, while conditioned
on prior samples, a larger fraction of them falls inside the reference manifold.

\begin{table}[t]
\caption{\textbf{Oracle conditioning against the learned prior.} \methodname{} (B/16) at $400$ epochs,
unguided, $50$K samples per arm. The \emph{oracle} arm ($\dagger$, semi-parametric) replaces the prior
sample with the representation of a real training image, one sample per representation. Each arm is
scored against $10$K images from the conditioning set and from a disjoint set of training images.
With the same generator, real representations give higher recall and prior samples higher precision.}
\label{tab:semiparam}
\vspace{1pt}
\centering\small
\setlength{\tabcolsep}{5pt}
\begin{tabular}{ll@{\hskip 13pt}rr}
\toprule
conditioning & reference & Prec.$\uparrow$ & Rec.$\uparrow$ \\
\midrule
\multirow{2}{*}{prior $\priornet$}
     & conditioning set & \textbf{.718} & .524 \\
     & disjoint set     & \textbf{.722} & .510 \\
\midrule
\multirow{2}{*}{oracle repr.$^\dagger$}
     & conditioning set & .645 & \textbf{.828} \\
     & disjoint set     & .626 & \textbf{.664} \\
\bottomrule
\end{tabular}
\end{table}

\FloatBarrier
\section{Flow Matching equivalence under a learned conditioning}
\label{app:grad_equiv}
This appendix explains why the shared embedding is passed through stop-gradient before conditioning the generative and prior branches.

Our representation-conditioned generator uses the Conditional Flow Matching (CFM) objective~\citep{lipman_flow_2023} for training.
Since the conditioning embedding is itself produced by the shared encoder, we need to ensure that optimizing this conditional objective is still equivalent to optimizing the corresponding marginal Flow Matching objective.
We will present the argument for the generation path, but the same reasoning applies to the representation prior.

At a given training step, we consider a sampled pair \( (\x,\c) \), where \( \x \) is a clean data sample and \( \c=\c_\theta(\x) \) is the embedding produced by the shared encoder for that same sample.
Following the standard FM setup, let
\begin{equation}
    \x \sim p_{\mathrm{data}},
    \qquad
    \eps \sim p_{\mathrm{noise}},
    \qquad
    \c_\theta(\x) = \pi_\e\circ f_\theta(\x\mid 1,\e_{\varnothing}),
    \label{eq:conditioning_embedding_gradient_equiv}
\end{equation}
and define the bridge state and conditional velocity by
\begin{equation}
    \x_t = t\x+(1-t)\eps,
    \qquad
    \v_t = \frac{d\x_t}{dt}=\x-\eps,
    \label{eq:image_velocity_target_gradient_equiv}
\end{equation}
where \( \x \) is the clean data endpoint, \( \eps \) is the noise endpoint, and \( \x_t \) is the intermediate point on the bridge between them.
This is the conditional path written \( \xt_t \) in Section~\ref{sec:prelim:fm}; we drop the tilde in this appendix.
For a generic intermediate location \( \u \), the marginal conditional velocity is
\begin{equation}
    \v_t^\star(\u\mid\c)
    =
    \E_{\x,\eps}\!\left[\v_t\mid \x_t=\u,\c,t\right].
    \label{eq:marginal_velocity_gradient_equiv}
\end{equation}
This is the average bridge velocity among all sampled pairs \((\x,\eps)\) whose interpolant passes through the same intermediate point \( \u \) under the same condition \( \c \) at time \( t \).

The network predicts the clean image, \( F_\theta(\x_t, t \mid \c) = \genb(\x_t, t \mid \c) \), where we write \( \theta \) for all generator parameters, including the decoder's \( \phi \), and the velocity model is the velocity it implies,
\begin{equation}
    \v_{\theta}(\x_t, t \mid \c) = \frac{F_\theta(\x_t, t \mid \c) - \x_t}{1-t}.
    \label{eq:velocity_model_gradient_equiv}
\end{equation}

The CFM and FM objectives for the generation path are
\begin{align}
    \L_{\mathrm{CFM}}(\theta;\c_\theta)
    &= \E_{\x,\eps,t}\left[ \|\v_{\theta}(\x_t,t\mid\c)-\v_t\|^2 \right],
    \label{eq:image_cfm_gradient_equiv} \\
    \L_{\mathrm{FM}}(\theta;\c_\theta)
    &= \E_{\x,\eps,t}\left[ \|\v_{\theta}(\x_t,t\mid\c)-\v_t^\star(\x_t\mid\c)\|^2 \right].
    \label{eq:image_fm_gradient_equiv}
\end{align}

The following result follows the same decomposition as Theorem~2 of~\citet{lipman_flow_2023}, but accounts for the fact that our conditioning variable is learned.
\begin{cor}[FM and CFM with learned conditioning]
\label{cor:learned_conditioning_fm_cfm_relation}
Let
\begin{equation*}
    m_t=\v_t^\star(\x_t\mid\c)
    =
    \E_{\x,\eps}[\v_t\mid \x_t,\c,t].
\end{equation*}
If the conditioning embedding \(\c_\theta\) is differentiated, then
\begin{equation}
    \L_{\mathrm{CFM}}(\theta;\c_\theta)
    =
    \L_{\mathrm{FM}}(\theta;\c_\theta)
    +
    R(\theta),
    \qquad
    R(\theta)
    =
    \E_{\x,\eps,t}\!\left[\|\v_t-m_t\|^2\right].
    \label{eq:learned_conditioning_cfm_fm_relation}
\end{equation}
Consequently,
\begin{equation}
    \nabla_\theta\L_{\mathrm{CFM}}(\theta;\c_\theta)
    =
    \nabla_\theta\L_{\mathrm{FM}}(\theta;\c_\theta)
    +
    \nabla_\theta R(\theta).
    \label{eq:learned_conditioning_cfm_fm_gradient_relation}
\end{equation}
\end{cor}
Thus, for non-stopped \(\c_\theta\), CFM and FM are not gradient-equivalent in general.
If \(\bar{\c}_\theta=\sg(\c_\theta)\), then \(R\) is fixed in the FM/CFM computation graph, so the extra term vanishes and the standard FM/CFM gradient equivalence is recovered.
In our architecture, this is precisely why we apply stop-gradient to the shared embedding before using it to condition the generative and prior branches.

It is worth naming what the extra term rewards, and what that would do to the representation.
By the tower property \( R(\theta)=\E\big[\mathrm{Var}(\v_t\mid\x_t,t,\c)\big] \), where
\( \mathrm{Var} \) is the total conditional variance. From \( \x_t=t\x+(1-t)\eps \) we have \( \x-\x_t=(1-t)(\x-\eps) \), so the velocity
target is the affine function \( \v_t=(\x-\x_t)/(1-t) \) of the clean sample, whose multiplier and
offset are both fixed once \( (\x_t,t) \) are given. Its variance therefore scales by
\( (1-t)^{-2} \), and
\begin{equation}
    R(\theta)
    =
    \E_{\x_t,t,\c}\!\left[\,\frac{\mathrm{Var}\!\left(\x \mid \x_t,t,\c_\theta(\x)\right)}{(1-t)^2}\,\right],
    \label{eq:residual_variance_is_about_x}
\end{equation}
is the residual uncertainty about the \emph{clean image} once the noised view and the conditioning are
known. Here \( \theta \) enters only through \( \c_\theta(\x) \), since \( \v_t \), \( \x_t \) and
\( t \) do not involve it. The two extremes bound what that term can do: an uninformative \( \c \)
leaves the full posterior uncertainty \( \mathrm{Var}(\x\mid\x_t,t) \), while a \( \c \) that
determines \( \x \) drives \( R \) to zero. Lowering \( R \) therefore means making
\( \c_\theta(\x) \) more informative about \( \x \) itself.

What matters for the representation is which information that recruits. The cheapest way
to pin down \( \x \) from \( \x_t \) is to carry the detail that noising destroys, and in the limit
\( R \) is minimized by a \( \c \) that is a lossless code of \( \x \), at which point the embedding
is the image again and carries no abstraction. That is exactly the information \( \L_{\mathrm{CL}} \) is
built to discard, since augmentation invariance means being blind to crops, color, grayscale, and other augmentations.
Without the stop-gradient, the generative objective would therefore not merely break an identity: it
would affect the contents of the shared embedding, and the solution it prefers is a
compression code rather than a representation. The stop-gradient does not prevent generation from
benefiting from whatever \( \e \) carries, it prevents the generative loss from choosing what that
is, leaving that to \( \L_{\mathrm{CL}} \).

\begin{proof}
    Following Theorem~2 of~\citet{lipman_flow_2023}, write
    \(V_\theta=\v_{\theta}(\x_t,t\mid\c)\) and
    \(m_t=\v_t^\star(\x_t\mid\c)\).
    Expanding the two squared losses gives
    \begin{align}
        \|V_\theta-\v_t\|^2
        &=
        \|V_\theta\|^2
        -2\left\langle V_\theta, \v_t \right\rangle
        +\|\v_t\|^2,
        \label{eq:cfm_integrand_expansion} \\
        \|V_\theta-m_t\|^2
        &=
        \|V_\theta\|^2
        -2\left\langle V_\theta, m_t \right\rangle
        +\|m_t\|^2.
        \label{eq:fm_integrand_expansion}
    \end{align}
    The first terms are identical.
    For the cross term, by definition
    \begin{equation*}
        m_t=\E_{\x,\eps}[\v_t\mid \x_t,t,\c],
    \end{equation*}
    and \(V_\theta=\v_{\theta}(\x_t,t\mid\c)\) is a deterministic function of the same conditioning
    variables. The first equality below is then the tower property, using that \( (\x_t,t,\c) \) is a
    function of \( (\x,\eps,t) \); the second takes \( V_\theta \) inside the inner conditional
    expectation, which is legitimate because it is measurable with respect to exactly the variables
    being conditioned on; the third is the tower property again. This is the conditioning step of the
    Flow Matching proof with the conditioning set enlarged to include \( \c \),
    \begin{align}
        \E_{\x,\eps,t}\!\left[\left\langle V_\theta, m_t \right\rangle\right]
        &=
        \E_{\x_t,t,\c}\!\left[
        \left\langle
        V_\theta,
        \E_{\x,\eps}[\v_t\mid \x_t,t,\c]
        \right\rangle
        \right] \nonumber \\
        &=
        \E_{\x_t,t,\c}\!\left[
        \E_{\x,\eps}\!\left[
        \left\langle V_\theta, \v_t \right\rangle
        \mid \x_t,t,\c
        \right]
        \right] \nonumber \\
        &=
        \E_{\x,\eps,t}\!\left[\left\langle V_\theta, \v_t \right\rangle\right].
        \label{eq:conditional_cross_term_identity}
    \end{align}
    None of those three steps requires \( \c \) to be independent of \( \theta \), so a learned
    conditioning leaves the cross term intact and its entire effect is confined to the last terms.
    Those are not equal in general: the CFM objective contains
    \(\E\|\v_t\|^2\), whereas the FM objective contains
    \(\E\|m_t\|^2\).
    Instead, the difference is the expected residual variance of the target velocity after conditioning on \( (\x_t,t,\c) \):
    \begin{align}
        \L_{\mathrm{CFM}}(\theta;\c_\theta)
        -
        \L_{\mathrm{FM}}(\theta;\c_\theta)
        &=
        \E_{\x,\eps,t}\|\v_t\|^2
        -
        \E_{\x,\eps,t}\|m_t\|^2 \nonumber \\
        &=
        \E_{\x,\eps,t}\!\left[
        \|\v_t-m_t\|^2
        \right].
                \label{eq:cfm_fm_target_variance_difference}
    \end{align}
    The last equality is the conditioning step once more, with \( m_t \) now playing the role
    \( V_\theta \) played above. Since \( m_t \) is itself a function of \( (\x_t,t,\c) \),
    \begin{equation*}
        \E\langle \v_t, m_t \rangle
        = \E\big[\langle \E[\v_t\mid\x_t,t,\c],\, m_t \rangle\big]
        = \E\|m_t\|^2 ,
    \end{equation*}
    so expanding \( \E\|\v_t-m_t\|^2 = \E\|\v_t\|^2 - 2\E\langle \v_t,m_t\rangle + \E\|m_t\|^2 \)
    collapses the cross term and leaves \( \E\|\v_t\|^2-\E\|m_t\|^2 \).

    Therefore, when the conditioning value is treated as fixed, this difference is independent of the velocity model parameters, and the CFM and FM gradients with respect to those parameters are equal.
    However, when \(\c_\theta\) is not fixed, \(m_t\) and the conditional variance term \(R(\theta)\) both depend on \(\theta\) through the conditioning path.
    Differentiating \Eqref{eq:learned_conditioning_cfm_fm_relation} gives \Eqref{eq:learned_conditioning_cfm_fm_gradient_relation}.
    Since \(\nabla_\theta R(\theta)\) need not be zero, the gradients are not equivalent in general.
    With stop-gradient, \(\partial\bar{\c}_\theta/\partial\theta=0\), so the derivative of \(R\) through the conditioning path is zero and the standard fixed-conditioning FM identity applies.
\end{proof}

\paragraph{Short counterexample without stop-gradient.}
The failure can already occur for a deterministic linear encoder. Take
\( \x,\eps\sim\N(0,\operatorname{diag}(2,1)) \) independently in \( \R^2 \) and fix
\( t=\tfrac12 \), so that \( \x_t=\tfrac12(\x+\eps) \) and \( \v_t=\x-\eps \), and let the
bottleneck encoder be \( \c_\theta(\x)=x_1+\theta x_2 \).

Since \( \operatorname{Cov}(\v_t,\x_t)=\tfrac12\operatorname{diag}(2,1)-\tfrac12\operatorname{diag}(2,1)=0 \)
and the pair is jointly Gaussian, \( \v_t \) is independent of \( \x_t \), so whatever
\( (\x_t,\c_\theta) \) reveals about \( \v_t \) arrives through \( \c_\theta \). Substituting the
definitions gives \( \c_\theta-x_{t,1}-\theta x_{t,2}=\tfrac12(v_{t,1}+\theta v_{t,2}) \), so
conditioning on \( (\x_t,\c_\theta) \) is conditioning on the single statistic
\( U=v_{t,1}+\theta v_{t,2} \), and \( m_t=\E[\v_t\mid U] \).

Now \( \v_t\sim\N(0,\Sigma_v) \) with \( \Sigma_v=\operatorname{diag}(4,2) \), so
\( \operatorname{Cov}(\v_t,U)=\Sigma_v(1,\theta)^{\!\top}=(4,2\theta)^{\!\top} \) and
\( \operatorname{Var}(U)=4+2\theta^2 \). For a jointly Gaussian pair the conditional mean is the
linear projection \( m_t=\operatorname{Cov}(\v_t,U)\,U/\operatorname{Var}(U) \), whose residual has
covariance \( \Sigma_v-\operatorname{Cov}(\v_t,U)\operatorname{Cov}(\v_t,U)^{\!\top}/\operatorname{Var}(U) \).
Taking the trace of that residual covariance,
\begin{equation*}
    R(\theta)
    =
    \operatorname{tr}\Sigma_v-\frac{\|\operatorname{Cov}(\v_t,U)\|^2}{\operatorname{Var}(U)}
    =
    6-\frac{16+4\theta^2}{4+2\theta^2}
    =
    \frac{4+4\theta^2}{2+\theta^2},
    \qquad
    \nabla_\theta R(\theta)=\frac{8\theta}{(2+\theta^2)^2},
\end{equation*}
which is nonzero for \( \theta\neq0 \). The residual variance therefore does depend on the learned
conditioning, and without stop-gradient the FM/CFM gradient equivalence fails.

The example also shows the \emph{direction} of this dependence. \(R\) is minimized at \(\theta=0\), where \(c_\theta=x_1\) and \(U=v_{t,1}\), the highest-variance coordinate of \(v_t\). Thus, in this toy problem, the additional FM term favors conditioning directions that explain more raw target variance. For images, this motivates the concern that, unless the contrastive objective explicitly constrains the representation away from such directions, the FM term can favor low-level variance that the representation objective is intended to discard.

\FloatBarrier

\section{Experiment details}
\label{app:training}

This appendix gives the training and evaluation setup behind Section~\ref{sec:experiments}.
Table~\ref{tab:app_current_run} summarizes the main hyperparameters at both scales. Unless an
experiment states otherwise, the settings below are the defaults for every ImageNet-1K
\citep{deng2009imagenet,imagenet15russakovsky} run.

All experiments are run on nodes of 8 NVIDIA H200 GPUs. To speed up experiments, we train on up to 3 nodes for B/16 and 6 nodes for L/16.

\subsection{Training setup}

We follow JiT \citep{li2025jit} for the generative hyperparameters and for
\( \L_{\mathrm{CFM}} \), including the optimizer, the augmentation policy of the original image and the flow
matching loss. We depart from it in two places, both chosen once and applied to every image generator in the
comparison, baselines included: the base learning rate, which is tuned to \( 1\cdot10^{-4} \); and the attention residuals of Appendix~\ref{app:arch}.
Since every arm trains at a global batch of \( 1024 \), the learning rate is
\( 4\cdot10^{-4} \).
For the contrastive hyperparameters and the augmentation policy of the augmented images we follow SogCLR \citep{yuan_provable_2022}.

The RCG baseline is our own re-implementation on this backbone and conditions on DINOv2 \citep{oquab2023dinov2} embeddings. For its prior, we use the architecture of their official implementation\footnote{\url{https://github.com/LTH14/rcg}}, a residual MLP with \( 12 \) blocks of width \( 1536 \), trained for \( 200 \) epochs with our flow-matching objective and prior timestep schedule. For a fair comparison, we also ran RCG with our prior, which performed worse.

\paragraph{Augmentations.} We use \( k_O = 1 \) original image and \( k_A = 2 \) augmented images. The
original image follows JiT's preprocessing, a center crop to the training resolution and a random horizontal
flip with probability \( 0.5 \). It is the only input the generative path sees and it is also one
contrastive anchor. The two augmented images are used by the representation path alone and follow the
SimCLR policy \citep{chen_simple_2020,grill2020bootstrap}: random resized crop with scale in
\( [0.08, 1.0] \), color jitter of strength \( (0.4, 0.4, 0.2, 0.1) \) applied with probability
\( 0.8 \), random grayscale with probability \( 0.2 \), Gaussian blur with \( \sigma \) uniform in
\( [0.1, 2.0] \), and a random horizontal flip. This is what exposes the shared encoder to both
regimes every step, noised original images for generation
and clean augmented ones for the invariances.

\paragraph{Timestep sampling.} For both flow matching losses we sample \( t \) from the
logit-normal schedule \citep{esser2024scaling}: draw \( s \sim \N(\mu, \sigma^2) \) and set
\( t = \mathrm{sigmoid}(s) \), which we write \( \logit(t) \sim \N(\mu, \sigma^2) \). The image path uses JiT's default \( (\mu, \sigma) = (-0.8, 0.8) \), and the prior
path \( (\mu', \sigma') = (-1.2, 0.6) \), the latter selected on a frozen-prior grid and applied to the RCG baseline
as well.

\paragraph{Losses.} The contrastive objective is the multi-view form of the Global Contrastive
Loss \citep{yuan_provable_2022}, with fixed temperature \( \tau = 0.15 \), moving-average parameter
\( \gamma = 0.9 \), and negatives gathered across GPUs. The generative objective is the conditional
flow matching loss \citep{lipman_flow_2023} with \( \x \)-prediction and the velocity
parameterization of JiT \citep{li2025jit}. We clamp the \( 1-t \) denominator at \( \max(1-t, 0.05) \) for numerical stability near \( t=1 \). The weight \( \lambda \) on
\( \L_{\mathrm{CL}} \) comes from the balancer of Section~\ref{sec:method:weighting} and
Algorithm~\ref{alg:balancer}: gradient norms are read on the feed-forward submodule of the last
encoder block, the estimate is refreshed every \( 10 \) steps, and the two moving averages use a
decay ramped linearly from \( 0.9 \) to \( 0.99 \) over the first \( 5000 \) steps. Their ratio is
clamped to \( [10^{-3}, 1] \) and multiplied by \( \lambda_0 = 1 \). The CBC weight is fixed at
\( \eta = 0.01 \) and is not balanced. \emph{None of these constants are tuned}.

\paragraph{Optimization.} At B/16 we train with AdamW~\citep{Loshchilov2019}, \( (\beta_1, \beta_2) = (0.9, 0.95) \), in
bfloat16 mixed precision, under a constant learning rate with a linear warmup over \( 5 \) epochs
and weight decay \( 0 \). At L/16 the parameters split across two optimizers: Muon~\citep{jordan2024muon} takes the
two-dimensional parameters inside the transformer blocks, at learning rate \( 2\cdot10^{-2} \) and
weight decay \( 5\cdot10^{-4} \), and everything else stays on AdamW at the same learning rate and
weight decay as at B/16. Moreover, for L/16, we use a \textsc{BatchNorm}~\citep{ioffe2015batch} after the \cls{} token at the encoder-decoder split. We keep two exponential moving averages of the weights, with
decays \( 0.9999 \) and \( 0.9996 \).

\begin{table}[t]
\caption{\textbf{Main hyperparameters} for the reported \methodname{} (B/16) and \methodname{} (L/16)
ImageNet-1K \( 256\times256 \) runs. Rows spanning both columns are identical at the two scales.
The baselines share this recipe apart from the objective and, for REPA and iREPA, the
\( 4/8 \) encoder--decoder split, following their published split; see
Appendix~\ref{app:training}. Guidance follows \citet{ho2022classifier} on the interval of \citet{kynkaanniemi2024applying}.}
\label{tab:app_current_run}
\vspace{1pt}
\centering
\tablestyle{4pt}{1.02}
\begin{tabular}{l | x{88} x{142}}
 & \textbf{\methodname{} (B/16)} & \textbf{\methodname{} (L/16)} \\
\shline
\rowcolor[gray]{0.9}\multicolumn{3}{l}{\textbf{Architecture}} \\
Depth: total / enc. / dec. & 12, 8, 4 & 24, 16, 8 \\
Hidden dim & 768 & 1024 \\
Heads & 12 & 16 \\
In-context start block & 4 & 8 \\
\cls{} normalization & identity & BatchNorm \\
Image size & \multicolumn{2}{c}{256} \\
Patch size & \multicolumn{2}{c}{image size / 16} \\
Patch bottleneck & \multicolumn{2}{c}{128} \\
In-context tokens & \multicolumn{2}{c}{32} \\
\shline
\rowcolor[gray]{0.9}\multicolumn{3}{l}{\textbf{Training}} \\
Epochs & 400 (800 system-level) & 400 \\
Optimizer & AdamW & Muon (block matrices) $+$ AdamW \\
Learning rate & \( 4\cdot10^{-4} \) & \( 2\cdot10^{-2} \) Muon, \( 4\cdot10^{-4} \) AdamW \\
Weight decay & 0.0 & \( 5\cdot10^{-4} \) Muon, 0.0 AdamW \\
LR schedule & \multicolumn{2}{c}{constant} \\
Warmup epochs & \multicolumn{2}{c}{5} \\
AdamW \( (\beta_1,\beta_2) \) & \multicolumn{2}{c}{(0.9, 0.95)} \\
Global batch & \multicolumn{2}{c}{1024} \\
Conditioning drop prob. & \multicolumn{2}{c}{10\%} \\
Number augmented views & \multicolumn{2}{c}{2} \\
EMA decays & \multicolumn{2}{c}{0.9999 (EMA1), 0.9996 (EMA2)} \\
\( \logit(t) \) schedule & \multicolumn{2}{c}{\( \N(-0.8, 0.8^2) \)} \\
Prior \( \logit(t) \) schedule & \multicolumn{2}{c}{\( \N(-1.2, 0.6^2) \)} \\
CL dimension & \multicolumn{2}{c}{256} \\
CL temperature / \( \gamma \) & \multicolumn{2}{c}{0.15, 0.9} \\
CBC weight \( \eta \) & \multicolumn{2}{c}{0.01} \\
Gradient-ratio \( \lambda_0 \) & \multicolumn{2}{c}{1} \\
Gradient-ratio update & \multicolumn{2}{c}{every 10 steps} \\
Gradient-ratio EMA warmup & \multicolumn{2}{c}{0.9 to 0.99 over 5000 steps} \\
Gradient-ratio clamp & \multicolumn{2}{c}{\( [10^{-3}, 1] \)} \\
\( t \)-division clamp & \multicolumn{2}{c}{0.05} \\
\midline
\rowcolor[gray]{0.9}\multicolumn{3}{l}{\textbf{Sampling}} \\
ODE solver / steps & \multicolumn{2}{c}{Heun, 50} \\
Sampling \( t \) schedule & \multicolumn{2}{c}{linear in \( [0.0, 1.0] \)} \\
CFG interval & \multicolumn{2}{c}{\( [0.1, 1.0] \)} \\
CFG sweep start / step & \multicolumn{2}{c}{from 1.0, step 0.1, extended while FID improves} \\
CFG sweep images & \multicolumn{2}{c}{8000} \\
\end{tabular}
\end{table}

\subsection{Evaluation}
\label{app:evaluation}

\paragraph{Checkpoint selection.} All metrics are computed on EMA weights.
Section~\ref{sec:exp_main} reports the \( 400 \)-epoch checkpoint, the budget at which every
baseline stops, Section~\ref{sec:exp_ablation} reports \( 200 \) epochs, and
Section~\ref{sec:exp_system} reports \methodname{} (B/16) at \( 800 \) epochs and \methodname{} (L/16) at
\( 400 \). We evaluate the EMA with decay \( 0.9999 \), except for the system-level comparison, where we evaluate both exponential moving
averages and select the \( 0.9996 \) one only for \methodname{} (L/16).

\paragraph{Generation.} All generation numbers use \( 50 \)K samples drawn with a Heun ODE solver
\citep{heun1900neue} with \( 50 \) steps (49 Heun steps and a final Euler step). For the class-conditional baselines the
batch is \( 50 \) samples for each of the \( 1000 \) ImageNet classes. \methodname{} is
class-unconditional and does not use labels for generation either. Instead, we unconditionally sample from the learned prior \( e_\omega \) the conditioning vectors for the generator. Thus, its batch carries no class structure by construction. Where guidance is reported,
it is applied on the interval \( t \in [0.1, 1.0] \) \citep{kynkaanniemi2024applying} at a scale
swept from \( 1.0 \) in steps of \( 0.1 \) on \( 8000 \) samples, continuing while
FID keeps improving.

\paragraph{Generation metrics.} We evaluate generation with the ADM suite
\citep{dhariwal2021diffusion}, using their reference batches and official
implementation.\footnote{\url{https://github.com/openai/guided-diffusion/tree/main/evaluations}}
All metrics are computed from Inception-V3 outputs on the \( 50 \)K generated images; all but the
Inception Score also read the suite's ImageNet \( 256\times256 \) reference batch.
\begin{itemize}
    \item \emph{FID} \citep{heusel2017gans} is the Fr\'echet distance between Gaussians fitted to
    the pooled Inception features of the generated and reference sets.

    \item \emph{sFID} \citep{nash2021generating} computes the same distance on flattened
    intermediate \emph{spatial} features instead of pooled ones, making it more sensitive to errors in spatial structure.

    \item \emph{IS} \citep{salimans2016improved} is the exponential of the average KL divergence
    between the class distribution the same Inception-V3 predicts for each generated image and the
    marginal over all of them. It rises when samples are individually classifiable and collectively
    spread across classes, and it is the one metric here computed from the generated set alone.

    \item \emph{Precision} and \emph{Recall} \citep{Kynkaanniemi2019} estimate the support of each
    distribution with a \( k \)-nearest-neighbor manifold (\( k=3 \)). Precision is the fraction
    of \emph{generated} images inside the reference manifold and measures fidelity. Recall is the
    fraction of \emph{reference} images inside the generated manifold and measures coverage. Since
    guidance typically buys fidelity at the cost of coverage, the two should be read together.
\end{itemize}

Not all metrics see the same reference set. FID and sFID are computed against first-
and second-moment statistics precomputed over the full reference dataset, while precision and
recall are estimated against the \( 10K \) reference images distributed with the batch, not against the full dataset.

\paragraph{Representation.} Both probes freeze the network and read one vector per image. For
\methodname{} that vector is the \cls{} token at the encoder/decoder split, taken before the
contrastive projection head. The baselines carry no \cls{} token, so for them it is the mean of the
patch tokens, excluding the in-context conditioning tokens, at their own split. Vanilla
JiT and the RCG generator have no split and are read after their last block. The contrastive-only
controls carry a \cls{} token and are read like \methodname{}.

The ImageNet probe (\emph{Acc.}) fits an \( \ell_2 \)-regularized multinomial logistic regression
on the raw features, with no augmentation: images are resized to \( 292 \)
with bicubic interpolation and center-cropped to \( 256 \). The regularization strength is chosen
over \( 45 \) logarithmically spaced values between \( 10^{-6} \) and \( 10^{5} \) on a held-out
\( 10\% \) of the training split, the classifier is then refit on the full training split at that
value, and we report top-1 accuracy on the validation split.

The transfer probe (\emph{transfer-8}) evaluates out-of-domain transfer on Food-101
\citep{bossard14}, CIFAR-10 and CIFAR-100 \citep{krizhevskyLearningMultipleLayers2009}, Stanford
Cars \citep{Krause2013CollectingAL}, Describable Textures \citep{cimpoi14describing}, Oxford-IIIT
Pets \citep{parkhi12a}, Caltech-101 \citep{li_andreeto_ranzato_perona_2022} and Oxford 102 Flowers
\citep{Nilsback08}, and reports their unweighted mean. We follow the evaluation protocol of each
dataset's own paper, so mean per-class accuracy for Caltech-101, Oxford-IIIT Pets and Oxford 102
Flowers and top-1 accuracy for the rest. Features are \( \ell_2 \)-normalized and
standardized. We then fit an \( \ell_2 \)-regularized multinomial logistic regression with the
L-BFGS solver of scikit-learn \citep{scikit-learn}. Training images are center-cropped to the
training resolution of \( 256 \) with a random horizontal flip, and the flip is removed at test
time. We report the best test accuracy over
\( 10 \) logarithmically spaced regularization values between \( 10^{-3} \) and \( 1 \).

\FloatBarrier

\section{Architecture details}
\label{app:arch}

\methodname{} is a JiT-style pixel-space transformer \citep{li2025jit} with three changes. First,
we split the transformer depth into a shared encoder \( f_\theta \) and a generative decoder
\( g_\phi \), two contiguous portions of one ViT-style backbone \citep{dosovitskiy2021an} with
total depth \( L = L_{\mathrm{enc}} + L_{\mathrm{dec}} \). Second, we prepend a learnable \cls{}
token and read it out at the encoder/decoder boundary as the global embedding \( \e \). Third, we
attach a contrastive projection head \( h_\psi \) to that embedding while the patch tokens continue
through the decoder for flow matching. Additionally, we condition the generative path on \( \sg(\e) \)
and learn a prior \( e_\omega \) \citep{NEURIPS2024_e304d374} over these embeddings for
class-unconditional sampling. Table~\ref{tab:components} lists each component with its input and
output spaces.

All remaining architectural choices follow JiT, except for the attention residuals below, which every model we train uses. For completeness, we go into each component in the following paragraphs.

\paragraph{Transformer block.} Each block uses pre-\textsc{RMSNorm} \citep{zhang2019root},
multi-head self-attention with QK-norm \citep{henry2020query} and two-dimensional rotary positional
embeddings \citep{su2024roformer}, then a second pre-\textsc{RMSNorm} and a \textsc{SwiGLU}
feed-forward network \citep{shazeer2020glu} whose inner width is \( \tfrac{8}{3}d \), an MLP ratio
of \( 4 \) under \textsc{SwiGLU}'s two-thirds rescaling. As in DiT and JiT, both
residual branches are modulated by adaLN-Zero \citep{peebles_scalable_2023}: a linear projection of
the conditioning vector produces shift, scale and gate parameters for the two sublayers, and the
modulation layers are zero-initialized so each branch begins near the identity.

\paragraph{Attention residuals.} Applied to every model we train in Section~\ref{sec:experiments}, ours and all baselines, attention residuals \citep{kimiteam2026attentionresiduals} replace the plain residual add of each sublayer. The sublayers stay pre-norm, since what changes is the residual stream and not where the normalization sits. Instead of reading the running residual sum, a sublayer
reads a learned softmax-weighted mixture over its depth history and appends its own output to that history: the keys are the normalized history entries, and the query is a learned vector per token role (patch, \cls{} or in-context tokens). Writing \( \vs_1, \dots, \vs_n \in \R^{d} \) for the entries at one token position of role \( \rho \), that mixture is
\begin{equation}\label{eq:attnres}
    \hat{\vs} \;=\; \sum_{i=1}^{n} \alpha_i\, \vs_i,
    \qquad
    \alpha_i \;=\; \frac{\exp \big\langle \vq_\rho,\, \mathrm{RMSNorm}(\vs_i) \big\rangle}
                        {\sum_{l=1}^{n} \exp \big\langle \vq_\rho,\, \mathrm{RMSNorm}(\vs_l) \big\rangle},
\end{equation}
a softmax over depth with no \( 1/\sqrt{d} \) scaling. A history entry is the sum of four
consecutive sublayer outputs, which bounds the number of entries the mixture runs over.
The queries are zero-initialized, so the mixture is uniform at initialization.

\paragraph{Patch and positional embeddings.} Images are tokenized by a low-rank linear patch
embedder \citep{li2025jit}: a \( p \times p \) strided convolution maps non-overlapping RGB patches
into a bottleneck of width \( 128 \), followed by a \( 1\times1 \) convolution to the transformer
width \( d \). Two positional signals are used and they act at different places. Fixed
two-dimensional sine-cosine embeddings \citep{NIPS2017_3f5ee243} are added to the patch tokens
only, before the \cls{} and in-context tokens are concatenated. The rotary embeddings act inside
attention, on the queries and keys after QK-norm, with an identity rotation at the non-patch
positions so that the \cls{} and in-context tokens carry no spatial position.
The final layer applies \textsc{RMSNorm}, adaLN modulation and a zero-initialized linear projection to
\( p^2 c_{\mathrm{out}} \) values per patch, which are unpatchified into an image.

\paragraph{Conditioning.} The timestep is embedded with sinusoidal features followed by a two-layer
\textsc{SiLU} \citep{elfwing2018sigmoid} MLP. The conditioning vector \( \c \) passes through a
linear projection and then enters the network twice. It is added to the timestep embedding that
drives the adaLN-Zero modulation of every block, and the same projected vector is also injected as
\( 32 \) in-context tokens at a fixed block, following the design of JiT and
MAR \citep{li2025jit,li2024autoregressive}. These tokens are added on the fourth block for B/16 and the eighth for L/16. The in-context
tokens are removed before the final image projection. On the representation pass the conditioning
input is the learned null \( \e_\varnothing \), and classifier-free guidance uses a second learned
null \( \c_\varnothing \).

\paragraph{Encoder/decoder split.} At block \( L_{\mathrm{enc}} \) the \cls{} token is read out as
the global embedding. At B/16 the readout is the raw token; at
L/16 it passes through a \textsc{BatchNorm}~\citep{ioffe2015batch}.
The \cls{} token is dropped at the split
and never reaches the decoder, so the decoder is a function of the patch tokens \( \h \) alone.  Our
\methodname{} (B/16) splits \( (L_{\mathrm{enc}}, L_{\mathrm{dec}}) = (8,4) \) of \( 12 \) blocks at
width \( 768 \) with \( 12 \) heads, and \methodname{} (L/16) splits \( (16,8) \) of \( 24 \) at width
\( 1024 \) with \( 16 \) heads.

\paragraph{Contrastive head.} The projection head \( h_\psi \) is a two-hidden-layer MLP,
\( d \to 2048 \to 2048 \to 256 \), followed by \( \ell_2 \)-normalization. It is used by \( \L_{\mathrm{CL}} \) and by nothing else, so it is discarded after training.

\paragraph{Prior.} It is a residual MLP over a single vector with additive
timestep conditioning.
That is, a linear layer maps the noised embedding \( \estdt_{t'} \in \R^{d} \) to width \( C \). The
timestep is encoded with \( 256 \) sinusoidal features and a two-layer \textsc{SiLU} MLP to width
\( C \). The hidden vector then passes through \( N \) residual blocks, each applying
\textsc{RMSNorm}, a \textsc{SwiGLU} projection to the block width, an additive linear projection of
the timestep embedding, and a linear projection back to \( C \). A final \textsc{RMSNorm},
\textsc{SiLU} and linear layer return to \( \R^{d} \). Every block's output projection and the
final projection are zero-initialized, so each block is the identity and the prior predicts zero at
initialization. The prior uses additive residual connections and it is
trained with the same \( \x \)-prediction flow matching parameterization as the image path, in
representation space. Table~\ref{tab:app_prior_arch} gives the configuration.

\begin{table}[h]
\caption{\textbf{Components of the shared model.} \( T \) is the number of patch tokens, \( d \)
the hidden width, \( d_z \) the contrastive projection dimension, and \( \c \) the conditioning
vector of Section~\ref{sec:method:arch}.}
\label{tab:components}
\vspace{1pt}
\centering\small
\setlength{\tabcolsep}{5pt}
\begin{tabular}{llll}
\toprule
Component & Symbol & Input & Output \\
\midrule
Shared encoder & \( f_\theta \) & \( (\x_t \mid t, \c) \in \R^{3\times H\times W}\times[0,1]\times\R^{d} \) & \( (\e,\h) \in \R^{d}\times\R^{T\times d} \) \\
\multicolumn{4}{l}{\small \emph{Noisy image, timestep and conditioning \(\mapsto\) global embedding \( \e \) and patch tokens \( \h \).}} \\
\midrule
Generative decoder & \( g_\phi \) & \( (\h \mid t, \c) \in \R^{T\times d}\times[0,1]\times\R^{d} \) & \( \x \in \R^{3\times H\times W} \) \\
\multicolumn{4}{l}{\small \emph{Patch tokens, timestep and conditioning \(\mapsto\) predicted clean image.}} \\
\midrule
Contrastive projection & \( h_\psi \) & \( \e \in \R^{d} \) & \( \z \in \R^{d_z} \) \\
\multicolumn{4}{l}{\small \emph{Global embedding \(\mapsto\) the space where \( \L_{\mathrm{CL}} \) is optimized.}} \\
\midrule
Embedding prior & \( e_\omega \) & \( (\estdt_{t'}, t') \in \R^{d} \times [0,1] \) & \( \estd \in \R^{d} \) \\
\multicolumn{4}{l}{\small \emph{Noised standardized embedding \(\mapsto\) its clean standardized counterpart.}} \\
\bottomrule
\end{tabular}
\end{table}

\begin{table}[ht]
\caption{\textbf{Prior architecture.} The residual-MLP backbone of \( e_\omega \), shared by both
scales; its input and output width follows the backbone, so the parameter count is the B/16 one.}
\label{tab:app_prior_arch}
\vspace{1pt}
\centering\small
\setlength{\tabcolsep}{6pt}
\begin{tabular}{lc}
\toprule
Config & Value \\
\midrule
Input / output width & representation dimension \( d \) \\
Hidden width \( C \) & 808 \\
Residual blocks \( N \) & 14 \\
Block \textsc{SwiGLU} width & 2155 \\
Timestep features & sinusoidal, width 256 \\
Timestep embedding & 808, added inside each block \\
Activation & \textsc{SiLU} \\
Normalization & \textsc{RMSNorm} \\
Output head & \textsc{RMSNorm}, \textsc{SiLU}, linear \\
Zero-initialized & block and final output projections \\
Dropout & 0.0 \\
Parameters (B/16) & 99.7M \\
\bottomrule
\end{tabular}
\end{table}

\FloatBarrier
\clearpage
\section{Qualitative results using learned representations}
\label{app:qualitative}

\methodname{} extracts its own semantic representations and uses them to condition generation.
We explore this connection through qualitative experiments using representations
produced by the model itself. First, we generate from real-image representations to examine
whether the outputs preserve the references' semantic content. We then interpolate between two
image representations to explore how the generated content changes along the resulting path.

All examples use \methodname{} (L/16) at $400$ epochs. Generated images use representation guidance with a scaling factor $s=2.0$ (see Equation~\ref{eq:cfg}).
Reference images come from the ImageNet training set, except flower buds in Figure~\ref{fig:showcase_flower_buds} which are external photographs.
Real references show the same $256\times256$ crops supplied to the encoder.
This section is meant to illustrate qualitative behavior rather than
quantify reconstruction accuracy or diversity.

\subsection{Generation from real-image representations}
\label{app:qualitative_real}

We encode each reference image and directly supply its representation to the generator,
bypassing the prior. The pairs in Figure~\ref{fig:showcase_real} illustrate preservation
of semantic content without exact reproduction of image details. Generated images retain
recognizable subjects and scene types, while pose, color, object markings, backgrounds, and
architectural details can change. For example, a cup retains its identity as a cup while its
decoration changes, and a building remains recognizable as the same type of structure despite
changes in its facade. Thus, the model's own representations support semantically related
generation without requiring the output to reconstruct the particular reference image.

\subsection{Interpolation between image representations}
\label{app:qualitative_interp}

We next interpolate between the representations of two real images and generate from the
intermediate vectors (Figures~\ref{fig:showcase_interp} and~\ref{fig:showcase_flower_buds}). We use spherical interpolation of
their directions and linear interpolation of their magnitudes in the generator's conditioning
space, bypassing the prior. Image-sampling noise is held fixed within each row, so changes
along a path reflect changes in the conditioning representation. In Figure~\ref{fig:showcase_interp}, the paths first
explore appearance and scene composition within a class, then connect different classes,
from two dog breeds to more distant concepts such as a bottle and a dog.
These cross-class paths reveal how structures from the endpoints can blend, although their
intermediate images need not depict plausible natural objects.

\label{app:qualitative_flower_buds}
Figure~\ref{fig:showcase_flower_buds} applies the same interpolation mechanism to five
selected pairs of bud and open-flower representations. Bud photographs were obtained
externally and were not part of the training set. Open-flower photographs come from the
ImageNet training set and were manually matched to the buds. Thus, the paired photographs may
depict different individual flowers. These are representation interpolations, not observed
growth sequences or predictions of a particular bud's future appearance. The selected paths
show opening-like changes in flower structure, while endpoint color and details can drift.
The six generated images include endpoint reconstructions at $\alpha=0$ and $1$;
these are distinct from the real photographs.

\clearpage

\begin{figure}[p]
\centering
\begingroup
\setlength{\tabcolsep}{1pt}

\begin{tabular}{@{}cc@{\hspace{7pt}}cc@{\hspace{7pt}}cc@{}}
\footnotesize Reference & \footnotesize Generated &
\footnotesize Reference & \footnotesize Generated &
\footnotesize Reference & \footnotesize Generated \\[3pt]
\includegraphics[width=.158\linewidth]{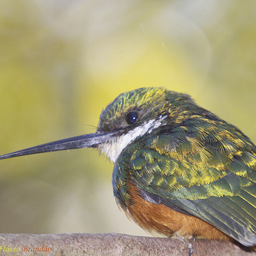} & \includegraphics[width=.158\linewidth]{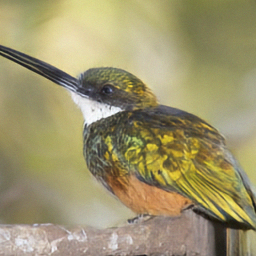} & \includegraphics[width=.158\linewidth]{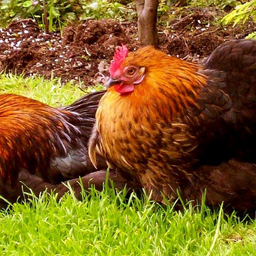} & \includegraphics[width=.158\linewidth]{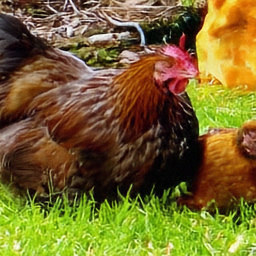} & \includegraphics[width=.158\linewidth]{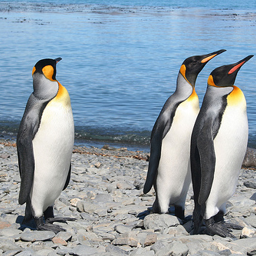} & \includegraphics[width=.158\linewidth]{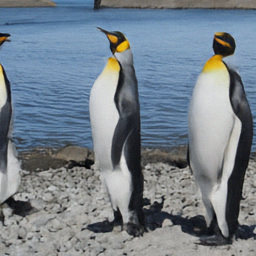} \\[3pt]
\includegraphics[width=.158\linewidth]{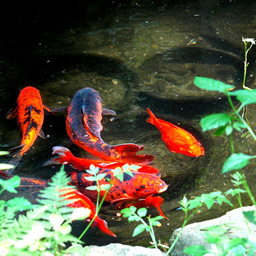} & \includegraphics[width=.158\linewidth]{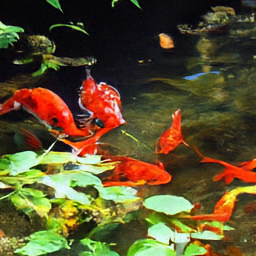} & \includegraphics[width=.158\linewidth]{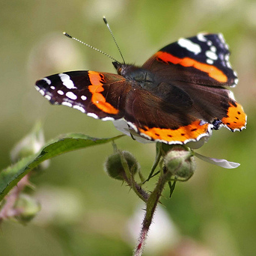} & \includegraphics[width=.158\linewidth]{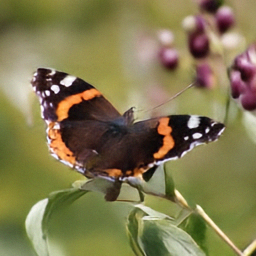} & \includegraphics[width=.158\linewidth]{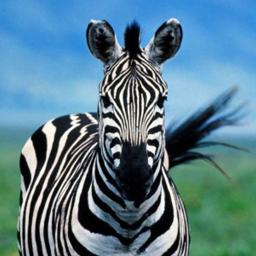} & \includegraphics[width=.158\linewidth]{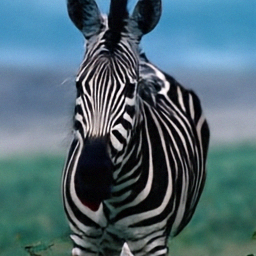} \\[3pt]
\includegraphics[width=.158\linewidth]{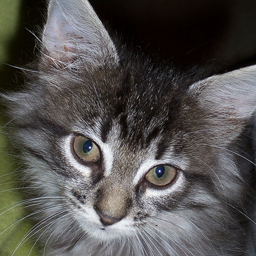} & \includegraphics[width=.158\linewidth]{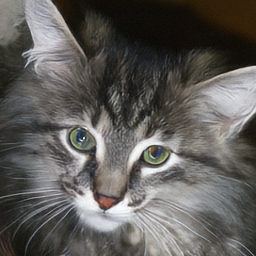} & \includegraphics[width=.158\linewidth]{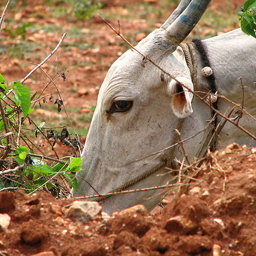} & \includegraphics[width=.158\linewidth]{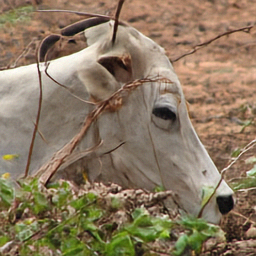} & \includegraphics[width=.158\linewidth]{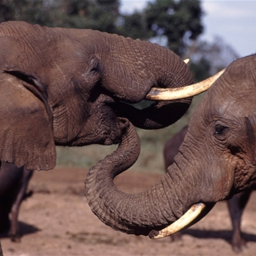} & \includegraphics[width=.158\linewidth]{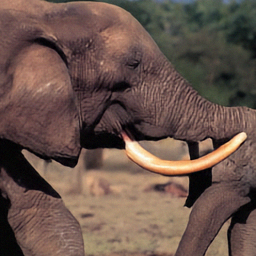} \\[3pt]
\includegraphics[width=.158\linewidth]{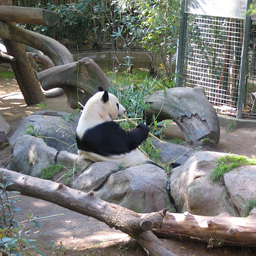} & \includegraphics[width=.158\linewidth]{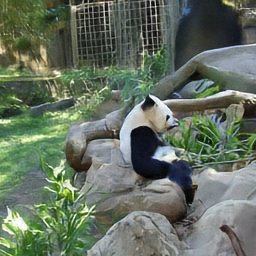} & \includegraphics[width=.158\linewidth]{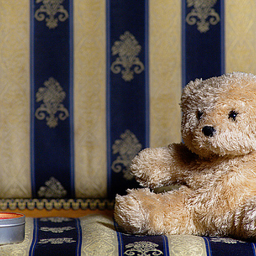} & \includegraphics[width=.158\linewidth]{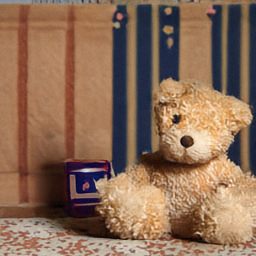} & \includegraphics[width=.158\linewidth]{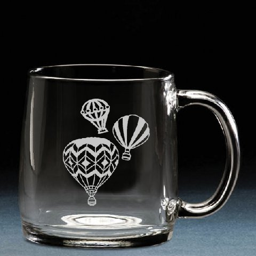} & \includegraphics[width=.158\linewidth]{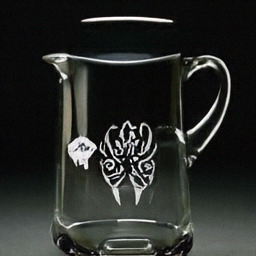} \\[3pt]
\includegraphics[width=.158\linewidth]{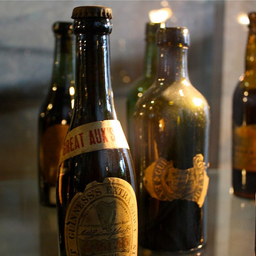} & \includegraphics[width=.158\linewidth]{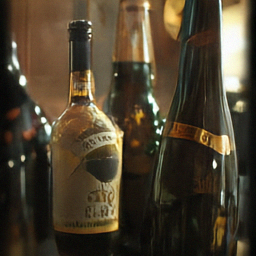} & \includegraphics[width=.158\linewidth]{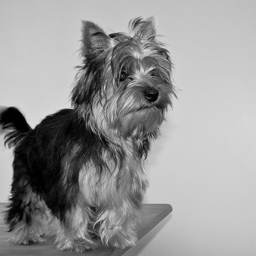} & \includegraphics[width=.158\linewidth]{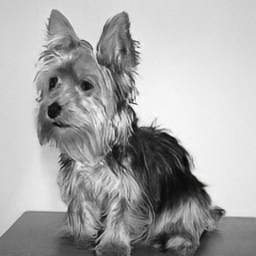} & \includegraphics[width=.158\linewidth]{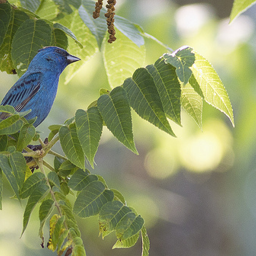} & \includegraphics[width=.158\linewidth]{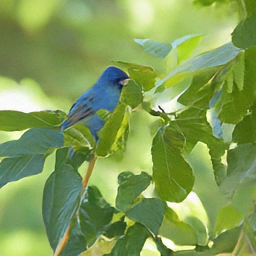} \\[3pt]
\includegraphics[width=.158\linewidth]{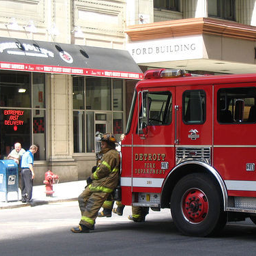} & \includegraphics[width=.158\linewidth]{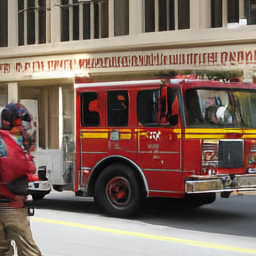} & \includegraphics[width=.158\linewidth]{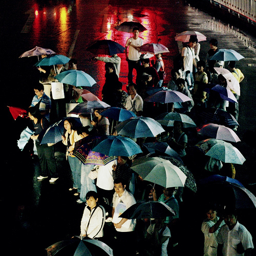} & \includegraphics[width=.158\linewidth]{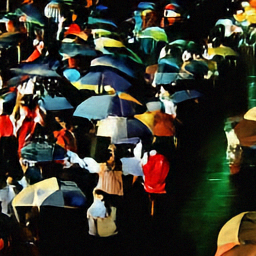} & \includegraphics[width=.158\linewidth]{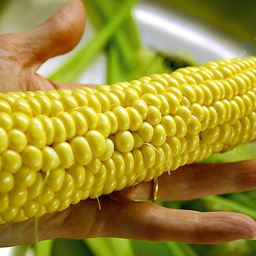} & \includegraphics[width=.158\linewidth]{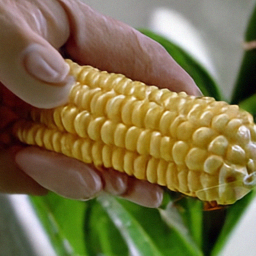} \\[3pt]
\includegraphics[width=.158\linewidth]{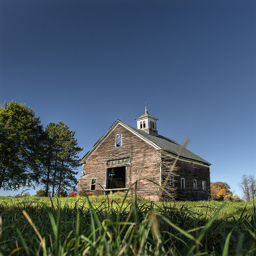} & \includegraphics[width=.158\linewidth]{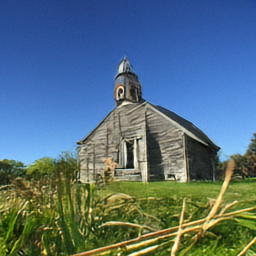} & \includegraphics[width=.158\linewidth]{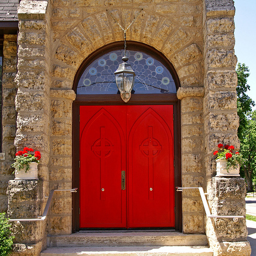} & \includegraphics[width=.158\linewidth]{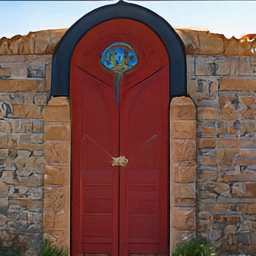} & \includegraphics[width=.158\linewidth]{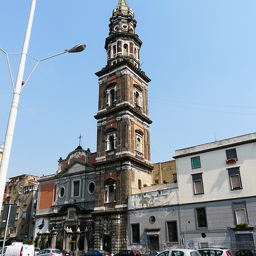} & \includegraphics[width=.158\linewidth]{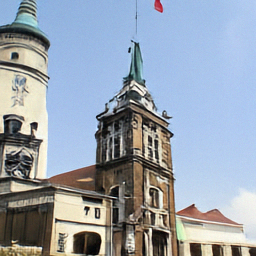} \\[3pt]
\includegraphics[width=.158\linewidth]{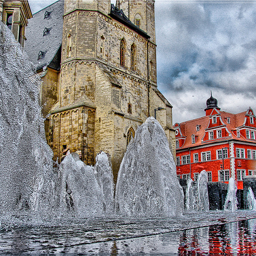} & \includegraphics[width=.158\linewidth]{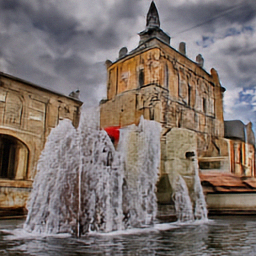} & \includegraphics[width=.158\linewidth]{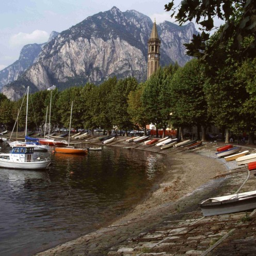} & \includegraphics[width=.158\linewidth]{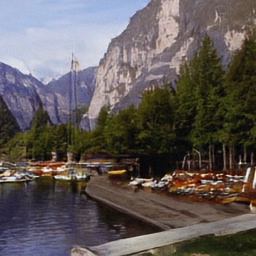} & \includegraphics[width=.158\linewidth]{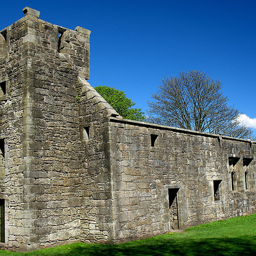} & \includegraphics[width=.158\linewidth]{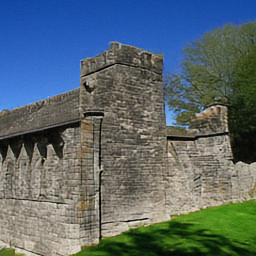} \\
\end{tabular}
\endgroup
\caption{\textbf{Preserving semantics while varying image details.} Each pair shows a real
reference and a generation conditioned on the representation extracted by
\methodname{} itself. Recognizable subjects and scene types persist across changes in pose,
color, markings, background, and structure. The generator starts from image noise and
bypasses the learned prior. All examples use \methodname{} (JiT-L/16) at $400$ epochs with
representation guidance.}
\label{fig:showcase_real}
\end{figure}

\clearpage

\begin{figure}[p]
\centering
\begingroup
\setlength{\tabcolsep}{1pt}

\begin{tabular}{@{}c@{\hspace{5pt}}ccccc@{\hspace{5pt}}c@{}}
\footnotesize Reference A & \multicolumn{5}{c}{\footnotesize Generated along the representation path} &
\footnotesize Reference B \\[2pt]
& \footnotesize $\alpha=0$ & \footnotesize $0.25$ & \footnotesize $0.50$ &
\footnotesize $0.75$ & \footnotesize $1$ & \\[3pt]
\includegraphics[width=.136\linewidth]{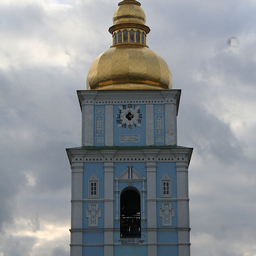} & \includegraphics[width=.136\linewidth]{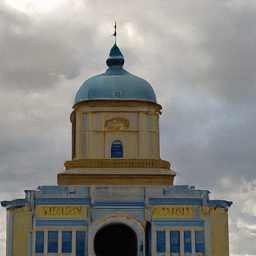} & \includegraphics[width=.136\linewidth]{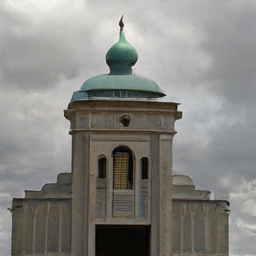} &
\includegraphics[width=.136\linewidth]{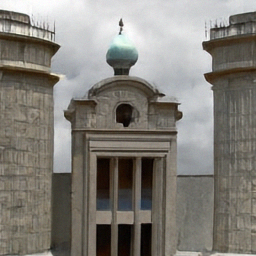} & \includegraphics[width=.136\linewidth]{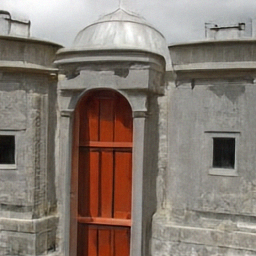} & \includegraphics[width=.136\linewidth]{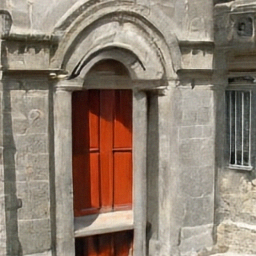} &
\includegraphics[width=.136\linewidth]{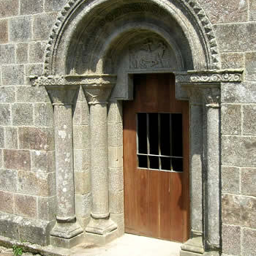} \\[4pt]
\includegraphics[width=.136\linewidth]{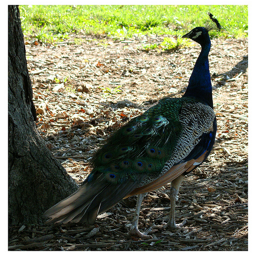} & \includegraphics[width=.136\linewidth]{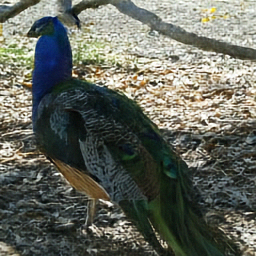} & \includegraphics[width=.136\linewidth]{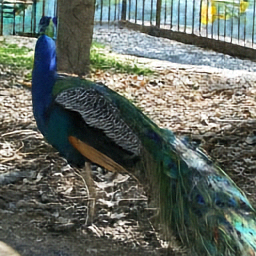} &
\includegraphics[width=.136\linewidth]{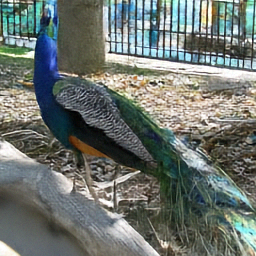} & \includegraphics[width=.136\linewidth]{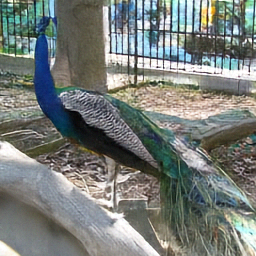} & \includegraphics[width=.136\linewidth]{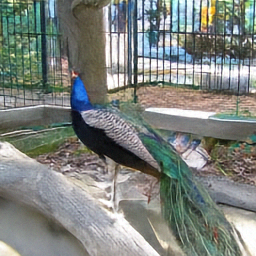} &
\includegraphics[width=.136\linewidth]{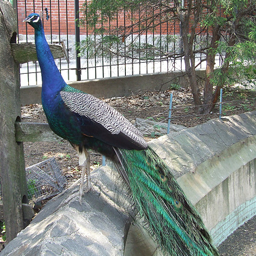} \\[4pt]
\includegraphics[width=.136\linewidth]{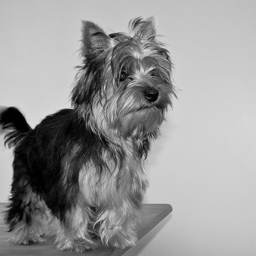} & \includegraphics[width=.136\linewidth]{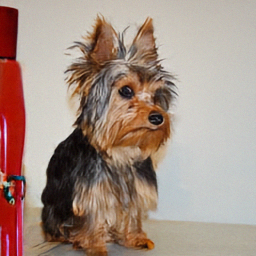} & \includegraphics[width=.136\linewidth]{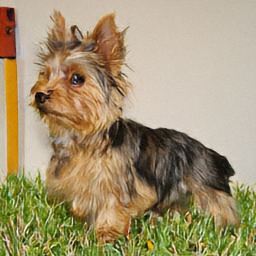} &
\includegraphics[width=.136\linewidth]{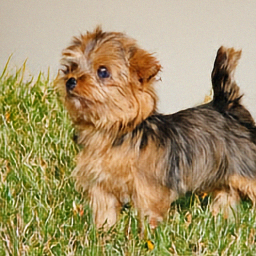} & \includegraphics[width=.136\linewidth]{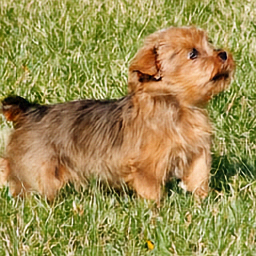} & \includegraphics[width=.136\linewidth]{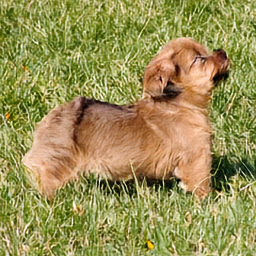} &
\includegraphics[width=.136\linewidth]{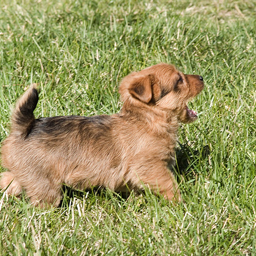} \\[4pt]
\includegraphics[width=.136\linewidth]{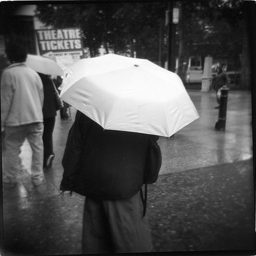} & \includegraphics[width=.136\linewidth]{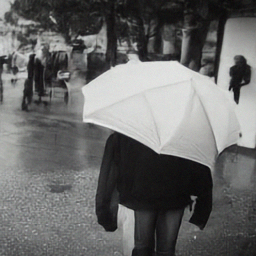} & \includegraphics[width=.136\linewidth]{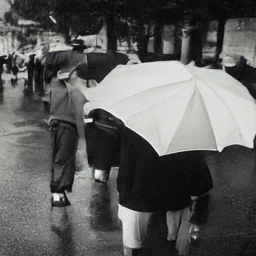} &
\includegraphics[width=.136\linewidth]{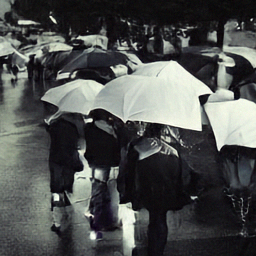} & \includegraphics[width=.136\linewidth]{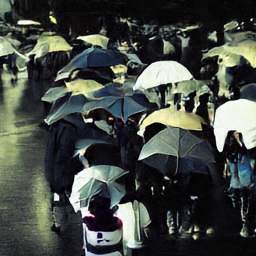} & \includegraphics[width=.136\linewidth]{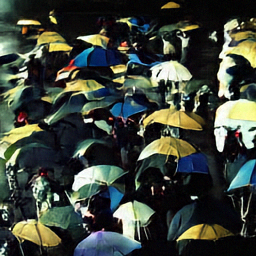} &
\includegraphics[width=.136\linewidth]{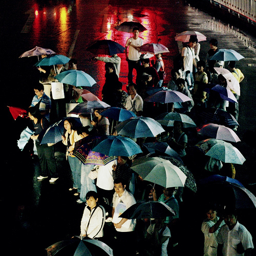} \\[4pt]
\includegraphics[width=.136\linewidth]{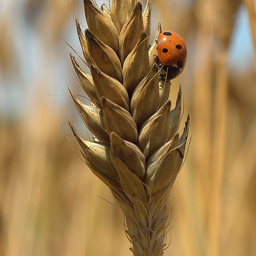} & \includegraphics[width=.136\linewidth]{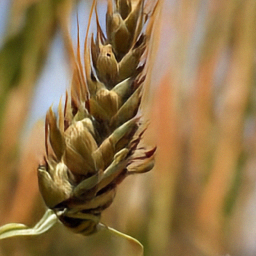} & \includegraphics[width=.136\linewidth]{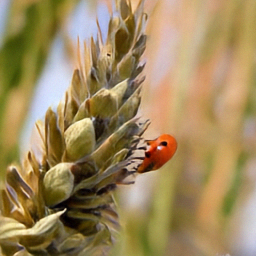} &
\includegraphics[width=.136\linewidth]{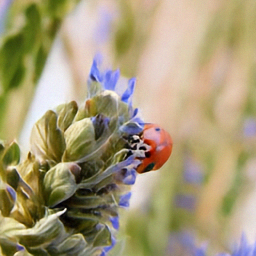} & \includegraphics[width=.136\linewidth]{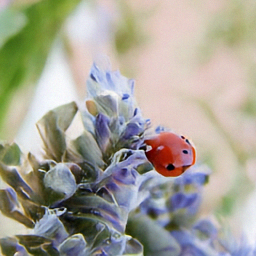} & \includegraphics[width=.136\linewidth]{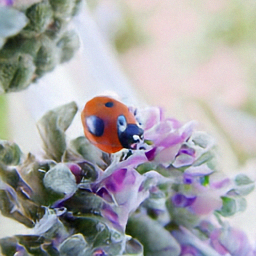} &
\includegraphics[width=.136\linewidth]{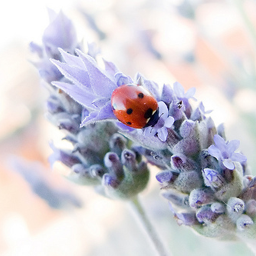} \\[4pt]
\includegraphics[width=.136\linewidth]{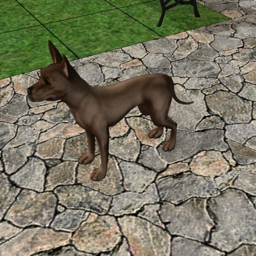} & \includegraphics[width=.136\linewidth]{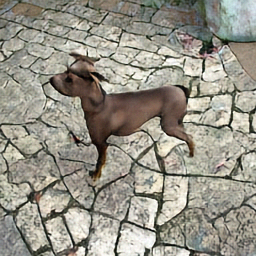} & \includegraphics[width=.136\linewidth]{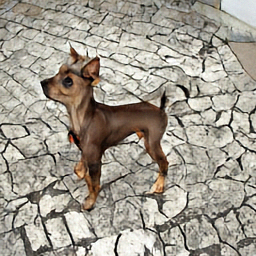} &
\includegraphics[width=.136\linewidth]{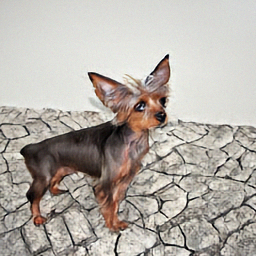} & \includegraphics[width=.136\linewidth]{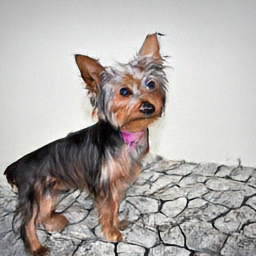} & \includegraphics[width=.136\linewidth]{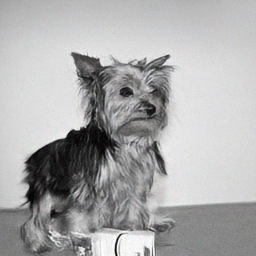} &
\includegraphics[width=.136\linewidth]{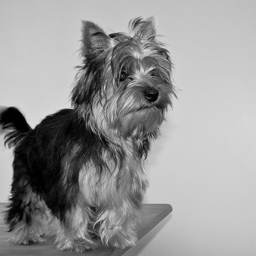} \\[4pt]
\includegraphics[width=.136\linewidth]{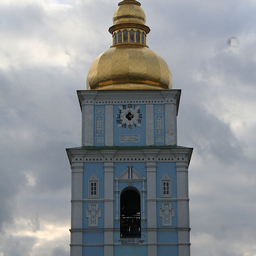} & \includegraphics[width=.136\linewidth]{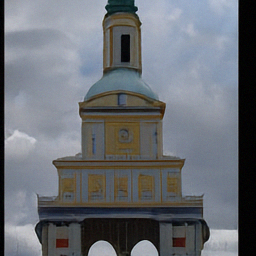} & \includegraphics[width=.136\linewidth]{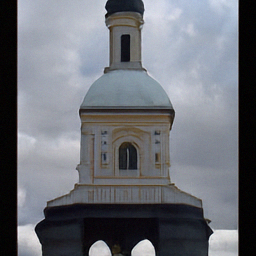} &
\includegraphics[width=.136\linewidth]{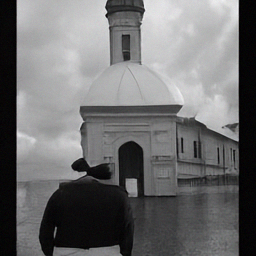} & \includegraphics[width=.136\linewidth]{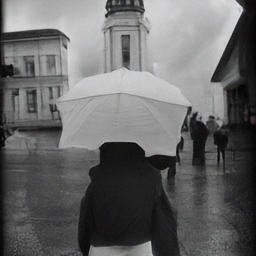} & \includegraphics[width=.136\linewidth]{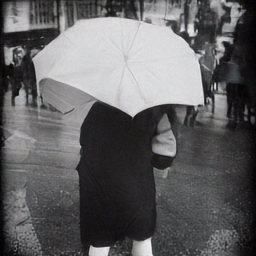} &
\includegraphics[width=.136\linewidth]{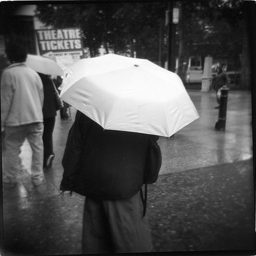} \\[4pt]
\includegraphics[width=.136\linewidth]{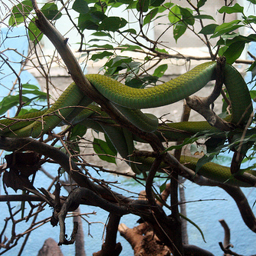} & \includegraphics[width=.136\linewidth]{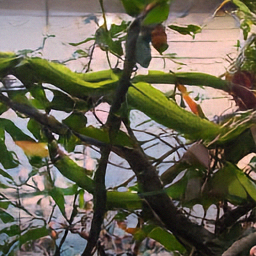} & \includegraphics[width=.136\linewidth]{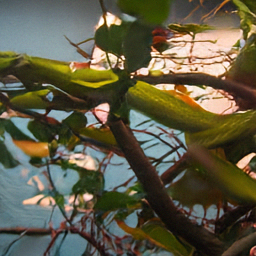} &
\includegraphics[width=.136\linewidth]{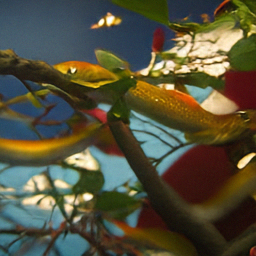} & \includegraphics[width=.136\linewidth]{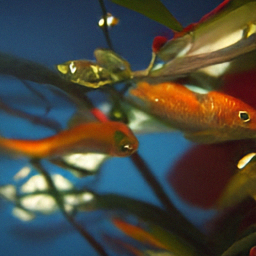} & \includegraphics[width=.136\linewidth]{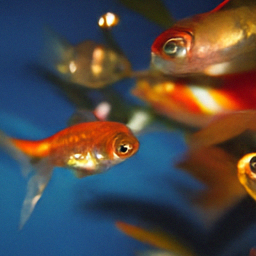} &
\includegraphics[width=.136\linewidth]{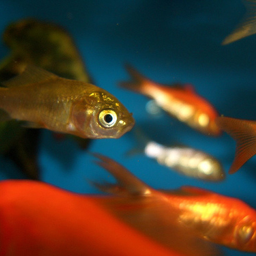} \\[4pt]
\includegraphics[width=.136\linewidth]{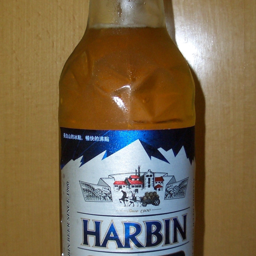} & \includegraphics[width=.136\linewidth]{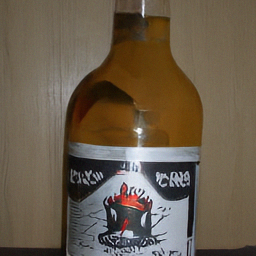} & \includegraphics[width=.136\linewidth]{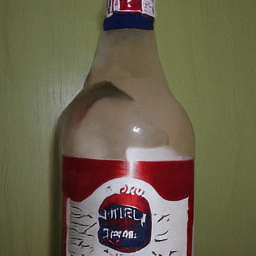} &
\includegraphics[width=.136\linewidth]{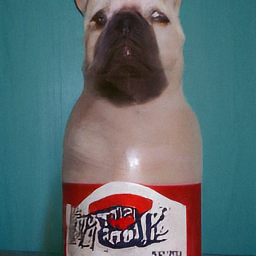} & \includegraphics[width=.136\linewidth]{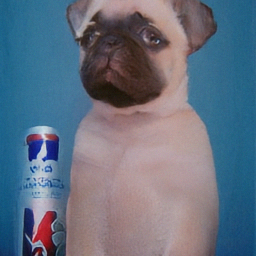} & \includegraphics[width=.136\linewidth]{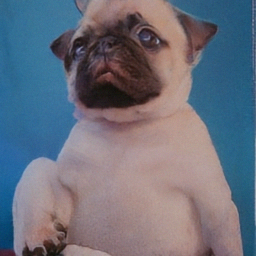} &
\includegraphics[width=.136\linewidth]{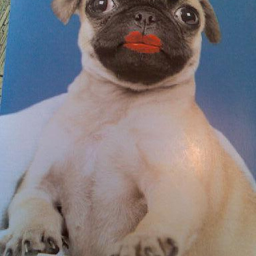} \\
\end{tabular}
\endgroup
\caption{\textbf{Interpolation in the learned conditioning space.} Real references appear at the outer
edges. All five interior images are generated from interpolated representations, with
$\alpha=0$ and $1$ using the endpoint representations directly. The first five rows connect
images from the same ImageNet class, while the last four connect different classes. Image noise is
fixed within each row. All paths use spherical interpolation, and \methodname{} (JiT-L/16) at $400$ epochs with representation guidance.}
\label{fig:showcase_interp}
\end{figure}

\clearpage

\begin{figure}[p]
\centering
\begingroup
\setlength{\tabcolsep}{1pt}

\begin{tabular}{@{}c@{\hspace{5pt}}cccccc@{\hspace{5pt}}c@{}}
\footnotesize Real bud & \multicolumn{6}{c}{\footnotesize Generated along the representation path} &
\footnotesize Real flower \\[2pt]
& \footnotesize $\alpha=0$ & \footnotesize $0.2$ & \footnotesize $0.4$ &
\footnotesize $0.6$ & \footnotesize $0.8$ & \footnotesize $1$ & \\[3pt]
\multicolumn{8}{@{}l}{\scriptsize (a) From a green bud to an open red-orange daisy} \\[1pt]
\includegraphics[width=.116\linewidth]{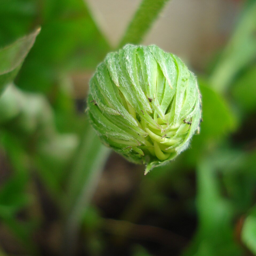} & \includegraphics[width=.116\linewidth]{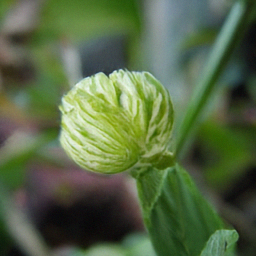} & \includegraphics[width=.116\linewidth]{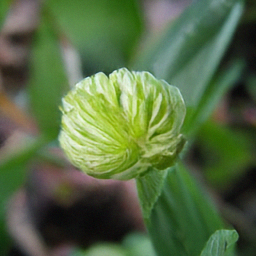} &
\includegraphics[width=.116\linewidth]{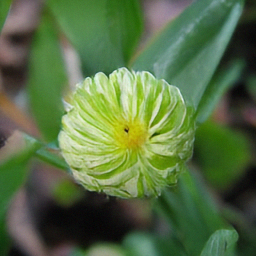} & \includegraphics[width=.116\linewidth]{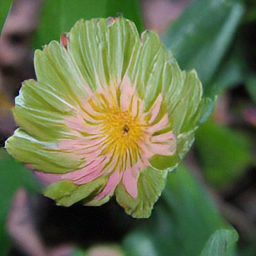} & \includegraphics[width=.116\linewidth]{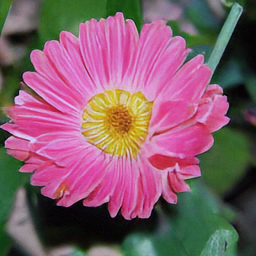} &
\includegraphics[width=.116\linewidth]{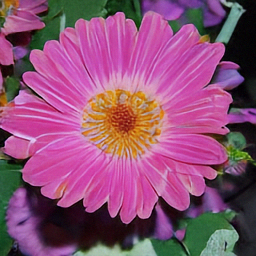} & \includegraphics[width=.116\linewidth]{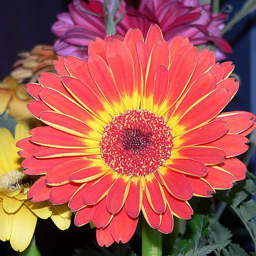} \\[4pt]
\multicolumn{8}{@{}l}{\scriptsize (b) From a purple-tipped bud to an open light-purple daisy} \\[1pt]
\includegraphics[width=.116\linewidth]{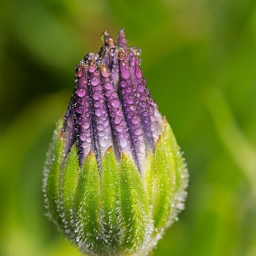} & \includegraphics[width=.116\linewidth]{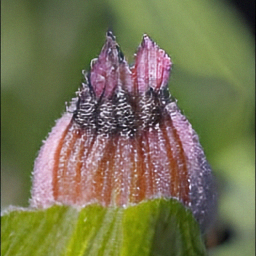} & \includegraphics[width=.116\linewidth]{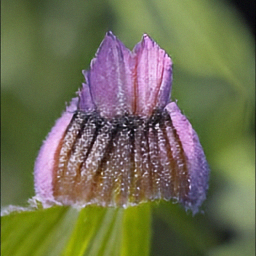} &
\includegraphics[width=.116\linewidth]{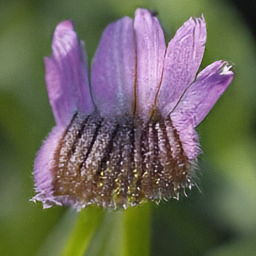} & \includegraphics[width=.116\linewidth]{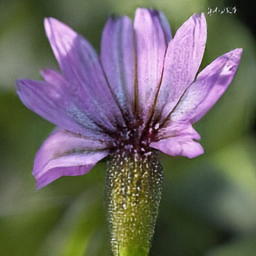} & \includegraphics[width=.116\linewidth]{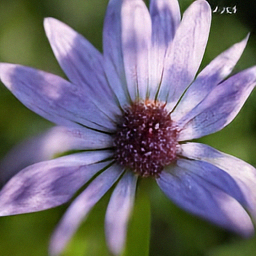} &
\includegraphics[width=.116\linewidth]{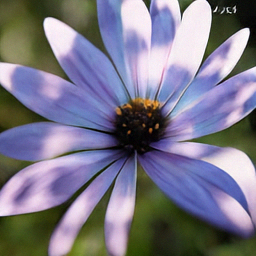} & \includegraphics[width=.116\linewidth]{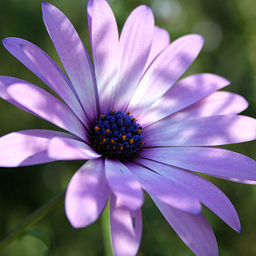} \\[4pt]
\multicolumn{8}{@{}l}{\scriptsize (c) From a partially closed flower to an open yellow lady's slipper} \\[1pt]
\includegraphics[width=.116\linewidth]{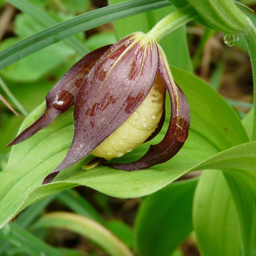} & \includegraphics[width=.116\linewidth]{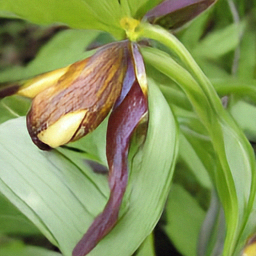} & \includegraphics[width=.116\linewidth]{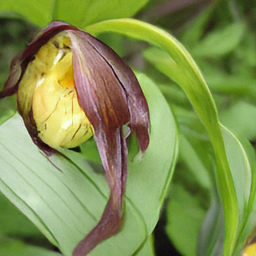} &
\includegraphics[width=.116\linewidth]{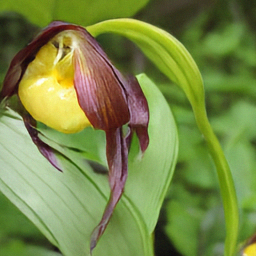} & \includegraphics[width=.116\linewidth]{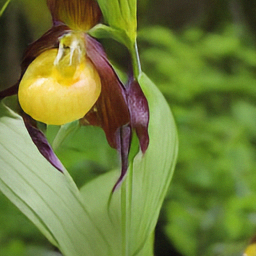} & \includegraphics[width=.116\linewidth]{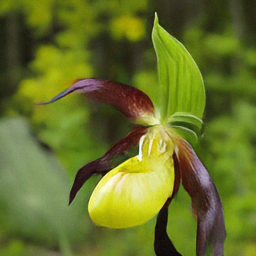} &
\includegraphics[width=.116\linewidth]{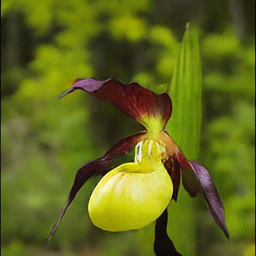} & \includegraphics[width=.116\linewidth]{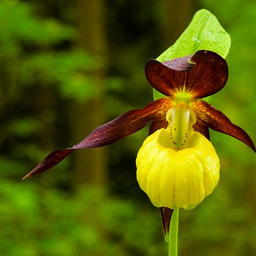} \\[4pt]
\multicolumn{8}{@{}l}{\scriptsize (d) From a pink bud to an open pink daisy} \\[1pt]
\includegraphics[width=.116\linewidth]{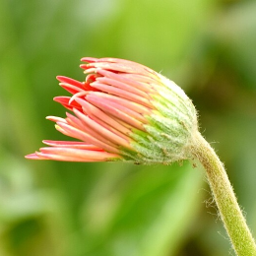} & \includegraphics[width=.116\linewidth]{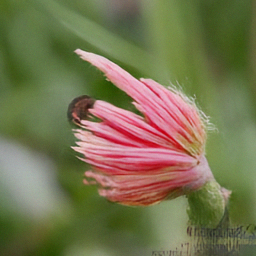} & \includegraphics[width=.116\linewidth]{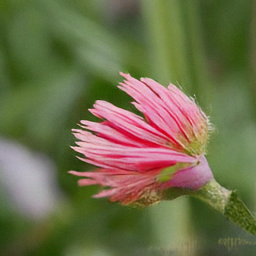} &
\includegraphics[width=.116\linewidth]{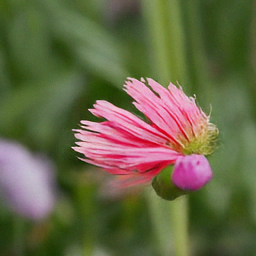} & \includegraphics[width=.116\linewidth]{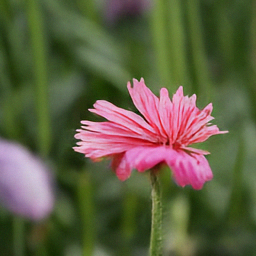} & \includegraphics[width=.116\linewidth]{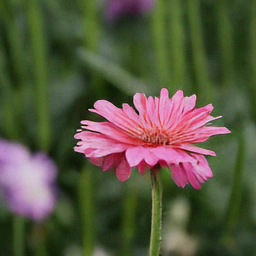} &
\includegraphics[width=.116\linewidth]{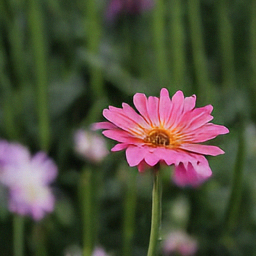} & \includegraphics[width=.116\linewidth]{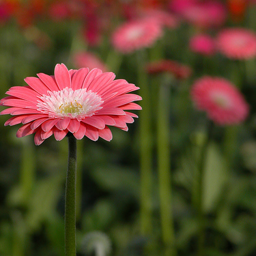} \\[4pt]
\multicolumn{8}{@{}l}{\scriptsize (e) From a purple-tipped bud to an open magenta daisy} \\[1pt]
\includegraphics[width=.116\linewidth]{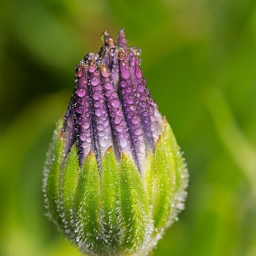} & \includegraphics[width=.116\linewidth]{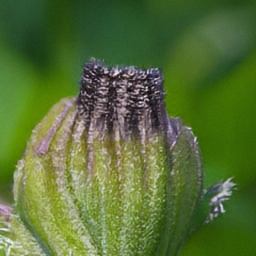} & \includegraphics[width=.116\linewidth]{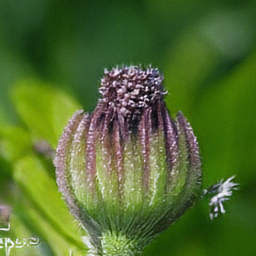} &
\includegraphics[width=.116\linewidth]{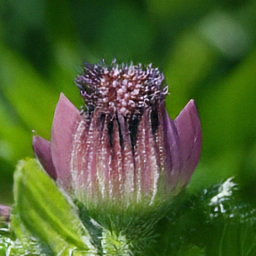} & \includegraphics[width=.116\linewidth]{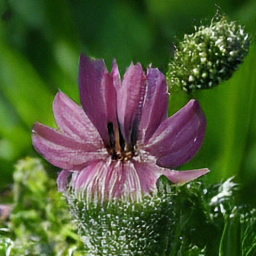} & \includegraphics[width=.116\linewidth]{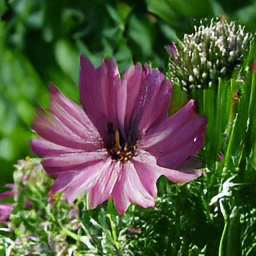} &
\includegraphics[width=.116\linewidth]{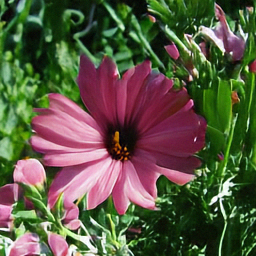} & \includegraphics[width=.116\linewidth]{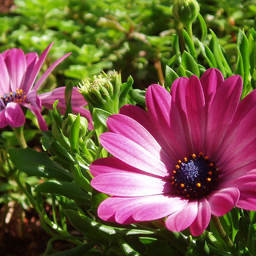} \\
\end{tabular}
\endgroup
\caption{\textbf{Flower-bud representation interpolations.}
These selected paths illustrate
representation changes rather than a recorded or predicted growth sequence.
Flower names for the right-hand references
follow their ImageNet class labels. Real photographs flank six generated images
at $\alpha=0,0.2,0.4,0.6,0.8,1$, including generated endpoint reconstructions.
All paths use spherical interpolation of the representations, fixed noise within each row, and \methodname{} (JiT-L/16) at $400$ epochs with representation guidance.}
\label{fig:showcase_flower_buds}
\end{figure}

\clearpage

\end{document}